\documentclass[11pt,letterpaper,twoside,reqno,nosumlimits]{amsart}

\usepackage{etoolbox}

\patchcmd{\section}{\scshape}{\bfseries}{}{}
\makeatletter
\renewcommand{\@secnumfont}{\bfseries}
\makeatother

\usepackage[dvipsnames]{xcolor}
\patchcmd{\section}{\normalfont}{\normalfont\color{MidnightBlue}}{}{}
\patchcmd{\subsection}{\normalfont}{\normalfont\color{MidnightBlue}}{}{}

\makeatletter
\def\subsubsection{\@startsection{subsubsection}{3}%
\z@{.5\linespacing\@plus.7\linespacing}{-.5em}%
{\normalfont\bfseries}}
\makeatother

\usepackage{fancyhdr}
\usepackage{amsmath,amsfonts,amsbsy,amsgen,amscd,mathrsfs,amssymb,amsthm,mathtools,tensor}

\usepackage[a4paper]{geometry}
\usepackage{tikz,pgflibraryshapes}
\usetikzlibrary{shadows}
\usetikzlibrary{arrows}
\usepackage{overpic}

\usepackage{stmaryrd}

\usepackage{caption}
\usepackage{graphicx}
\usepackage{color}
\usepackage[bf,SL,BF]{subfigure}
\usepackage{url}
\usepackage{epstopdf}
\usepackage{pdfsync}
\usepackage[colorlinks]{hyperref}
\usepackage[all]{xypic}
\usepackage{comment}
\usepackage{mathabx}
\usepackage{upgreek}

\usepackage{mathrsfs}
\usepackage{yfonts}

\hypersetup{
    linkcolor=blue,
}

\newlength{\fixboxwidth}
\usepackage{epsfig,amsbsy,graphicx,multirow}

\usepackage{ algorithm, algorithmic}
\renewcommand{\algorithmiccomment}[1]{\bgroup\hfill//~#1\egroup}

\usepackage[all,cmtip]{xy}

\usepackage{accents}

 \numberwithin{equation}{section}

\def\cl{{ \operatorname{\bf cl}}}

\def\R{\mathbb{R}}

\def\cN{\mathcal{N}}

\def\Y{{\bf\mathcal{Y}}}

\def\E{\mathbb{E}}

\def\Xf{\mathfrak{X}}

\def\Vk{\mathfrak{V}}
\def\Wk{\mathfrak{W}}
\def\Mk{\mathfrak{M}}

\def\X{{\bf\mathcal{X}}}
\def\F{\mathcal{F}}
\def\G{\mathcal{G}}

\def\L{\mathcal{L}}
\def\Lf{\mathfrak{L}}
\def\Hf{\mathfrak{H}}
\def\Fk{\mathfrak{F}}

\def\wK{{\Gamma}}

\def\A{\mathcal{A}}

\def\H{\mathcal{H}}

\def\restrict#1{\raise-.5ex\hbox{\ensuremath|}_{#1}}

\def\<{\big\langle}
\def\>{\big\rangle}

\def\diiv{\operatorname{div}}

\def\Tr{\operatorname{Tr}}

\def\Cov{\operatorname{Cov}}

\def\det{\operatorname{det}}

\def\dim{{\operatorname{dim}}}

\def\ba{{\bf a}}
\def\bvarphi{{\boldsymbol{\varphi}}}

\definecolor{red}{rgb}{0.9, 0, 0}

\newtheorem{Theorem}{Theorem}[section]
\newtheorem{Proposition}[Theorem]{Proposition}
\newtheorem{Lemma}[Theorem]{Lemma}
\newtheorem{Corollary}[Theorem]{Corollary}
\newtheorem{Remark}[Theorem]{Remark}
\newtheorem{Example}[Theorem]{Example}

\newtheorem{Definition}[Theorem]{Definition}

\newtheorem{Condition}[Theorem]{Condition}
\newtheorem{Problem}{Problem}

\makeatletter
\newcommand{\oset}[3][0ex]{%
  \mathrel{\mathop{#3}\limits^{
    \vbox to#1{\kern-2\ex@
    \hbox{$\scriptstyle#2$}\vss}}}}
\makeatother

\makeatletter
\newcommand{\uset}[3][0ex]{%
  \mathrel{\mathop{#3}\limits_{
    \vbox to#1{\kern-2\ex@
    \hbox{$\scriptstyle#2$}\vss}}}}
\makeatother

\usepackage{aurical}
\usepackage[T1]{fontenc}

\def\Fp{{\Fontauri\bfseries F}}

\begin{document}
\title[Do ideas have shape?]{Do ideas have shape?  Idea registration as the continuous limit of artificial neural networks}

\date{\today}

\author{Houman Owhadi}

\thanks{Caltech,  MC 9-94, Pasadena, CA 91125, USA, owhadi@caltech.edu}

\maketitle

\begin{abstract}
We introduce  a Gaussian Process (GP) generalization  of ResNets (with unknown functions of the network  replaced by GPs and identified via MAP estimation), which includes  (ResNets trained with $L_2$ regularization on weights and biases) as a particular case (when employing particular kernels).
We show that ResNets (and their warping GP regression extension) converge, in the infinite depth limit, to a generalization of image registration variational algorithms. In this generalization, images are replaced by functions mapping input/output spaces to a space of unexpressed abstractions (ideas), and material points are replaced by data points. Whereas computational anatomy aligns images via warping of the material space, this generalization aligns ideas (or abstract shapes as in Plato's theory of forms) via the warping of the Reproducing Kernel Hilbert Space (RKHS) of functions mapping the input space to the output space.
 While the Hamiltonian interpretation of ResNets is not new, it was based on an Ansatz.
We do not rely on this Ansatz and present the first rigorous proof of convergence of ResNets with trained weights and biases towards a  Hamiltonian dynamics driven flow. Since our proof is constructive and based on discrete and continuous mechanics, it reveals several remarkable properties of ResNets and their GP generalization. ResNets regressors are kernel regressors with data-dependent warping kernels.
 Minimizers of $L_2$ regularized ResNets satisfy a discrete least action principle implying the near preservation of the norm of weights and biases across layers. The trained weights of ResNets with scaled/strong $L^2$ regularization can be identified by solving an autonomous Hamiltonian system. The trained ResNet parameters are unique up to (a function of) the initial momentum, and the initial momentum representation of those parameters is generally sparse.  The kernel (nugget) regularization strategy provides a provably robust alternative to Dropout for ANNs. We introduce a functional generalization of GPs and show that
 pointwise GP/RKHS error estimates lead to probabilistic and deterministic generalization error estimates for ResNets.
 When performed with feature maps, the proposed analysis identifies the (EPDiff) mean fields limit of trained ResNet parameters as the number of data points goes to infinity.
 The search for good architectures can be reduced to that of good kernels, and we show that the composition of warping regression blocks with reduced equivariant multichannel kernels (introduced here) recovers and generalizes CNNs to arbitrary spaces and groups of transformations.
\end{abstract}


\section{Introduction}
\subsection{Overview}
This paper introduces a Gaussian Process (GP) generalization of residual neural networks (ResNets) \cite{he2016deep} (in which the unknown functions of the network are replaced by GPs and identified via MAP estimation), which includes ResNets (with $L_2$ regularization on weights and biases) as a particular case (when employing particular kernels for the underlying GPs).
One of its main results is to show  that residual neural networks (ResNets) \cite{he2016deep}  (and their GP generalizations) are essentially
discretized solvers for a  generalization of image registration/computational anatomy variational problems.
This identification initiates a theoretical understanding of deep learning from the perspectives of (1) shape analysis with images replaced by abstractions, (2) Lagrangian/Hamiltonian mechanics, (3) GP/Kernel regression with data-dependent warping kernels.
While the discretized ODE interpretation of ResNet is not new \cite{weinan2017proposal, chen2018neural}, it was based on the Ansatz that ResNets with trained weights and biases can be approximated by training a discrete regular ODE. \cite{thorpe2018deep}  proved this
Ansatz by establishing the $\Gamma$-convergence of ResNets with trained weights and biases to an ODE limit defined by their activation function.
While the Hamiltonian \cite{haber2017stable} and optimal control \cite{LiChenetal17, han2019mean} perspectives are not new, they were also based on a similar Ansatz.
We do not rely on this Ansatz and present the first rigorous proof of convergence of ResNets with trained weights and biases towards a  Hamiltonian dynamics driven flow. Since our proof is constructive and based on discrete and continuous mechanics, it reveals several remarkable properties of ResNets and their GP generalization. (1) The $L^2$ norm of the trained weights and biases are nearly constant (across layers\footnote{This provides a (minimization success) criteria characterizing minimizers of the training loss.}) when the number of layers of the network is finite and exactly constant in the infinite depth limit, thereby providing a (minimization success) criteria characterizing minimizers of the training loss. (2) The trained parameters of the network are entirely determined by those of the first layer (associated with the initial momentum). Furthermore, the initial momentum is generality sparse\footnote{This is analogous to that of support vectors in support vector machines}, which provides a sparse representation of trained weights and biases. (3) Regressing the data with a ResNet is equivalent to kernel ridge regression with a data adapted warped kernel. (4) GP (probabilistic and deterministic) error estimates imply generalization error estimates. (5) The brittleness of Bayesian inference with respect to the prior implies that of ResNets, and we propose a  regularization strategy  (generalizing the concept of nuggets from kernels to networks), ensuring the rigorous stabilization of the underlying network. (6) When performed with feature maps, the proposed analysis identifies the (EPDiff) mean-field limit of the trained weights and biases of the network as the number of data points goes to infinity. (7) The kernel generalization of ResNets enables the generalization of convolutional neural networks (CNNs) architectures \cite{lecun1999object}  to networks that are equivariant with respect to arbitrary groups of transformations through the introduction of structured kernels.
The convergence of ResNet regression towards GP regression with data-dependent kernels and the techniques developed in this paper suggest that Deep Learning can be understood and analyzed as (1) kernel-based learning with data-dependent (adapted) structured kernels (as suggested in \cite{belkin2021fit}), (2) or as completing computational graphs with GPs \cite{owhadi2021computational}.
While kernel methods may be perceived as old and outdated due to unfavorable efficacy comparisons with ANN-based methods, these comparisons are oftentimes made with given/fixed kernels, whereas learning the kernel \cite{owhadi2019kernel, chen2020consistency, hamzi2021learning, akian2022learning} can improve accuracy by several orders of magnitude \cite{akian2022learning, hamzi2021learning} and outperform \cite{hamzi2021simple} ANN-based methods  both in terms of accuracy and complexity\footnote{
In the setting of forecasting time-series learning, learning the kernel improves accuracy by several orders of magnitude \cite{hamzi2021learning} and outperforms ANN-based, and PDE methods for weather/climate forecasting \cite{hamzi2021simple}  both in terms of accuracy and complexity.
While ANN-based methods are usually trained by minimizing training error  \cite{yoo2020deep, shirdel2021deep} show that one could achieve improved generalization errors by training ANNs as data-dependent kernels \cite{owhadi2019kernel, chen2020consistency}.} .

\subsection{Structure of the paper}
This paper is structured as follows. Sec.~\ref{secidreg} presents  the GP/kernel generalization of ResNets (warping regression, Sec.~\ref{secresnetb} and Sec.~\ref{subresnetpar}) and their infinite depth convergence towards solutions of a generalization of image registration problems (idea registration, Sec.~\ref{subinfdelim}) and towards kernel regression with data adapted warping kernels (Sec.~\ref{subsecgytref67f6}).
The underlying setting is that of operator-valued kernels, and Sec.~\ref{secovk} (of the appendix) provides a reminder on those kernels. For ease of presentation, we cover primary results in the main part of the paper and more technical results (along with reminders) in the appendix.
For instance, existence and uniqueness results are discussed in Sec.~\ref{subsedeses} and covered in Sec.~\ref{seclkjdekjdhjdex}.
Sec.~\ref{seclkjdekjdhjd} analyzes the discrete and limit variational problems (presented in Sec.~\ref{secidreg}) from the perspectives of Lagrangian and Hamiltonian mechanics (our convergence results are based on this analysis).
Sec.~\ref{secregular} introduces and analyzes a  regularization strategy for the underlying networks. This strategy is rigorous (it implies the continuity of the regressor with respect to the training/testing data), and it generalizes approaches commonly employed with kernel methods and in image registration.
Sec.~\ref{seckjhhehwd78w696d} introduces a functional (operator-valued kernel based) generalization of GPs, leading to probabilistic and deterministic generalization error estimates and deep residual GP interpretation of the methods discussed in Sec.~\ref{secidreg}.
Sec.~\ref{secfres} discusses further results. These include the feature-map representation/analysis of the proposed methods (Sec.~\ref{secmacfm}, unpacked in Sec.~\ref{secfm} of the appendix), numerical experiments (Sec.~\ref{subsecnumexmp}, unpacked in Sec.~\ref{secnumexp} of the appendix),
 the (EPDiff) mean-field/hydrodynamic limit of the underlying methods  (Sec.~\ref{subjhsecgyf67f6}), the multi-resolution generalization of the proposed analysis (Sec.~\ref{secspamulti}), generalizations obtained by composing warping regression/idea registration blocks (Sec.~\ref{subsecukygeuydd}, unpacked in Sec.~\ref{seckehddjhevd}), structured kernels enabling a generalization of  CNN equivariant architectures to arbitrary groups of transformations acting on arbitrary spaces (Sec.~\ref{secremint}, unpacked in Sec.~\ref{secrem}).
 Sec.~\ref{secrelwo} presents and discusses related papers and Sec.~\ref{seckejhdbeyudydgor} concludes this paper.
See also \cite{owhadiyoutube20} for an oral/visual presentation of the content of this paper.

\section{Idea registration}\label{secidreg}
This section
identifies the infinite depth limits of ResNets (Residual Neural Networks
\cite{he2016deep}) with trained weight and biases as solutions to a generalization of image registration problems (idea registration).
Results are obtained and presented in a generalized setting \cite{owhadi2021computational} (containing ResNets as a particular case) in which the classical ResNet layers are replaced with Gaussian Processes and trained by computing their MAP estimator given the data.
Sec.~\ref{subsett1} and \ref{subkermethsol} setup notations by describing the supervised learning problem and its classical kernel-based solutions. Subsec.~\ref{secresnetb} and \ref{subresnetpar} show how ResNets can be analyzed in a kernel setting via warping regression.
Subsec.~\ref{subinfdelim} identifies idea registration as the infinite depth limit of warping regression/ResNets.

\subsection{The supervised learning problem}\label{subsett1}

Let $\X$ and $\Y$ be separable Hilbert spaces\footnote{Although  $\X$ and $\Y$ are finite-dimensional in all practical applications,
and although we will restrict some of our proofs to the finite-dimensional setting to minimize technicalities, as demonstrated in
\cite{nelsen2020random}, it is useful to keep the infinite-dimensional viewpoint in the identification of discrete models with desirable attributes inherited from the infinite-dimensional setting.} endowed with the inner products $\<\cdot,\cdot\>_\X$ and $\<\cdot,\cdot\>_\Y$.
We employ the setting of supervised learning, which can be expressed as solving the following problem.
\begin{Problem}\label{pb828827hee}
Let $f^\dagger$ be an unknown continuous function mapping $\X$ to $\Y$.
 Let
$Z=(Z_1,\ldots,Z_N)$ be a random Gaussian vector, independent from $\xi$, with i.i.d. $\cN(0,\lambda I_\Y)$ entries\footnote{$\lambda\geq 0$ and $I_\Y$ is the identity map on $\Y$}.
Given the information\footnote{For a $N$-vector $X=(X_1,\ldots,X_N)\in \X^N$ and a function $f\,:\, \X\rightarrow \Y$, write $f(X)$ for the $N$ vector with entries
$\big(f(X_1),\ldots,f(X_N)\big)$ (we will keep using this generic notation).}
  $f^\dagger(X)+Z=Y$ with the data $(X,Y)\in \X^N\times \Y^N$ approximate $f^\dagger$.
\end{Problem}
Using red arrows to represent unknown functions, black arrows to represent known functions, dashed arrows to represent the data and blue squares to represent  random variables, we can represent the underlying problem (assuming to data to be noisy with centered $\cN(0,I_\Y)$ Gaussian noise where $I_\Y$ is the identity operator on $\Y$) as that of completing (identifying the unknown function $f$ in) the following computational graph \cite{owhadi2021computational}
\centerline{
\begin{tikzpicture}[->,>=stealth',shorten >=1pt,auto,node distance=3cm,
                    thick,main node/.style={rectangle,draw,font=\sffamily\Large\bfseries}]

\node[main node] (1) {$x$};
\node[main node] (2) [right of=1] {$y$};
\node[main node] (3) [above of=2,blue, node distance=1cm] {$z$};

\path[every node/.style={font=\sffamily\Large\bfseries},red]
    (1) edge node [above ] {$f$} (2);

\path[every node/.style={font=\sffamily\Large\bfseries}]
    (3) edge node [above ] {} (2);

\path[every node/.style={font=\sffamily\Large}]
    (1) edge  [bend right,dashed ] node[below ] {$(X,Y)$} (2);

\end{tikzpicture}}
which can be unpacked as $y=f(x)+z$ and $(X,Y)$ represents the data $Y_i=f(X_i)+Z_i$ where the $Z_i$ are independent copies of $z$.

\subsection{Kernel method solutions to the approximation problem \ref{pb828827hee}}\label{subkermethsol}
Write $\L(\Y)$ for the set of bounded linear operators mapping $\Y$ to $\Y$.
Let $K\,:\, \X\times \X\rightarrow \L(\Y)$ be an operator valued kernel\footnote{See Sec.~\ref{secovk} for a reminder on operator-valued kernels.} defining a reproducing kernel Hilbert space (RKHS) of functions
mapping $\X$ to $\Y$. Write $\mathcal{H}_K$ and $\|\cdot\|_K$ for the RKH space and norm defined by $K$.

\subsubsection{The optimal recovery solution}\label{subsecoptrecrr1}
Assume  $K$ to be non-degenerate.
Using the relative error in $\|\cdot\|_{K}$-norm as a loss, for $\lambda=0$, the minimax optimal recovery solution of Problem \eqref{pb828827hee}
 is \cite[Thm.~12.4,12.5]{owhadi2019operator} the minimizer (in $\H_K$) of
\begin{equation}\label{eqhgvygvgyv}
\ell(X,Y):=\begin{cases}
\text{Minimize }&\|f\|_{K}^2\\
\text{subject to }&f(X)=Y
\end{cases}
\end{equation}
 By the representer theorem \cite{micchelli2005kernels}, the minimizer of \eqref{eqhgvygvgyv} is
 \begin{equation}
 f(\cdot)=\sum_{j=1}^N K(\cdot,X_j) V_j\,,
 \end{equation}
  where the coefficients $V_j \in \Y$ are identified by solving the  system of linear equations
 \begin{equation}
 \sum_{j=1}^N K(X_i,X_j) V_j=Y_i\text{ for all }i\in \{1,\ldots,N\}\,,
 \end{equation}
 i.e.
 $K(X,X) V=Y$ where $V=(V_1,\ldots,V_N),\,Y=(Y_1,\ldots,Y_N)\in \Y^N$ and $K(X,X)$ is the $N\times N$ block-operator matrix\footnote{ For $N\geq 1$ let $\Y^N$ be the N-fold product space endowed with the inner-product $\<Y,V\>_{\Y^N}:=\sum_{i=1}^N \<Y_i,V_i\>_\Y$ for
 $Y=(Y_1,\ldots,Y_N), V=(V_1,\ldots,V_N) \in \Y^N$.
  ${\bf A}\in \L(\Y^N)$ given by
  $
  {\bf A}=\begin{pmatrix}A_{1,1}&\cdots & A_{1,N}\\ \vdots & & \vdots\\A_{N,1}& \cdots & A_{N,N} \end{pmatrix}
 $
 where $A_{i,j}\in \L(\Y)$, is called a block-operator matrix. Its adjoint  ${\bf A^T}$ with respect to  $\<\cdot,\cdot\>_{\Y^N}$ is the
 block-operator matrix with entries $(A^T)_{i,j}=(A_{j,i})^T$.} with entries $K(X_i,X_j)$. Therefore, writing $K(\cdot,X)$ for the vector $(K(\cdot,X_1),\ldots,K(\cdot,X_N))\in (\H_K)^N$,
 the minimizer of \eqref{eqhgvygvgyv} is
 \begin{equation}\label{eqhgvygvgyv2}
 f(\cdot)= K(\cdot,X)  K(X,X)^{-1} Y\,,
 \end{equation}
which implies
\begin{equation}\label{eqkjhejdehgdkdd}
\ell(X,Y)=\|f\|_K^2=Y^T K(X,X)^{-1} Y\,,
\end{equation}
 where $K(X,X)^{-1}$ is the inverse of $K(X,X)$ (whose existence is implied by the non-degeneracy of $K$ combined with
 $X_i\not=X_j$ for $i\not=j$).

\subsubsection{The ridge regression solution}\label{subsecoptrecrr2}
Let $\ell_\Y \,:\, \Y^N \times \Y^N \rightarrow [0,\infty]$ be an arbitrary  continuous positive loss.
A ridge regression solution (also known as Tikhonov regularizer) to Problem \ref{pb828827hee} (for $\lambda>0$) is a minimizer of
\begin{equation}\label{eqledhehdiudh}
\ell(X,Y):=\inf_{f\in \H_K}\lambda\,\|f\|_K^2+\ell_\Y(f(X),Y)\,.
\end{equation}
with the {\bf empirical squared error}
\begin{equation}\label{eqjhguyghfuvf4}
\ell_\Y(Y',Y)=\|Y'-Y\|_{\Y^N}^2:=\sum_{i=1}^N\|Y_i'-Y_i\|_\Y^2\,,
\end{equation}
as a prototypical example.  The solution obtained by minimizing \eqref{eqledhehdiudh} with $\ell_\Y$=\eqref{eqjhguyghfuvf4} is then equivalent to replacing $f$ (the graph shown in Subsec.~\ref{subsett1}) by a centered $\cN(0,K)$ Gaussian Process (GP)
and computing its MAP estimator given the noisy data $f(X)=Y+Z$ with  $Z\sim \cN(0,\lambda I_{\Y^N})$.

By the representer theorem, \eqref{eqledhehdiudh} admits a minimizer of the form $ f(\cdot)= K(\cdot,X) V$ where $V\in \Y^N$ is identified as the minimizer of
\begin{equation}\label{eqledhehdwswswededsw}
\ell(X,Y)=\inf_{V\in \Y^N}\lambda\,V^T K(X,X) V+\ell_\Y(K(X,X) V,Y)\,.
\end{equation}
 In particular, for $\ell_\Y$ defined as in \eqref{eqjhguyghfuvf4}, the minimizer of \eqref{eqledhehdiudh} is
\begin{equation}\label{eqajkjwdhjbdjeh}
f(x)=K(x,X)\big(K(X,X)+\lambda I\big)^{-1} Y\,,
\end{equation}
(writing $I$ for the identity matrix)
and the value of \eqref{eqledhehdiudh} at the minimum is
\begin{equation}\label{eqlkjdkjweedjkb}
\ell(X,Y)=\lambda Y^T \big(K(X,X)+\lambda I\big)^{-1} Y\,.
\end{equation}

\subsection{Warping regression}\label{secresnetb}
Motivated by the  structure of ResNets
\cite{he2016deep} we  seek to approximate $f^\dagger$ in Problem \ref{pb828827hee}
 by a function of the form
\begin{equation}\label{eqkjedjdiseu}
f^\ddagger=f\circ \phi_L\,,
\end{equation}
where (writing $I$ for the identity map on $\X$)
\begin{equation}\label{eqkjehdbehdhjbd}
\phi_L:=(I+v_L)\circ \cdots \circ (I+v_1)
\end{equation}
is a function (large deformation) mapping $\X$ to itself obtained from the unknown residuals (small deformations) $v_k\,:\, \X\rightarrow \X$
and $f\,:\, \X\rightarrow \Y$ is an unknown function mapping $\phi_L(X)$ (the image of the data $X$ under the deformation $\phi_L$) to $Y$.

Using the computational graph representation of Sec.~\ref{subsett1},  this problem can be represented, for $L=3$, as that of identifying the unknown functions $v_1,v_2,v_3,f$ in the following computational graph,\\
\centerline{
\begin{tikzpicture}[->,>=stealth',shorten >=1pt,auto,node distance=2.5cm,
                    thick,main node/.style={rectangle,draw,font=\sffamily\Large\bfseries}]

\node[main node] (1) {$x$};
\node[main node] (2) [right of=1]  {$q_2$};
\node[main node] (3) [right of=2] {$q_3$};
\node[main node] (4) [right of=3] {$q_4$};
\node[main node] (5) [right of=4] {$y$};
\node[main node] (6) [above of=5,blue,node distance=1.5cm] {$z$};

\path[every node/.style={font=\sffamily\Large\bfseries}]
    (1) edge node [below ] {} (2)
    (1) edge  [bend left,red] node[above ] {$v_1$} (2);

\path[every node/.style={font=\sffamily\Large\bfseries}]
    (2) edge node [below ] {} (3)
    (2) edge  [bend left,red] node[above ] {$v_2$} (3);

\path[every node/.style={font=\sffamily\Large\bfseries}]
    (3) edge node [below ] {} (4)
    (3) edge  [bend left,red] node[above ] {$v_3$} (4);

\path[every node/.style={font=\sffamily\Large\bfseries},red]
    (4) edge node [above ] {$f$} (5);

\path[every node/.style={font=\sffamily\Large\bfseries}]
    (6) edge node [right ] {} (5);


\tikzstyle{every to}=[draw,dashed]
\draw[dashed] (1) to[out=-90,in=-90,looseness=0.3,style={font=\sffamily\Large,dashed}] node[above] {$(X,Y)$} (5);

\end{tikzpicture}}
which can be unpacked as $q_2=v_1(x)+x$, $q_3=v_2(q_2)+q_2$, $q_4=v_3(q_3)+q_3$, $y=f(q_4)+z$ where $z$ is a $\cN(0,\lambda I_\Y)$ random variable.

The ResNets \cite{he2016deep} approach to completing this graph is to replace the unknown functions $(v_1,\ldots,v_L,f)$ by one or two layers neural networks whose parameters are identified by minimizing the data mismatch loss $\ell_\Y(f\circ \phi_L(X),Y)$. In this paper, we will analyze the Gaussian Process (GP) approach to identifying these unknown functions. This approach
(generalized in \cite{owhadi2021computational} to arbitrary computational graphs) can be summarized as approximating $v_1,\ldots,v_L, f$ with a MAP estimator of independent GPs   given the (noisy) data. Letting $\Gamma$ be a kernel (associated with the randomization of the $v_i$) defining an RKHS $\H_\Gamma$ of functions mapping $\X$ to $\X$ (we write $\|\cdot\|_\Gamma$ for the corresponding RKHS norm), this GP approach is equivalent to
identifying $(v_1,\ldots,v_L,f)$ with a minimizer of
\begin{equation}\label{eqlktddsyeytdsedhjd}
\begin{cases}
\text{Minimize } &\frac{\nu}{2}\,L\sum_{s=1}^L \|v_s\|_{\Gamma}^2+\lambda\,\|f\|_{K}^2+\ell_\Y\big(f\circ \phi_L(X),Y\big)\\
\text{over }&v_1,\ldots,v_L \in \H_\Gamma \text{ and } f\in \H_K\,,
\end{cases}
\end{equation}
where $\nu$ is a strictly positive parameter balancing the regularity of $\phi_L$ with that of $f$ (the scaling $\nu L/2$ will be shown to ensure a nontrivial limit as $L\rightarrow \infty$ for $\nu>0$) and
$\lambda>0$ balances the regularity of $f$ with the loss $\ell_\Y$=\eqref{eqjhguyghfuvf4}.
Note that \eqref{eqlktddsyeytdsedhjd} addresses overparameterization ($\H_\Gamma$ and $\H_K$ may be infinite-dimensional) by
penalizing the lack of regularity of the $v_k$ and $f$ with respect to the RKHS norms defined by   $K$ and  $\Gamma$.

\begin{Remark}
As $\nu\rightarrow \infty$, the regressor obtained by minimizing \eqref{eqlktddsyeytdsedhjd} converges towards \eqref{eqajkjwdhjbdjeh}.
The quadratic regularization in $\|v_s\|_{\Gamma}^2$ in \eqref{eqlktddsyeytdsedhjd} is equivalent to choosing Gaussian priors on the unknown functions $v_s$ (or equivalently, on the weights and biases of the network).
This choice impacts the generalization of the network. In the CGC setting \cite{owhadi2021computational}, other choices of priors can be implemented by writing representing the $v_s$ as nonlinear deterministic functions of GPs (i.e., $v_s=h(\bar{v}_s)$ where $h$ is deterministic and $\bar{v}_s$ is a GP). We also refer to \cite{cohen2021scaling} for a numerical analysis of the scaling properties of ResNets under a regularization that is weaker than that used in \eqref{eqlktddsyeytdsedhjd}.
\end{Remark}

\subsection{ResNets as a particular case}\label{subresnetpar}
Using the setting (Sec.~\ref{secovk}) of operator-valued kernels \cite{alvarez2012kernels}, we show (Subsec.~\ref{subsecekuwege62}) that
if  $\Gamma(x,x')=\bvarphi^T(x) \bvarphi(x') I_\X$ and $K(x,x')=\bvarphi^T(x) \bvarphi(x') I_\Y$ where
$I_\X$ ($I_\Y$) is the identity operator on $\X$ ($\Y$) and
$\bvarphi\,:\, \X \rightarrow \X\oplus \R$ is a nonlinear map $\bvarphi(x)=\big(\ba(x),1\big)$ defined by  an activation function
$\ba\,:\, \X \rightarrow \X$ (e.g., an elementwise nonlinearity)
  then
  minimizers of \eqref{eqlktddsyeytdsedhjd} are of the form $f(x)=\tilde{w} \bvarphi(x)$ and $v_s(x)=w_s \bvarphi(x)$ where\footnote{Write $\L(\F,\X)$ for the set of linear maps from $\F$ to $\X$ and $\|w_s\|_{\L(\F,\X)}$ for the Frobenius norm of $w_s$.} $\tilde{w}\in \L(\X\oplus \R,\Y)$ and the $w_s \in \L(\X\oplus \R,\X)$
are   minimizers of
\begin{equation}\label{eqjkehbwjhebdhdjq}
\min_{\tilde{w},w_1,\ldots,w_L}\frac{\nu L}{2}\sum_{s=1}^L \|w_s\|_{\L(\X\oplus \R,\X)}^2+ \lambda \|\tilde{w}\|_{\L(\X\oplus \R,\Y)}^2+\ell_\Y\big(f\circ \phi_L(X),Y\big)\,,
\end{equation}
with
\begin{equation}\label{eqkjhebkjhedbd}
f\circ \phi_L (x)=(\tilde{w}\bvarphi)\circ (I+w_L \bvarphi)\circ \cdots \circ (I+w_1 \bvarphi)\,.
\end{equation}
\eqref{eqkjhebkjhedbd} has the structure of one  ResNet block \cite{he2016deep}   and minimizing \eqref{eqjkehbwjhebdhdjq} is equivalent to training the network with scaled/strong $L_2$ regularization\footnote{$L_2$ regularization is
often understood as the sum (or the average) of squared weights over the layers.
The factor $L$ in front of the regularization term, makes it a  stronger type of regularization as $L\rightarrow \infty$. } on weights and biases\footnote{Writing $\bvarphi(x)=\big(\ba(x),1\big)$ has the same effect as using a bias neuron (an always active neuron), therefore $\tilde{w}$ and the $w_s$ incorporate both weights and biases.}. Composing \eqref{eqkjhebkjhedbd} over a hierarchy of spaces (layered in between $\X$ and $\Y$, as described in Sec.~\ref{seckehddjhevd} and \ref{secjhgwyde6dga}) produces input-output functions that have the functional form of artificial neural networks (ANNs) \cite{lecun2015deep} and ResNets.
 If $K$ and $\Gamma$  are reduced equivariant multichannel (REM) kernels (introduced in Sec.~\ref{secrem}) then the input-output functions obtained by composition blocks of the form \eqref{eqkjhebkjhedbd} are  convolutional neural networks (CNNs) \cite{lecun1999object} and their generalization.

\subsection{Idea registration and the continuous limit of warping regression/ResNets}\label{subinfdelim}
Let $C([0,1],\H_\Gamma)$ be the space of continuous functions $v\,:\, \X \times [0,1]\rightarrow \X$ such that $x\rightarrow v(x,t)$ belongs to $\H_\Gamma$ (for all $t\in [0,1]$) and is uniformly (in $t$ and $x$) Lipschitz continuous. For $v\in C([0,1],\H_\Gamma)$ write $\phi^v\,:\, \X\times [0,1]\rightarrow \X$ for the solution of
 \begin{equation}\label{eqflmp}
 \begin{cases}
 \dot{\phi}(x,t)=v\big( \phi(x,t),t\big)&\text{ for }(x,t)\in \X\times [0,1]\\
 \phi(x,0)=x&\text{ for }x\in \X\,.
 \end{cases}
 \end{equation}
 We show (Cor.~\ref{corwjkdkdb736d}) that, in the (infinite-depth/continuous-time limit) limit $L\rightarrow \infty$,  the adherence values (accumulation points) of the minimizers \eqref{eqkjedjdiseu} of \eqref{eqlktddsyeytdsedhjd} are
of the form
\begin{equation}\label{eqkjelkdjendkjdn}
f^\ddagger(\cdot)=f\circ \phi^v(\cdot,1)
\end{equation}
 where $(v,f)$ are minimizers of
\begin{equation}\label{eqlkjgehgddjedhjdB}
\begin{cases}
\text{Minimize } &\frac{\nu}{2}\,\int_0^1 \|v\|_{\Gamma}^2\,dt+ \lambda\,\|f\|_K^2+\ell_\Y\big(f\circ \phi^v(X,1),Y\big)\\
\text{over }&v \in C([0,1],\H_\Gamma)\text{ and }f\in \H_K\,.
\end{cases}
\end{equation}
To prove this we work under  the following regularity conditions\footnote{Note that Cond.~\ref{condeqlkedjehd7d}.(1) is equivalent to the non singularity of $\wK(X,X)$ and (2) implies that
 $(x,x')\rightarrow \wK(x,x')$ and its first and second order partial derivatives are continuous and uniformly bounded.} \ref{condnugget} and \ref{condeqlkedjehd7d}  on the kernel $K$ and $\wK$.
\begin{Condition}\label{condnugget}
Assume that (1)
$x\rightarrow K(x,x')$ is continuous and for all $x'$ (2)
$\X$ and $\Y$ are finite-dimensional.
\end{Condition}
\begin{Condition}\label{condeqlkedjehd7d}
Assume that (1) there exists
 $r>0$ such that  $Z^T \wK(X,X) Z\geq r Z^T Z$ for all $Z \in \X^N$, (2) $\wK$ admits $\F$ and $\psi$ as feature space/map, $\F$ is finite-dimensional, $\psi$ and its first and second order partial derivatives are continuous and uniformly bounded, and (3)
 $\X$ is finite-dimensional.
\end{Condition}
By the Picard-Lindel\"{o}f theorem \cite[Thm.~1.2.3]{arino2006fundamental} the solution of \eqref{eqflmp} exists and is unique if $\X$ is finite-dimensional\footnote{
The simplicity of the proof of existence and uniqueness of solutions for \eqref{eqflmp} is the main reason why we work under Cond.~\ref{condeqlkedjehd7d}. Although \cite[Thm.~3.3]{teixeira2005strong}  could be used when $\dim(\X)=\infty$, the existence and uniqueness of solutions for ODEs  can be quite delicate in general infinite-dimensional spaces \cite{li1975existence}.}, which is ensured by  Cond.~\ref{condeqlkedjehd7d}.

\begin{Corollary}\label{corwjkdkdb736d}
As $L\rightarrow \infty$,
  (1) the minimum value of \eqref{eqlktddsyeytdsedhjd}  converges towards the minimum value of   \eqref{eqlkjgehgddjedhjdB}.
If $(v_1,\ldots,v_L, f)$ is a sequence of minimizers of  \eqref{eqlktddsyeytdsedhjd}
then the set of adherence values
  of $f\circ  (I+v_L)\circ \cdots \circ (I+v_1)$ is
  \begin{equation}
 \big\{f\circ \phi^v(\cdot,1)\mid (v,f) \text{ is a minimizer of }  \eqref{eqlkjgehgddjedhjdB} \big\}\,,
  \end{equation}
 i.e., the sequence $(v_1,\ldots,v_L, f)$ can be partitioned into subsequences such that, along each  subsequence,
   $f\circ  (I+v_L)\circ \cdots \circ (I+v_1)(x)$ converges (for all $x\in \X$) towards $f\circ \phi^v(x,1)$
  where $(v,f)$ is a minimizer of \eqref{eqlkjgehgddjedhjdB}.
\end{Corollary}
\begin{proof}
The proof is a direct consequence\footnote{Observe that we are using the fact that a minimizer $f$ of \eqref{eqlktddsyeytdsedhjd} is unique given $(v_1,\ldots,v_L)$ and a minimizer
$f$ of  \eqref{eqlkjgehgddjedhjdB}  is unique given $v$.} of Thm.~\ref{thmhgw7gfdd}.
\end{proof}

 \eqref{eqlkjgehgddjedhjdB} has the structure of variational formulations used in computational anatomy \cite{grenander1998computational}, image registration \cite{brown1992survey} and shape analysis \cite{younes2010shapes}.
  \begin{figure}[h!]
	\begin{center}
			\includegraphics[width= \textwidth]{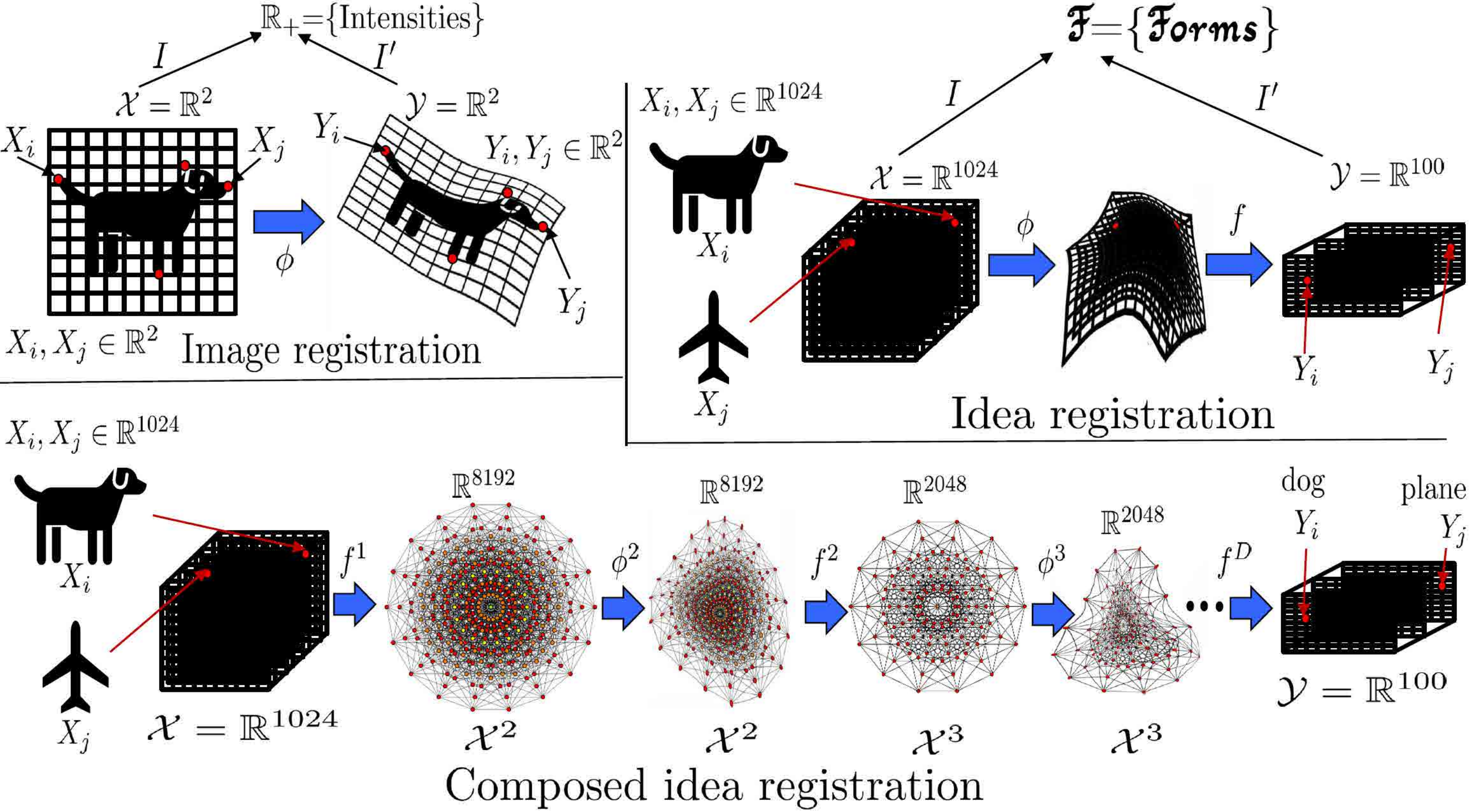}
		\caption{Image registration. Idea registration. Composed idea registration.}\label{figideaformation3}
	\end{center}
\end{figure}
Recall that the core idea of image registration is to represent the image
 of an anatomical structure  as a function $I$ mapping material points in $\X=\Y=\R^2$ to intensities in $\R_+$ (see Fig.~\ref{figideaformation3}).
The distance between an image $I$ and a template $I'$ is then defined by minimizing
\begin{equation}\label{eqlkwjdejhd}
\min_v \nu \int_0^1 \|\Delta v(\cdot,t)\|_{L^2([0,1]^2)}^2\,dt+\|I(\phi^v(\cdot,1))-I'\|_{L^2([0,1]^2)}\,,
\end{equation}
 over diffeomorphisms $\phi^v$ of $\R^2$ driven by the vector field $v$ ($\dot{\phi}^v=v(\phi,t)$) such that $\phi^v(x,0)=x$ \cite{younes1998computable, trouve1998diffeomorphisms}.  The regularizer $\|\Delta v\|_{L^2}$ can  be replaced by higher order Sobolev norms\footnote{\cite{dupuis1998variational} shows that if $\Gamma$ is defined by the Green's function of a differential operator of sufficiently high order in a Sobolev space, then $\phi^v$ is a diffeomorphism (a differentiable bijection). Although the bijectivity of $\phi^v$ is a natural requirement in image registration, it is not needed for idea registration  since two inputs may share the same label.}
 \cite{dupuis1998variational} or the $L^2$ norm of differential operators adapted to the underlying problem
 \cite{ miller2002metrics}.
\emph{Landmark matching} \cite{joshi2000landmark} simplifies the loss \eqref{eqlkwjdejhd} to
\begin{equation}\label{eqkejdgggkejddj}
\min_v \nu \int_0^1 \|\Delta  v\|_{L^2([0,1]^2)}^2\,dt+\sum_i |\phi^v(X_i,1)-Y_i|^2\,,
\end{equation}
where the $X_i$ and $Y_i$ are a finite number of landmark/control (material) points on the two images $I$ and $I'$ (e.g., in Fig.~\ref{figideaformation3}, $X_i$ is the tip of the tail of the first dog and $Y_i$ is the tip of the tail of the second dog).
The variational problem \eqref{eqlkjgehgddjedhjdB} looks like the image registration with landmark matching variational problem \eqref{eqkejdgggkejddj} with a few differences.
The  matching material/landmark  points $(X_i,Y_i)\in \R^2\times \R^2$ are replaced by matching data points $(X_i,Y_i)\in \X\times \Y$.
The deformation $\phi$ is not acting on $\R^2$ but on $\X$, which could be high dimensional. The images $I\,:\, \R^2 \rightarrow \R_+$ and $I'\,:\, \R^2 \rightarrow \R_+ $   are replaced  (see Fig.~\ref{figideaformation3})
by functions $I\,:\, \X \rightarrow \text{\Fp}$ and $I'\,:\, \Y \rightarrow \text{\Fp}$,    which we  call ideas\footnote{The etymology of ``idea'' is (\url{https://www.etymonline.com/word/idea}) {\it  ``mental image or picture''\ldots from Greek idea ``form''\ldots  In Platonic philosophy, ``an archetype, or pure immaterial pattern, of which the individual objects in any one natural class are but the imperfect copies.''}}.
The space of grayscale intensities $\R_+$ is replaced by an abstract space \Fp, which we will call \emph{space of forms} in reference to Plato's theory of forms\footnote{According to Plato's theory of forms the reason why we know that a  particular dog is a dog is that there exists an ideal form (a universal intelligible archetype known as a dog)  and the particular dog is a shadow (as in Plato's cave) or an imperfect copy/projection of that ideal form.
}
 \cite{Platoforms}.
Since the spaces $\X$ and $\Y$ may be distinct,  \eqref{eqlkjgehgddjedhjdB}  composes  the deformation $\phi^v(\cdot,1)\,:\, \X\rightarrow \X$ with the map $f\,:\, \X\rightarrow \Y$ to align the ideas $I\,:\, \X \rightarrow \text{\Fp}$ and $I'\,:\, \Y \rightarrow \text{\Fp}$.
In that sense, \eqref{eqlkjgehgddjedhjdB} (which we call idea registration) compares ideas by creating alignments via  deformations/transformations of  RKHS  spaces\footnote{Credit to \url{https://en.wikipedia.org/wiki/User:Tomruen} for the $N$-cube images in Fig.~\ref{figideaformation3}.}.
Since \eqref{eqjkehbwjhebdhdjq} is a particular case of \eqref{eqlktddsyeytdsedhjd}, the convergence of
\eqref{eqlktddsyeytdsedhjd} towards \eqref{eqlkjgehgddjedhjdB} implies that  ResNets are discretized image/idea registration algorithms (they converge towards \eqref{eqlkjgehgddjedhjdB} in the continuous/infinite-depth limit) with material/landmark points replaced by data points, and images replaced by functions mapping the input/output spaces to an abstract space of forms/shapes which Plato would have called \emph{ideas}\footnote{Plato introduced the intriguing notion that ideas have an actual shape  \cite{Platoforms}.}.
The kernel representation of ResNet blocks as \eqref{eqlktddsyeytdsedhjd} and the identification of ResNets as discretized idea registration problems have several remarkable consequences, which we will highlight in the following sections.

\subsection{Warping kernels}\label{subsecgytref67f6}

The following proposition shows  that solving Problem \ref{pb828827hee} by minimizing \eqref{eqlktddsyeytdsedhjd} (warping regression) or \eqref{eqlkjgehgddjedhjdB} (idea registration)  is equivalent to approximating $f^\dagger$ with a (ridge regression) minimizer of \eqref{eqledhehdiudh} with the kernel $K(x,x')$ replaced by the learned kernel $K^\phi:=K(\phi(x),\phi(x'))$ with $\phi=\phi_L$ or $\phi=\phi^v(\cdot,1)$.   Therefore warping regression and idea registration are equivalent to performing ridge regression in an RKHS  that is learned from the data $(X,Y)$ (the $\nu$ penalty avoids overfitting that RKHS to the data).
Furthermore, for $\ell_\Y$=\eqref{eqjhguyghfuvf4}, warping regression and idea registration are equivalent to estimating $f^\dagger$ with the GP regressor
\begin{equation}\label{eqlkhewkhdehduudhui}
\E_{\xi \sim \cN(0,K^\phi)}\big[\xi(x)\mid\xi(X)+Z=Y\big]=\E_{\xi \sim \cN(0,K)}\big[\xi(\phi^v(x,1))\mid\xi(\phi^v(X,1))+Z=Y\big]\,,
\end{equation}
where $\cN(0,K^\phi)$ is the centered Gaussian process prior with covariance function $K^\phi$ (see Sec.~\ref{subsecuideydiueyd} for presentation of GPs defined by operator-valued kernels).

\begin{Proposition}\label{propdkjehbjdhb}
Let $\phi$ be an arbitrary function mapping $\X$ to $\X$. Let $K^\phi$ be the warped kernel
\begin{equation}
K^\phi(x,x'):=K\big(\phi(x),\phi(x')\big)\,.
\end{equation}
If $f$ is a minimizer of
\begin{equation}\label{eqv1}
\lambda\,\|f'\|_{K}^2+\ell_\Y\big(f'\circ \phi(X),Y\big)
\end{equation}
 over $f'\in \H_K$, then
\begin{equation}
f \circ \phi(\cdot)=f^\phi(\cdot)
\end{equation}
where  $f^\phi$ is a minimizer of
\begin{equation}\label{eqlkedkjehkd}
\lambda\,\|f'\|_{K^\phi}^2+\ell_\Y(f'(X),Y)
\end{equation}
 over $f'\in \H_{K^\phi}$. Furthermore,
\begin{equation}\label{eqnlkdekjddnjd}
\inf_{f'\in \H_K} \lambda\,\|f'\|_{K}^2+\ell_\Y\big(f'\circ \phi(X),Y\big)=\inf_{f'\in \H_{K^\phi}} \lambda\,\|f'\|_{K^\phi}^2+\ell_\Y(f'(X),Y) \,.
\end{equation}
In particular, (1) if $(v_1,\ldots,v_L,f)$  is a minimizer of \eqref{eqlktddsyeytdsedhjd} then $f\circ \phi_L=f^{\phi}$ where $f^{\phi}$ is a minimizer of \eqref{eqlkedkjehkd} with $\phi=\phi_L$, (2) if $(v,f)$  is a minimizer of \eqref{eqlkjgehgddjedhjdB} then $f\circ \phi^v(\cdot,1)=f^{\phi}$ where $f^\phi$ is a minimizer of \eqref{eqlkedkjehkd} with $\phi=\phi^v(\cdot,1)$.
 \end{Proposition}
 \begin{proof}
By the representer theorem minimizers of \eqref{eqv1} and \eqref{eqlkedkjehkd} must be of the form $f=K(\cdot,\phi(X))V$ and $f^\phi=K(\phi(\cdot),\phi(X))W$. Observing that $\|f\|_K^2=V^T K(\phi(X),\phi(X))V$ and $\|f^\phi\|_{K^\phi}^2=W^T K(\phi(X),\phi(X))W$ concludes the proof.
 \end{proof}

 \begin{Remark} {\bf Warping kernels} of
 the form $K(\phi(x),\phi(x'))$ defined by a warping of the space $\phi$ can be traced back to  spatial statistics \cite{sampson1992nonparametric,perrin1999modelling, schmidt2003bayesian, zammit2019deep} where they enable the nonparametric  estimation of  nonstationary and anisotropic spatial covariance structures, and to  numerical homogenization \cite{OwZh:2007a} (where they enable upscaling with non separated scales).
\end{Remark}

\subsection{Existence and uniqueness of minimizers}\label{subsedeses}
We defer existence and uniqueness results on the minimizers of \eqref{eqlktddsyeytdsedhjd} (and therefore \eqref{eqjkehbwjhebdhdjq}) and \eqref{eqlkjgehgddjedhjdB} to Sec.~\ref{seclkjdekjdhjdex} (Thm.~\ref{thmksbsahshkd} and \ref{thmksbsahskd2g}).
Although these variational problems have minimizers, they may not be unique (Sec.~\ref{subseckjjhgdejwedh}), which is why we can only describe convergence in the sense of adherence values.
However, these minimizers will be shown to be unique up to the value of an initial momentum (\eqref{eqkjelkbejdhbd} \eqref{eqlhedjkekdkj}) entering in the kernel representation of $v^1$ and $v(0)$.
For $L_2$-regularized ResNets these results imply (1) that all trained weights and biases are uniquely determined by those of the first layer, (2) the possibility of training with  geodesic shooting (Sec.~\ref{subsecgyf6hv7f6}, \ref{seckdejhdjkdw} and \ref{secfr}) as done in image registration \cite{allassonniere2005geodesic}.

\section{Analysis through Lagrangian/Hamiltonian mechanics}\label{seclkjdekjdhjd}

The rigorous identification of the continuous limit of ResNets/warping regression is based on a  discrete and continuous Lagrangian/Hamiltonian mechanics analysis of warping regression and idea regression. This analysis leads to quantitative and representation results on the trained weights and biases of ResNets and on solutions to warping regression/idea registration problems.
The main steps and results of this analysis are articulated in this section.

\subsection{Ridge regression loss}\label{subseckeljjdh}
The variational problem  \eqref{eqlktddsyeytdsedhjd} can be written
\begin{equation}\label{eqlkjeddjedhjd}
\begin{cases}
\text{Minimize } &\frac{\nu}{2}\,L\sum_{s=1}^L \|v_s\|_{\Gamma}^2+\ell\big((I+v_L)\circ \cdots \circ (I+v_1)(X),Y\big)\\
\text{over }&v_1,\ldots,v_L \in \H_\Gamma\,,
\end{cases}
\end{equation}
where $\ell\,:\, \X^N \times \Y^N \rightarrow [0,\infty]$ is  the
   {\bf ridge regression loss} \eqref{eqledhehdiudh}=\eqref{eqlkjdkjweedjkb}.
 Condition \ref{condnugget} (which we will from now on assume to be satisfied) ensures the continuity of $\ell$=\eqref{eqlkjdkjweedjkb}.
We will now focus on the reduction of \eqref{eqlkjeddjedhjd} and only assume
$\ell\,:\, \X^N \times \Y^N \rightarrow [0,\infty]$  to be continuous and positive.

\subsection{Discrete least action principle}\label{subsecgyf67f6}
Write $q^1:=X$ and for $s\in \{2,\ldots,L\}$ write
\begin{equation}
q^{s+1}:=\phi_s(q^1)
\end{equation}
for the image of the input data $X$ under the discrete flow
\begin{equation}\label{eqkjehdbehdhjbdphis}
\phi_s:=(I+v_s)\circ \cdots \circ (I+v_1)\,.
\end{equation}
  Although $\ell$ may not be convex, the first part of \eqref{eqlkjeddjedhjd} is quadratic and can be reduced a discrete least action principle on
 $q^1,\ldots,q^{L+1} \in \X^N$.
To prove this, we will
from now on, work under Condition \ref{condeqlkedjehd7d}.

\begin{Theorem}\label{thmeqkljedjnedjd}
 $v_1,\ldots,v_L \in \H_\Gamma$ is  a minimizer of \eqref{eqlkjeddjedhjd} if and only if\footnote{Write $\Gamma(q^s,q^s)$ for the $N\times N$ block matrix with blocks $\Gamma(q^s_i,q^s_j)$, and $\Gamma(\cdot,q^s)$ for the $1\times N$ block vector with blocks $\Gamma(\cdot,q^s_i)$.}
\begin{equation}\label{eqkjdkedkjnd}
v_s(x)= \wK( x, q^s) \wK(q^s,q^s)^{-1} (q^{s+1}-q^s) \text{ for } x\in \X, s\in \{1,\ldots,L\}\,,
\end{equation}
where   $q^1,\ldots,q^{L+1} \in \X^N$ is a   minimizer  of   ($\Delta t:=1/L$)
\begin{equation}\label{eqlsedhjd}
\begin{cases}
\text{Minimize } &\frac{\nu}{2}  \sum_{s=1}^L  (\frac{q^{s+1}-q^s}{\Delta t})^T \wK(q^s,q^s)^{-1} (\frac{q^{s+1}-q^s}{\Delta t})\, \Delta t+
\ell\big(q^{L+1},Y\big)\\
\text{over } &q^2,\ldots,q^{L+1} \in \X^N
\text{ with }q^1=X\,.
\end{cases}
\end{equation}
\end{Theorem}
\begin{proof}
Introduce the variables $q^{s+1}_i=(I+v_s)\circ \cdots \circ (I+v_1)(X_i)$ for $2 \leq s \leq L$, and
$q^1_i=X_i$. \eqref{eqlkjeddjedhjd} is then equivalent to
\begin{equation}\label{eqkljedjnedjd}
\begin{cases}
\text{Minimize } &\nu\frac{L}{2}\sum_{s=1}^L \|v_s\|_{\Gamma}^2+\ell\big(q^{L+1},Y\big)\\
\text{over }&v_1,\ldots,v_s\in \H_\Gamma,\quad q^1,\ldots,q^{L+1} \in \X^N \\
\text{subject to }& q^1=X\text{ and } v_s(q^{s})=q^{s+1}-q^{s}\text{ for all }s
\end{cases}
\end{equation}
Minimizing with respect to the $v_s$ first we obtain
$\|v_s\|_{\Gamma}^2=(q^{s+1}-q^s)^T \wK(q^s,q^s)^{-1} (q^{s+1}-q^s)$ and \eqref{eqkjdkedkjnd}.
\eqref{eqkljedjnedjd} can then be reduced to \eqref{eqlsedhjd}.
\end{proof}

\begin{Remark}
The introduction of the intermediate variables $q^s$ (tracking the propagation of the input data $X$ across layers of the network in the proof Thm.~\ref{thmeqkljedjnedjd}) is generic. Similar intermediate variables are also introduced in  \cite{chen2021solving} to generalize GP methods to the solving and learning of arbitrary nonlinear PDEs (with guaranteed convergence) and in \cite{owhadi2021computational} to introduce a computational graph completion (CGC) framework\footnote{The CGC framework includes ANNs as a particular case and generalizes solving linear systems of equations to that of solving undetermined nonlinear systems of equations with a computational graph encoding imperfectly known dependencies between variables and functions.} for generating, organizing and reasoning with computational knowledge.
\end{Remark}

\subsection{Continuous limit and neural least action principle}\label{subsecgeyf67f6}
Interpreting $\Delta t=1/L$ as the time step, \eqref{eqlsedhjd} is the discrete least action principle  \cite{marsden2001discrete} obtained by using the approximation
$$(\frac{q^{s+1}-q^s}{\Delta t})^T \wK(q^s,q^s)^{-1} (\frac{q^{s+1}-q^s}{\Delta t}) \approx \dot{q}_{\frac{s}{L}}^T \wK(q_s,q_s) \dot{q}_{\frac{s}{L}}$$ in the continuous least action principle
\begin{equation}\label{eqlsweedhsejd}
\begin{cases}
\text{Minimize } &\nu\,\A[q]+\ell\big(q(1),Y\big)\\
\text{over } &q \in C^1([0,1],\X^N) \text{ subject to } q(0)=X\,.
\end{cases}
\end{equation}
where $\A[q]$ is the action
\begin{equation}\label{eqactq}
\A[q]:=\int_0^1 \Lf(q,\dot{q}) \,dt\,
\end{equation}
defined by the Lagrangian
\begin{equation}\label{eqejkhekffkjefkfj}
\Lf(q,\dot{q}):= \frac{1}{2}\dot{q}^T \wK(q, q)^{-1} \dot{q}\,,
\end{equation}
and $C^1([0,1],\X^N)$ is the set of continuously differentiable functions $q\,: \, [0,1]\rightarrow \X^N$ mapping
$s\in [0,1]$ to $q_s\in \X^N$.
Consequently, minimizing \eqref{eqkljedjnedjd} corresponds to
using a first-order variational symplectic integrator (simulating a nearby mechanical system \cite{hairer2006geometric}) to approximate
\eqref{eqlsweedhsejd}. We will present convergence results in Thm.~\ref{thmhgw7gfdd}.

\subsection{Euler-Lagrange equations and geodesic motion.}\label{subsecgreeyf67f6}
Following classical Lagrangian mechanics \cite{marsden2013introduction}, a minimizer of
\eqref{eqlsweedhsejd} follows the Euler-Lagrange equations $\frac{d}{dt} \frac{\partial \Lf }{\partial \dot{q}}-\frac{\partial \Lf}{\partial q}=0$, i.e.
\begin{equation}\label{eqklkedjdd}
 \frac{d}{dt}\big(\wK(q,q)^{-1} \dot{q}\big)= \partial_q \big(\frac{1}{2}\dot{q}^T \wK(q, q)^{-1} \dot{q}\big)
\end{equation}
Furthermore,
$\wK^{-1}(q,q)$ can be interpreted as a mass matrix or metric tensor \cite[p.~3]{marsden2013introduction} and the Euler-Lagrange equations
are equivalent to the equations of geodesic motion \cite[Sec.~7.5]{marsden2013introduction} corresponding to minimizing the length $\int_0^1 \sqrt{\dot{q}^T \wK(q, q)^{-1} \dot{q}}\,ds$ of the curve $q$ connecting $X$ to $q(1)$ (which, using the equivalence between minimizing length and length squared, can also be recovered as a limit by replacing $\frac{L}{2}\sum_{s=1}^L \|v_s\|_{\Gamma}^2$ by $\sum_{s=1}^L \|v_s\|_{\Gamma}$ in \eqref{eqlkjeddjedhjd}).

\subsection{Hamiltonian mechanics.}\label{subsewecgyf67f6}
Introduce the momentum variable
\begin{equation}\label{eqlhedjkekdkj}
p=\frac{\partial \Lf}{\partial \dot{q}}= \wK(q,q)^{-1} \dot{q}\,,
\end{equation}
and the Hamiltonian ($\Hf(q,p)= p^T \dot{q} -\Lf(q,\dot{q})=\frac{1}{2}\dot{q}^T \wK(q, q)^{-1} \dot{q}$).
\begin{equation}\label{eqkjbdejhdbeyudbs}
\Hf(q,p)=  \frac{1}{2}p^T \wK(q, q) p\,.
\end{equation}

The following theorem summarizes the classical \cite{marsden2013introduction} correspondence between the Lagrangian and Hamiltonian viewpoints.
\begin{Theorem}\label{thmlkndjd}
If $q$ is a minimizer of the least action principle \eqref{eqlsweedhsejd} then $(q,p)$ follows the Hamiltonian dynamic \begin{equation}\label{eqkedmdledkemdl}
\begin{cases}
&\dot{q}=\frac{\partial \Hf(q,p)}{\partial p}=\wK(q,q) p\\
&\dot{p}=-\frac{\partial \Hf(q,p)}{\partial q}=- \partial_q ( \frac{1}{2} p^T  \wK(q, q) p)\,,
\end{cases} \text{with initial value }(q(0)=X,p(0))\,.
\end{equation}
The energy $\Hf(q,p)$ is conserved by this dynamic and any function  $F$ of $(q,p)$ evolves according to the Lie derivative
\begin{equation}
\frac{d}{dt} F(q,p)=\{F,\Hf\}=\partial_q F \partial_p \Hf -\partial_p F \partial_q \Hf=\partial_q F \wK(q,q) p -\partial_p F \partial_q ( \frac{1}{2} p^T  \wK(q, q) p) \,.
\end{equation}
\end{Theorem}

\subsection{Near energy conservation and near $L_2$-norm conservation of trained weights and biases across layers}\label{subsnearenerpres}
If $q^1,\ldots,q^{L+1} \in \X^N$ is a   minimizer  of \eqref{eqlsedhjd}, then,  introducing the momentum variables
 \begin{equation}\label{eqkjelkbejdhbd}
p^s=\wK(q^s,q^s)^{-1} \frac{q^{s+1}-q^s}{\Delta t}\,,
\end{equation}
 $(q^s,p^s)$ follows the discrete Hamiltonian dynamics
 \begin{equation}\label{ejkhdbejhdbd}
\begin{cases}
q^{s+1}&= q^s+\Delta t\, \wK(q^s,q^s) p^s\\
p^{s+1}&=p^s-\frac{\Delta t}{ 2} \partial_{q^{s+1}}\big((p^{s+1})^T \wK(q^{s+1},q^{s+1}) p^{s+1}\big)\,,
\end{cases}
\end{equation}
  and
  the near energy preservation of variational integrators  \cite{marsden2001discrete,hairer2006geometric} implies (Thm.~\ref{thmksbsahskd2g}) that the norms $\|w_s\|_{\L(\F,\X)}^2$ of minimizers of \eqref{eqjkehbwjhebdhdjq} (of weights and biases of ResNet blocks after training with scaled/strong $L_2$ regularization) are nearly constant (fluctuate by at most $\mathcal{O}(1/L)$) across $i\in \{1,\ldots,L\}$.

\subsection{Hamiltonian mechanics in feature space}\label{subgfcsecgyf67f6} Let $\F$ and $\psi$ be a feature space/map of $\wK$ as  in Cond.~\ref{condeqlkedjehd7d}. Using the identity $\Gamma(x,x')=\psi^T(x)\psi(x')$, the Hamiltonian system \eqref{eqkedmdledkemdl}
  can be written
 \begin{equation}\label{eqkedmdledkemdlN}
\begin{cases}
\dot{q}_i =&  \psi^T(q_i) \alpha \\
\dot{p}_i=&-  \partial_x \big(  p_i^T  \psi^T(x) \alpha \big)\Big|_{x=q_i}\,,
\end{cases}
\end{equation}
where $\alpha$ is the time dependent element of $\F$ defined by
\begin{equation}\label{eqkedmdbhledkemdlN3}
\alpha:=\sum_{j=1}^N \psi(q_j) p_j\,.
\end{equation}
Energy preservation and the identity
$
\|\alpha\|_\F^2=p^T \Gamma(q,q)p\,,
$
implies the following.
\begin{Proposition}\label{propkjhdegd7}
$t\rightarrow \|\alpha(t)\|_\F$ is  constant.
\end{Proposition}

\begin{Remark}
\eqref{eqkedmdledkemdlN} suggests that $p(t)$ is the adjoint of $q(t)$ as defined in the Neural ODE literature \cite[Equ.~4]{chen2018neural}.
The numerical experiments of Sec.~\ref{subsefcfcdcgyf67f6} show that the vectors $p(1)$ and $p(0)$ are dominated  by a few of their entries
and support the suggestion that momentum variables promote sparsity in the representation of the regressor.
This observation suggests that the adjoint introduced in the Neural ODE literature may not only promote memory efficiency by avoiding the storage of ``any intermediate quantities of the forward pass'' \cite[p.~1]{chen2018neural} but also through its sparsity.
This sparsity of the momentum map representation (and of the adjoint) is very similar to the sparsity of image deformations in momentum map (adjoint) representation \cite{bruveris2011momentum} observed and discussed in  \cite{vialard2012diffeomorphic, fishbaugh2013geodesic}.
\end{Remark}

\subsection{Existence and uniqueness}\label{sugygbsecgyf67f6}
 Cond.~\ref{condeqlkedjehd7d} provides sufficient regularity on $\wK$ for the existence and uniqueness of a solution to \eqref{eqkedmdledkemdl} in $C^2([0,1],\X^N)\times C^1([0,1],\X^N)$.

\begin{Theorem}\label{thml2kjej2ke}
 \eqref{eqkedmdledkemdl}  admits a unique solution in $C^2([0,1],\X^N)\times C^1([0,1],\X^N)$.
\end{Theorem}
\begin{proof}
\eqref{eqkedmdledkemdlN} implies that $\|\dot{p}_i\|_\Y \leq  \|p_i\|_\Y \|\alpha\|_\F \sup_x \|\nabla \psi(x)\|$.
Therefore Prop.~\ref{propkjhdegd7} implies that $p(t)$ remains in a bounded domain $B$ (for $t\in [0,1]$).
 The regularity of $\Gamma$ (Cond.~\ref{condeqlkedjehd7d})  implies that the vector field of \eqref{eqkedmdledkemdl} is uniformly Lipschitz for $p\in B$.  We conclude from the global version of the Picard-Lindel\"{o}f theorem \cite[Thm.~1.2.3]{arino2006fundamental}.
\end{proof}

\subsection{Geodesic shooting}\label{subsecgyf6hv7f6}
The Hamiltonian representation of  minimizers of \eqref{eqlsweedhsejd} enables its reduction to the search for an initial momentum $p(0)$. This method, known as geodesic shooting in image registration \cite{allassonniere2005geodesic}, is summarized in the following theorem.

\begin{Theorem}\label{thmlllakbhhd}
Write $p=\wK(q,q)^{-1}\dot{q}=$\eqref{eqlhedjkekdkj} for   $q\in  C^1([0,1],\X^N)$.
$q$ is a minimizer of
 \eqref{eqlsweedhsejd} if and only if $(q,p)$  follows the Hamiltonian dynamic  \eqref{eqkedmdledkemdl}, $q(0)=X$ and $p(0)$ is  a minimizer of
 \begin{equation}\label{eqkjlkjkejwkdj}
\Vk\big(p(0),X,Y\big):= \frac{\nu}{2}p^T(0) \wK\big(X,X\big) p(0)+ \ell(q(1),Y)\,.
 \end{equation}
Furthermore, $p(1)$ satisfies
\begin{equation}\label{eqkjdjhebdjhd}
\nu\, p(1) + \partial_{q(1)}\ell(q(1),Y)=0\,.
\end{equation}
\end{Theorem}
\begin{proof}
 Zeroing the Fr\'{e}chet derivative of \eqref{eqlsweedhsejd} with respect to the trajectory $q(t)$ implies  that a minimizer of \eqref{eqlsweedhsejd} must satisfy the Hamiltonian dynamic  \eqref{eqkedmdledkemdl} and the boundary condition \eqref{eqkjdjhebdjhd} (which is analogous to the one obtained in image registration \cite[Eq.~7]{allassonniere2005geodesic}).
 Since the energy $p^T \wK(q,q) p/2$ is preserved along the Hamiltonian flow, the minimization of \eqref{eqlsweedhsejd} can be reduced to that of \eqref{eqkjlkjkejwkdj} with respect to $p(0)$.
\end{proof}

\subsection{Idea registration.}\label{secjhghgyuygy}

Instead of reducing \eqref{eqlkjeddjedhjd} to \eqref{eqkljedjnedjd}, consider its infinite depth limit
 and observe that, in the limit $L\rightarrow \infty$, $(I+v_k)\circ \cdots \circ (I+v_1)$ approximates (at time $t_k:=\frac{k}{L}$) the flow map
   $\phi^v$
 where $v$ is a
    minimizer of\footnote{Observe that, with $\ell=$\eqref{eqledhehdiudh}, \eqref{eqlkjgehgddjedhjd} is equivalent to minimizing \eqref{eqlkjgehgddjedhjdB}.}
\begin{equation}\label{eqlkjgehgddjedhjd}
\begin{cases}
\text{Minimize } &\frac{\nu}{2}\,\int_0^1 \|v\|_{\Gamma}^2\,dt+\ell\big(\phi^v(X,1),Y\big)\\
\text{over }&v \in C([0,1],\H_\Gamma)\,.
\end{cases}
\end{equation}
The  proof of this convergence, stated  in  Thm.~\ref{thmhgw7gfdd}, is based on the following reduction theorem which establishes that,  minimizers of \eqref{eqlkjgehgddjedhjdB} have, as in landmark matching \cite{joshi2000landmark},
 the representation
  \begin{equation}\label{eqttjjkjtc}
\dot{\phi}^v(x,t)= \wK(\phi^v(x,t),q) p\,,
\end{equation}
where the position and momentum  variables $(q,p)$ are  in $\X^N\times \X^N$, started from $q(0)=X$, and following the dynamic \eqref{eqkedmdledkemdl} defined by the  Hamiltonian \eqref{eqkjbdejhdbeyudbs}.
Therefore (Thm.~\ref{thmksbsahshkd}), the norm $\|v(\cdot,t)\|^2_\Gamma$ (of the weights and biases in the continuous infinite depth limit) must be a constant over $t\in [0,1]$. Furthermore \eqref{ejkhdbejhdbd} is a first-order variational/symplectic integrator
 for approximating the Hamiltonian flow of \eqref{eqkjbdejhdbeyudbs}.

\begin{Theorem}\label{thmkljedjnedesdjd}
  $v$ is a minimizer of \eqref{eqlkjgehgddjedhjd} if and only if
\begin{equation}\label{eqtttc}
\dot{\phi}^v(x,t)= \wK(\phi^v(x,t),q_t) \wK(q_t,q_t)^{-1}\dot{q}_t \text{ with } \phi^v(x,0)=x\in \X\,
\end{equation}
where $q$ is a minimizer of the least action principle \eqref{eqlsweedhsejd}. Furthermore (defining $\A[q]$ as in \eqref{eqactq}),
for $q \in C^1([0,1],\X^N)$,
\begin{equation}\label{eqkjkjhdjhe}
\A[q]=\inf_{v\in C([0,1],\H_\Gamma)\,:\, \phi^v(q(0),t)=q(t)\,\forall t\in [0,1]} \int_0^1 \frac{1}{2}\|v\|_{\Gamma}^2\,dt\,,
\end{equation}
and the representer PDE \eqref{eqtttc} can be written as \eqref{eqttjjkjtc},
where $(q,p)$ is the solution of the Hamiltonian system \eqref{eqkedmdledkemdl} with initial condition $q(0)=X$ and $p(0)$ identified as a minimizer of \eqref{eqkjlkjkejwkdj}.
\end{Theorem}
\begin{proof}
The proof of \eqref{eqtttc} and \eqref{eqkjkjhdjhe} is identical to that of  Thm.~\ref{thmeqkljedjnedjd}. \eqref{eqttjjkjtc} follows from
theorems \ref{thmlkndjd} and \ref{thmlllakbhhd}.
\end{proof}

Figure \ref{figmr} summarizes the correspondence between the least  action principles obtained from
\eqref{eqlkjeddjedhjd} under reduction and/or infinite depth limit.

 \begin{figure}[h!]
	\begin{center}
			\includegraphics[width= \textwidth]{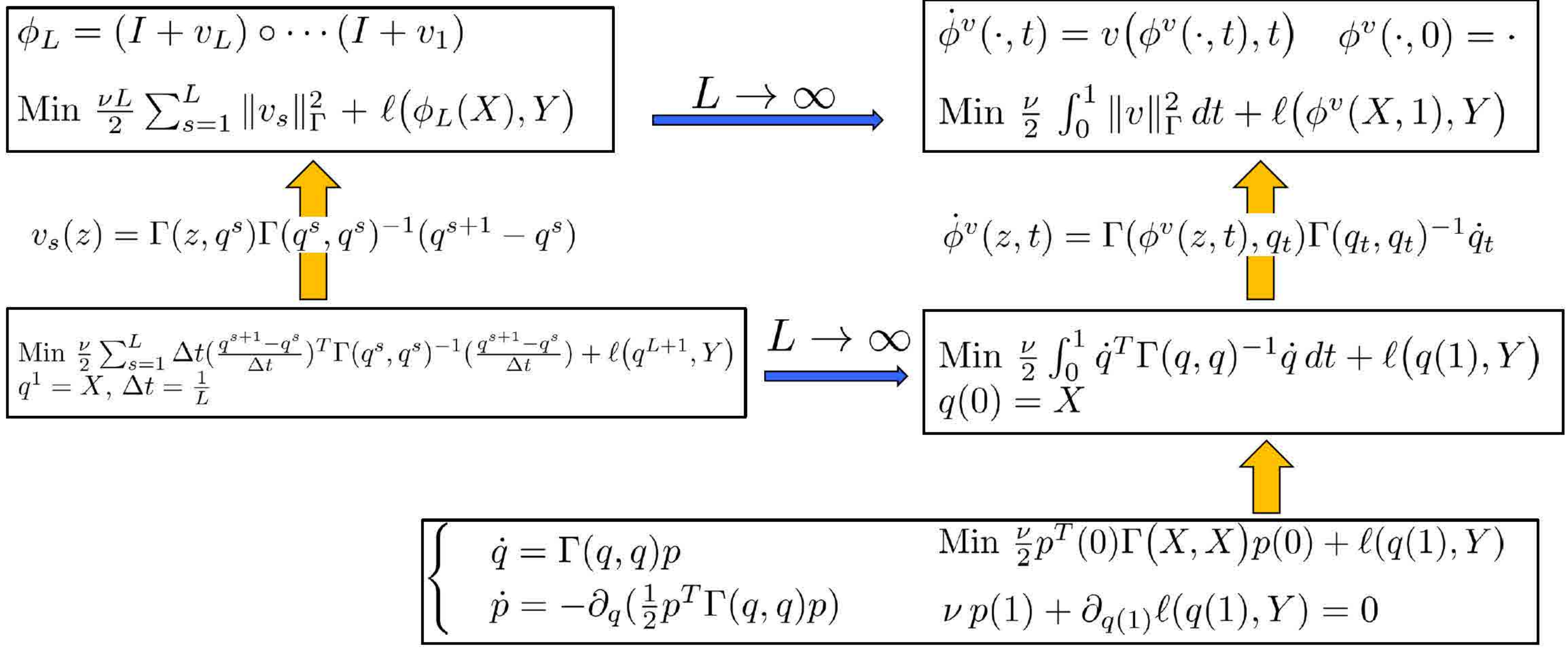}
		\caption{Least action principles  after reduction and/or infinite depth limit.}\label{figmr}
	\end{center}
\end{figure}

\subsection{Information in momentum variables and sparsity}\label{secspa}

Consider the Hamiltonian system \eqref{eqkedmdledkemdl}. While  $q(t)$ has a clear interpretation as the displacement of the input data $X$ (at layer $t$ of the continuous limit network), that of the momentum variable $p$ is less transparent.
  The following theorem shows that the entry $p_i$ of $p$ is zero if $q_i(1)$ does not contribute to the loss. Therefore, as with support vector machines \cite{steinwart2008support}, if $\ell_\Y$ is the
 {\bf hinge loss}
\begin{equation}\label{eqlklnkjndejd}
\ell_\Y(Y',Y)=\sum_{i=1}^N  \big(Y_{i,\text{class}(Y_i)}'-\max_{j\not=\text{class}(Y_i)}Y_{i,j}'-1\big)_+\,,
\end{equation}
 used for classification problems\footnote{\eqref{eqlklnkjndejd} seeks to maximize the margin between correct and incorrect labels and is defined for $\Y=\R^{d_\Y}$ by writing $Y_{i,j}'$ for the entries of $Y_i'$, using $\text{argmax}_j Y_{i,j}'$ for the predicted label for the data $i$,
 setting $\text{class}(Y_i)=j$ if the label/class of $X_i$ is $j$ and writing  $a_+:=\max(a,0)$.}, then the only points $X_i$ with non zero momentum are those for which
  $\phi^v(X_i,1)$ is included in the (hinge loss) margin.
  In that sense $p_i$ represents the contribution of the data point $(X_i,Y_i)$ to the predictor $\phi^v(\cdot,1)$ obtained from
  \eqref{eqlkjgehgddjedhjd} and \eqref{eqttjjkjtc}.
 This phenomenon, clearly illustrated in Fig.~\ref{figspiral2}, is analogous to the sparse representations obtained with support vector machines \cite{steinwart2008support} where
the predictor is represented with the subset of the training points (the support vectors) within the safety margin of the hinge loss.

\begin{Theorem}\label{themkjehdbeddh}
Let $(q,p)$ be  the solution of the Hamiltonian system \eqref{eqkedmdledkemdl} with initial state $q(0)=X$ and $p(0)$ minimizing \eqref{eqkjlkjkejwkdj}.
For $i\in \{1,\ldots,N\}$, it holds true that $p_i(t)=0$ for all $t\in [0,1]$ if and only if $\partial_{q_i(1)} \ell(Y,q(1))=0$.
\end{Theorem}
\begin{proof}
Combine   \eqref{eqkjdjhebdjhd} with  Lem.~\ref{lemkjhedbhdbdd}.
\end{proof}
\begin{Lemma}\label{lemkjhedbhdbdd}
Let $(q,p)$ be a solution of the Hamiltonian system \eqref{eqkedmdledkemdl}. If $p_i(t_0)=0$ for some $t_0\in [0,1]$ then $p_i(t)=0$ for all $t\in [0,1]$.
\end{Lemma}
\begin{proof}
$\dot{p}=- \partial_q ( \frac{1}{2} p^T  \wK(q, q) p)$ implies that $\dot{p}_i(t)=0$ if $p_i(t)=0$ and $p_i(t)=0$ for  $t\geq t_0$ follows by integration. Since the time reversed trajectory $t \rightarrow \big(q,-p\big)(1-t)$ also satisfies the Hamiltonian system \eqref{eqkedmdledkemdl} the result also follows by integration for $t\in [0,t_0]$.
\end{proof}

\section{Regularization}\label{secregular}

The landmark matching \cite{joshi2000landmark} setting of Sec.~\ref{seclkjdekjdhjd}  requires  non-overlapping  data, and minimizers, and minimal values obtained from that setting may depend non-continuously  on the input $X$ (since $\Gamma(X,X)$ will become singular as $X_i \rightarrow X_j$ for some $i\not=j$). To ensure continuity and avoid singularities, idea registration must be regularized as it is commonly done in image registration \cite{micheli2008differential}. The proposed regularization, introduced and analyzed in this section, provides an alternative to Dropout for ANNs \cite{srivastava2014dropout}.

\subsection{ResNets/ANNs are brittle because Bayesian inference is brittle}
Minimal values and minimizers of  \eqref{eqlktddsyeytdsedhjd} and \eqref{eqlkjgehgddjedhjdB} may not be continuous in the data $X$. Furthermore,
the observation of Sec.~\ref{subsecgytref67f6} that warping regression can be interpreted as performing ridge regression with a data-dependent prior suggests Bayesian brittleness (the extreme lack of robustness of Bayesian posterior values with respect to the prior \cite{owhadi2015brittlenessa, owhadi2015brittlenessb, owhadi2013brittleness}) as a cause for the high sensitivity of ANNs with respect to the testing data $x$ or the training data $X$ reported in \cite{szegedy2013intriguing} (this lack of stability was predicted in \cite{brittlnessmachinelearning} based on \cite{owhadi2015brittlenessa}).
This fragility endures even if the training data $X$ is randomized \cite{owhadi2017qualitative} and may not be resolved without loss of accuracy since robustness and accuracy/consistency are conflicting requirements \cite{owhadi2015brittlenessb, owhadi2017qualitative}.
The Hamiltonian representation \eqref{eqkjdkedkjnd} of minimizers of warping regression/idea registration  variational problems suggests Hamiltonian chaos \cite{casetti1996riemannian} as another cause of the instability of ResNets/ANNs (from this dynamical  perspective, the instability of ResNets is related to the curvature
fluctuations of the metric defined by $\Gamma(q,q)$ \cite{casetti1996riemannian}) and that  Lyapunov characteristic exponents could also be used as a measure of instability for ANNs.

\subsection{A simple  rigorous regularization strategy}\label{subsecreg1}
To ensure continuity,  \eqref{eqlktddsyeytdsedhjd} and \eqref{eqlkjgehgddjedhjdB} must be regularized  and we now generalize  an image registration regularization  strategy \cite{micheli2008differential}
 to idea registration.
The proposed regularization strategy can be summarized as approximating $f^\dagger$
with $f^\ddagger=f\circ \phi_L=$\eqref{eqkjedjdiseu} where $\phi_L=(I+v_L)\circ \cdots \circ (I+v_1)=\eqref{eqkjehdbehdhjbd}$, and  $(v_1,\ldots,v_L,f)$ are identified by minimizing the following regularized version of \eqref{eqlktddsyeytdsedhjd}.
  \begin{equation}\label{eqlktddsyeytdsedhjdreg}
\begin{cases}
\text{Minimize } &\frac{\nu}{2}\,L\sum_{s=1}^L \big(\|v_s\|_{\Gamma}^2+ \frac{1}{r} \|q^{s+1}-(I+v_s)(q^s)\|_{\X^N}^2\big)
\\&+\lambda\,\big(\|f\|_{K}^2+\frac{1}{\rho} \|f(q^{L+1})-Y'\|_{\Y^N}^2 \big)+\ell_\Y\big(Y',Y\big)\\
\text{over }&v_1,\ldots,v_L \in \H_\Gamma,\, f\in \H_K,\, q^1,\ldots,q^{L+1} \in \X^N,\, q^1=X,\, Y'\in \X^N\,,
\end{cases}
\end{equation}
where  $r,\rho>0$ are regularization parameters (akin to the nuggets employed in Kriging/spatial statistics \cite{schafer2021sparse}) and
$\|X\|_{\X^N}:=\sum_{i=1}^N \|X_i\|_\X^2$, $\|Y\|_{\Y^N}:=\sum_{i=1}^N \|Y_i\|_\Y^2$.
Note that for $\rho\downarrow 0$ and $\ell_\Y=\eqref{eqjhguyghfuvf4}$, \eqref{eqlktddsyeytdsedhjdreg} reduces to
 \begin{equation}\label{eqlktddsyeytdsedhjdregintro}
\begin{cases}
\text{Minimize } &\frac{\nu}{2}\,L\sum_{s=1}^L \big(\|v_s\|_{\Gamma}^2+ \frac{1}{r} \|q^{s+1}-(I+v_s)(q^s)\|_{\X^N}^2\big)
\\&+\lambda\,\|f\|_{K}^2+ \|f(q^{L+1})-Y\|_{\Y^N}^2 \\
\text{over }&v_1,\ldots,v_L \in \H_\Gamma,\, f\in \H_K,\, q^2,\ldots,q^{L+1} \in \X^N,\, q^1=X\,,
\end{cases}
\end{equation}
and $r>0$ relaxes the constraint that the
input data $X$ must propagate without error through each layer of the network. Note that as $r\downarrow 0$,  the trajectory defined by $q^s$ satisfies  $q^{s+1}=q^s+v_s(q_s)$ (with $q^1=X$).
Using the computational graph representation of Sec.~\ref{secresnetb}, for $L=3$, the solution obtained by minimizing \eqref{eqlktddsyeytdsedhjdregintro} can be identified by replacing the $v_i$ with $\cN(0,\Gamma)$ GPs, $f$ with a $\cN(0,K)$ GP and computing their MAP estimators  given the structure/data represented by the following computational graph,\\
\centerline{
\begin{tikzpicture}[->,>=stealth',shorten >=1pt,auto,node distance=2.5cm,
                    thick,main node/.style={rectangle,draw,font=\sffamily\Large\bfseries}]

\node[main node] (1) {$x$};
\node[main node] (2) [right of=1]  {$q_2$};
\node[main node] (3) [right of=2] {$q_3$};
\node[main node] (4) [right of=3] {$q_4$};
\node[main node] (5) [right of=4] {$y$};
\node[main node] (6) [above of=2,blue,node distance=1.5cm] {$z_1$};
\node[main node] (7) [above of=3,blue,node distance=1.5cm] {$z_2$};
\node[main node] (8) [above of=4,blue,node distance=1.5cm] {$z_3$};
\node[main node] (9) [above of=5,blue,node distance=1.5cm] {$z$};

\path[every node/.style={font=\sffamily\Large\bfseries}]
    (1) edge node [below ] {} (2)
    (1) edge  [bend left,red] node[above ] {$v_1$} (2);

\path[every node/.style={font=\sffamily\Large\bfseries}]
    (2) edge node [below ] {} (3)
    (2) edge  [bend left,red] node[above ] {$v_2$} (3);

\path[every node/.style={font=\sffamily\Large\bfseries}]
    (3) edge node [below ] {} (4)
    (3) edge  [bend left,red] node[above ] {$v_3$} (4);

\path[every node/.style={font=\sffamily\Large\bfseries},red]
    (4) edge node [below ] {$f$} (5);

\path[every node/.style={font=\sffamily\Large\bfseries}]
    (6) edge node [right ] {} (2);

\path[every node/.style={font=\sffamily\Large\bfseries}]
    (7) edge node [right ] {} (3);

\path[every node/.style={font=\sffamily\Large\bfseries}]
    (8) edge node [right ] {} (4);

\path[every node/.style={font=\sffamily\Large\bfseries}]
    (9) edge node [right ] {} (5);


\tikzstyle{every to}=[draw,dashed]
\draw[dashed] (1) to[out=-90,in=-90,looseness=0.3,style={font=\sffamily\Large,dashed}] node[above] {$(X,Y)$} (5);

\end{tikzpicture}
}
which can be unpacked as $q_2=v_1(x)+x+z_1$, $q_3=v_2(q_2)+q_2+z_2$, $q_4=v_3(q_3)+q_3+z_3$, $y=f(q_4)+z$ where $z_1,z_2,z_3,z$ are $\cN(0,r I_\Y)$ Gaussian random variables.

\subsection{Regularized ResNets}

When performed with activation functions as in the setting of Sec.~\ref{subresnetpar} ($\Gamma(x,x')=\bvarphi^T(x) \bvarphi(x') I_\X$ and $K(x,x')=\bvarphi^T(x) \bvarphi(x') I_\Y$ with $\bvarphi(x)=\big(\ba(x),1\big)$), the proposed regularization provides a principled alternative to Dropout\footnote{Although dropout does not appear to change the training loss function, the  stochasticity introduced in the network edges implies that the network is trained with an effective loss in which the output values (at all layers) are, as with our proposed approach, a stochastic perturbation of those of the testing map.} for ANNs \cite{srivastava2014dropout}.
In the setting of one ResNet block, this regularization does not change the functional form \eqref{eqkjhebkjhedbd} of the block but replaces (Thm.~\ref{thmidjheyd88h2ddde}) the training \eqref{eqjkehbwjhebdhdjq} of the weights and biases by the minimization of
\begin{equation}\label{eqlkjgedehgddjededdeddddsdddsshjssdBN2intro}
\begin{split}
\min_{w^s, \tilde{w}, q^s}
&\frac{\nu L}{2} \sum_{s=1}^L \big(\|w^s\|_{\L(\X\oplus \R,\X)}^2
+\frac{1}{r} \|q^{s+1}-q^s -
w^s \bvarphi(q^s)\|_{\X^N}^2
\big)
\\&+
\lambda\,\|\tilde{w}\|_{\L(\X\oplus \R,\Y)}^2+\|\tilde{w}\bvarphi(q^{L+1})-Y\|_{\Y^N}^2\,.
\end{split}
\end{equation}
 Note that training with regularization is equivalent to replacing the exact propagation $q^{s+1}=q^s+w^s \bvarphi(q^s)$   of the input data by $q^{s+1}=q^s+w^s \bvarphi(q^s) +Z^s$  ($Y'=\tilde{w}\bvarphi(q^{L+1})+Z$) where the $Z^s$ and $Z$ are propagation error variables ($Z^s\in \X^N$, $Z\in \Y^N$) whose norms are added to the total loss at the training stage.
 Indeed minimizing \eqref{eqlkjgedehgddjededdeddddsdddsshjssdBN2intro} is equivalent to minimizing
\begin{equation}\label{eqlkjgedehgddjededdeddddsdddsshjssdBN2intro2}
\begin{cases}
&\underset{w^s, \tilde{w}, q^s, Z^s, Z}{\min}
\frac{\nu L}{2} \sum_{s=1}^L \big(\|w^s\|_{\L(\X\oplus \R,\X)}^2
+\frac{1}{r} \|Z^s\|_{\X^N}^2
\big)
+
\lambda\,\|\tilde{w}\|_{\L(\X\oplus \R,\Y)}^2+\|Z\|_{\Y^N}^2\,, \\
&\text{s.t. } q^1=X\,,\quad q^{s+1}=q^s +w^s \bvarphi(q^s)+Z^s\, \text{ and } \tilde{w}\bvarphi(q^{L+1})+Z=Y
w^s \bvarphi(q^s)
\end{cases}
\end{equation}
These slack variables have, as in Tikhonov regularization, a natural interpretation as Gaussian noise added to the output of each layer at the training stage.
In particular $r$  plays the same role as $\lambda^{-1}$ in the ridge regression \eqref{eqledhehdiudh} with $\ell_\Y=$\eqref{eqjhguyghfuvf4} (it can be interpreted as variances of propagation errors) and
\eqref{eqlkjgedehgddjededdeddddsdddsshjssdBN2intro} converges to \eqref{eqjkehbwjhebdhdjq} as $r \downarrow 0$.
While reducing overfitting by training with noise seems to have been exhaustively explored
(variants include adding noise to the input data \cite{holmstrom1992using}, to the weights and biases \cite{an1996effects} and to the activation functions
\cite{gulcehre2016noisy}), the proposed approach seems to be distinct in the sense that the $Z^s$ and $Z$ are variables to be trained alongside the weights and biases of the network.
Furthermore, since
the proposed strategy is equivalent to
adding the nugget $r$  to the kernel $\Gamma$  in \eqref{eqlktddsyeytdsedhjd} ($\Gamma \rightarrow \Gamma + r I$ where $I$ is the identity operator), it
is the natural generalization of the regularization strategy employed in kriging. In particular, as in kriging,  the noise variables $z_s$ and $z$ are set to zero at the testing stage.
To summarize, all inference methods are characterized by a tradeoff between accuracy and robustness \cite{owhadi2017qualitative, bajgiran2021uncertainty}. Adding a propagation error moves this tradeoff towards robustness and ensures the robustness/stability of the network with respect to both training and testing data. When implementing this approach, the constraint $q^{s+1}=q^s +w^s \bvarphi(q^s)+Z^s$ should be used to eliminate the noise variables, which reduces
 \eqref{eqlkjgedehgddjededdeddddsdddsshjssdBN2intro2}  to \eqref{eqlkjgedehgddjededdeddddsdddsshjssdBN2intro} (\eqref{eqlkjgedehgddjededdeddddsdddsshjssdBN2intro} is the loss to be implemented).

\subsection{Simple and rigorous regularization strategy for ANNs}
The regularization strategy presented in Sec.~\ref{subsecreg1} is generic. It naturally extends to arbitrary ANNs, and computational graph completion problems \cite{owhadi2021computational} and can be employed to make them robust as nuggets are employed to make kriging robust.
In particular, for a vanilla ANN of the form $f_1\circ \cdots \circ f_L$, where the $f_i$ are layers of the network, while the traditional approach is to minimize $\|f_1\circ \cdots \circ f_L(X)-Y\|_{\Y^N}$ with possibly a small weights (e.g. $L^2$) regularization $+\sum_{s=1}^N \|w_s\|^2_{L^2}$, the proposed approach (leading to rigorous stabilization of the underlying ANN) is to minimize
\begin{equation}\label{eqhjwgdhgedy}
\|q^L-Y\|_{\Y^N}+\lambda \sum_{s=1}^L \|w_s\|^2_{L^2}+\sum_{s=1}^L \rho^{-1}\|q^{s+1}-f_s(q^s)\|_{L^2}^2
\end{equation}
with respect to weights and the trajectory $q^s$ (initiated at $q^1=X$).
While \eqref{eqhjwgdhgedy} has  been proposed \cite{carreira2014distributed, choromanska2019beyond} as a distribution optimization approach to
relax backpropagation in deeply nested networks; our following analysis suggests that it should also be employed to ensure the robustness of the underlying ANN.

\begin{Remark}
Note that the loss $\mathcal{A}(\rho)=\eqref{eqhjwgdhgedy}$ is a decreasing function of $\rho$ (the strength of the regularization). Since this loss is equal to
$\|q^L-Y\|_{\Y^N}+\lambda \sum_{s=1}^L \|f_s\|^2_{K_s}+\sum_{s=1}^L \rho^{-1}\|q^{s+1}-f_s(q^s)\|_{L^2}^2$ (where the kernels $K_s$ are as in
 Sec.~\ref{subresnetpar}) it follows that the value of $\|q^L-Y\|_{\Y^N}+\lambda \sum_{s=1}^L \|f_s\|^2_{K_s}$ when training with $\rho>0$ is smaller than that of $\|q^L-Y\|_{\Y^N}+\lambda \sum_{s=1}^L \|f_s\|^2_{K_s}$ when training with $\rho=0$. This implies that the underlying RKHS norms of the functions associated with the layers of the network are smaller when training with regularization $\rho>0$. In that sense $\rho>0$ avoids overfitting at all layers of the network by balancing the RKHS norms of the $f_s$ with their propagation errors.
It is well understood in numerical approximation/Krigging that interpolation (generalization) errors are small if the target function has a small RKHS norm with respect to the underlying kernel \cite{OwhScobook2018}. Although we have not quantified the impact of regularization on generalization, the error bounds presented in Sec.~\ref{subsecprobahoeuwgediu} suggest that a similar mechanism could lead to improved generalization for ANNs.
\end{Remark}

\subsection{Reduction to a discrete least action principle}
We will now analyze minimizers of \eqref{eqlktddsyeytdsedhjdreg}. Given the added regularization, Cond.~\ref{condnugget} and \ref{condeqlkedjehd7d} can, as described below, be relaxed to the following conditions, which we assume to hold true in this section.
\begin{Condition}\label{condnuggetreg}
Assume that (1) $\X$ and $\Y$ are finite-dimensional
 (2) $x\rightarrow K(x,x')$ is continuous for all $x'$, and (3)
 $(x,x')\rightarrow \wK(x,x')$ and its first and second order partial derivatives are continuous and uniformly bounded.
\end{Condition}

Minimizing in $f$ and $Y'$ first, \eqref{eqlktddsyeytdsedhjdreg} is equivalent to
 \begin{equation}\label{eqlktddsyeytdsefsreg}
 \begin{cases}
\text{Minimize } &\frac{\nu}{2}\,L\sum_{s=1}^L \big(\|v_s\|_{\Gamma}^2+ \frac{1}{r} \|q^{s+1}-(I+v_s)(q^s)\|_{\X^N}^2\big)+\ell\big(q^{L+1},Y\big)\\
\text{over }&v_1,\ldots,v_L \in \H_\Gamma,\, \, q^1,\ldots,q^{L+1} \in \X^N,\, q^1=X\,.
\end{cases}
\end{equation}
where $\ell\,:\, \X^N \times \Y^N \rightarrow [0,\infty]$  is the loss defined by
 \begin{equation}\label{eqlktdelldreg}
 \ell(X',Y):=
\begin{cases}
\text{Minimize } &\lambda\,\big(\|f\|_{K}^2+\frac{1}{\rho} \|f(X')-Y'\|_{\Y^N}^2 \big)+\ell_\Y\big(Y',Y\big)\\
\text{over }&f\in \H_K,\, Y'\in \Y^N\,.
\end{cases}
\end{equation}

For $q\in \X^N$, write $\wK_r(q,q):=\wK(q,q)+r I$ and $K_\rho(q,q):=K(q,q)+\rho I$
for the $N\times N$ block operator
matrices with blocks $\wK(q_i,q_j)+ r \delta_{i,j}  I_\X $ and
$K(X_i,X_j)+\rho  \delta_{i,j} I_\Y$ (writing $I_\X$ ($I_\Y$) for the identity operator on $\X$ ($\Y$)).

\begin{Theorem}\label{thmeqkljedjnedjdreg}
 $(v_1,\ldots,v_L ,f)$,  is  a minimizer of \eqref{eqlktddsyeytdsedhjdreg} if and only if
\begin{equation}\label{eqkjdkedkjndreg}
v_s(x)= \wK( x, q^s) \wK_r(q^s,q^s)^{-1} (q^{s+1}-q^s) \text{ for } x\in \X, s\in \{1,\ldots,L\}\,,
\end{equation}
where $q^1,\ldots,q^{L+1} \in \X^N$ is a minimizer of (write $\Delta t:=1/L$)
\begin{equation}\label{eqlsedhjdreg}
\begin{cases}
\text{Minimize } &\frac{\nu}{2}  \sum_{s=1}^L  (\frac{q^{s+1}-q^s}{\Delta t})^T \wK_r(q^s,q^s)^{-1} (\frac{q^{s+1}-q^s}{\Delta t})\, \Delta t+
\ell\big(q^{L+1},Y\big)\\
\text{over } &q^2,\ldots,q^{L+1} \in \X^N
\text{ with }q^1=X\,,
\end{cases}
\end{equation}
and  $ f$ is a minimizer of \eqref{eqlktdelldreg} defining $\ell(q^{L+1},Y)$. Furthermore $f$ is a minimizer of \eqref{eqlktdelldreg} defining $\ell(X',Y)$ if and only if
\begin{equation}\label{eqkjhebdjehbdhd}
f(\cdot)= K(\cdot,X') V
\end{equation}
where $V$ is a minimizer of
\begin{equation}\label{eqledhewededswreg}
\ell(X',Y)=\inf_{V\in \Y^N}\lambda\,V^T K_\rho(X',X') V+\ell_\Y(K(X,X) V,Y)=\eqref{eqlktdelldreg}\,.
\end{equation}
Finally, under Cond.~\ref{condnuggetreg}, $\ell=$\eqref{eqlktdelldreg}=\eqref{eqledhewededswreg} is continuous (in both arguments), positive and admits a minimizer $f\in \H_K$ that is unique if $Y'\rightarrow \ell_\Y(Y',Y)$ is  convex.
\end{Theorem}
\begin{proof}
\eqref{eqajkjwdhjbdjeh} implies \eqref{eqkjdkedkjndreg}. The representer theorem implies \eqref{eqkjhebdjehbdhd}.
 Using \eqref{eqlkjdkjweedjkb} we get \eqref{eqlsedhjdreg} and \eqref{eqledhewededswreg} from $\min_{f\in \H_K} \|f\|_{K}^2+\frac{1}{\rho} \|f(X')-Y'\|_{\Y^N}^2 = (Y')^T K_\rho(X', X')^{-1}  Y'$ and
$
\min_{v_s\in {\H_\Gamma}} \big(\|v_s\|_{\Gamma}^2+ \frac{1}{r} \|q^{s+1}-(I+v_s)(q^s)\|_{\X^N}^2\big)=(q^{s+1}-q^s)^T
\wK_r(q^s,q^s)^{-1} (q^{s+1}-q^s)
$.\\ $Z^T K_\rho(X',X') Z\geq \rho Z^T Z$ ensures that the variable $Z$ in  \eqref{eqledhewededswreg} can be restricted to live in a compact set.
The continuity of $\ell$ then follows from \cite[Lem.~5.3,5.4]{still2018lectures} and
uniqueness follows from the (strict) convexity of \eqref{eqledhewededswreg} in $Z$.
\end{proof}

\subsection{Continuous least action principle and Hamiltonian system}
The continuous limit of the discrete least action principle \eqref{eqlsedhjdreg} is
\begin{equation}\label{eqlsweedhsejdreg}
\begin{cases}
\text{Minimize } &\nu\,\A_r[q]+\ell\big(q(1),Y\big)\\
\text{over } &q \in C^1([0,1],\X^N) \text{ subject to } q(0)=X\,,
\end{cases}
\end{equation}
where $\A_r$ is the continuous action defined by
\begin{equation}\label{eqactqreg}
\A_r[q]:=\int_0^1 \frac{1}{2}\dot{q}^T \wK_r(q, q)^{-1} \dot{q}  \,dt\,,
\end{equation}
whose regularized Lagrangian $\Lf_r(q,\dot{q})= \frac{1}{2}\dot{q}^T \wK_r(q, q)^{-1} \dot{q}$
 is identical to \eqref{eqejkhekffkjefkfj}  with
$\wK(q,q)$ replaced by $\wK_r(q,q)$.
We will now show that all the results of Sec.~\ref{seclkjdekjdhjd} remain true  under regularization and the relaxed conditions \ref{condnuggetreg} with $\wK(q,q)$ replaced by $\wK_r(q,q)$.

\begin{Theorem}\label{thmlllakbhhdreg}
For   $q\in  C^1([0,1],\X^N)$ introduce the momentum
\begin{equation}\label{eqlhedjkekdkjreg}
p=\wK_r(q,q)^{-1} \dot{q}\,.
\end{equation}
$q$ is minimizer of \eqref{eqlsweedhsejdreg} if and only if $q(0)=X$, $(q,p)$  follows the Hamiltonian dynamic
\begin{equation}\label{eqkedmdledkemdlreg}
\begin{cases}
&\dot{q}=\wK_r(q,q) p\\
&\dot{p}=- \partial_q ( \frac{1}{2} p^T  \wK_r(q, q) p)\,,
\end{cases}
\end{equation}
defined by the regularized Hamiltonian $\Hf_r(q,p)= \frac{1}{2}p^T \wK_r(q, q) p$, and $p(0)$ is  a minimizer of
 \begin{equation}\label{eqkjlkjkejwkdjreg}
\Vk^r\big(p(0),X,Y\big):= \frac{\nu}{2}p^T(0) \wK_r\big(X,X\big) p(0)+ \ell(q(1),Y)\,.
 \end{equation}
Furthermore,  \eqref{eqkedmdledkemdlreg} admits a unique solution in $C^2([0,1],\X^N)\times C^1([0,1],\X^N)$,
the energy $\Hf_r(q,p)= \frac{1}{2}p^T \wK_r(q, q) p$ is an invariant of the dynamic,
 and
$p(1)$ satisfies \eqref{eqkjdjhebdjhd}.
\end{Theorem}
\begin{proof}
The proof is identical to those presented in sec.~\ref{seclkjdekjdhjd} and \ref{seclkjdekjdhjdex}. Since
$ r p^T p \leq p^T \wK_r(q, q) p$, \eqref{eqkjlkjkejwkdjreg} and energy preservation imply that  $p(t)$ is confined to a compact set.
\end{proof}

\subsection{Regularized idea registration.}\label{secjhghgyuygyreg}
In the limit $L\rightarrow \infty$, $(I+v_k)\circ \cdots \circ (I+v_1)$ (obtained from \eqref{eqlktddsyeytdsefsreg}) approximates (at time $t_k:=\frac{k}{L}$) the flow map
   $\phi^v$ (defined as the solution of \eqref{eqflmp})
 where $v$ is a
    minimizer of
\begin{equation}\label{eqlkjgehgddjedhjdreg}
\begin{cases}
\text{Minimize } &\frac{\nu}{2}\,\int_0^1 \big(\|v\|_{\Gamma}^2+ \frac{1}{r}\|\dot{q}-v(q,t)\|_{\X^N}^2\big)\,dt+\ell\big(q(1),Y\big)\\
\text{over }&v \in C([0,1],\H_\Gamma),\,q\in C^1([0,1],\X^N),\,q(0)=X\, .
\end{cases}
\end{equation}

The proof of the following theorem is identical to that of Thm.~\ref{thmkljedjnedesdjd}.
\begin{Theorem}\label{thmkljedjnedesdjdreg}
  $v$ is a minimizer of \eqref{eqlkjgehgddjedhjdreg} if and only if
\begin{equation}\label{eqtttcreg}
v(x,t)= \wK\big(x,q(t)\big) \wK_r\big(q(t),q(t)\big)^{-1}\dot{q}(t)
\end{equation}
where $q$ is a minimizer of the least action principle \eqref{eqlsweedhsejdreg}. Furthermore the value of \eqref{eqlkjgehgddjedhjdreg} after minimization over $v$ is \eqref{eqlsweedhsejdreg} and
 the representer ODE \eqref{eqtttc} can be written as in \eqref{eqttjjkjtc}
where $(q,p)$ is the solution of the Hamiltonian system \eqref{eqkedmdledkemdlreg} with initial condition $q(0)=X$ and $p(0)$ identified as a minimizer of \eqref{eqkjlkjkejwkdjreg}.
\end{Theorem}

\subsection{Existence/identification of minimizers and  energy preservation}\label{subseckjjhgdejwedh2}
As in Subsec.~\ref{subseckjjhgdejwedh}, minimizers of \eqref{eqlkjgehgddjedhjdreg} may not be unique.
The proof of the following theorem is identical to that of Thm.~\ref{thmksbsahshkd}.

\begin{Theorem}\label{thmksbsahshkdreg}
The minimum values of \eqref{eqlsweedhsejdreg}, \eqref{eqkjlkjkejwkdjreg} and \eqref{eqlkjgehgddjedhjdreg} are identical.
\eqref{eqlsweedhsejdreg}, \eqref{eqkjlkjkejwkdjreg} and \eqref{eqlkjgehgddjedhjdreg} have minimizers.
$q$ is a minimizer of \eqref{eqlsweedhsejdreg} if and only if $(q,p)$ ($p=\wK_r(q,q)^{-1}\dot{q}$) follows the Hamiltonian dynamic  \eqref{eqkedmdledkemdlreg} (with $q(0)=X$) and $p(0)=\wK_r\big(q(0),q(0)\big)^{-1} \dot{q}(0)$ is a minimizer of $\Vk^r\big(p(0),X,Y\big)=$\eqref{eqkjlkjkejwkdjreg}.
$v$ is a minimizer of \eqref{eqlkjgehgddjedhjdreg} if and only if
\begin{equation}\label{eqkjhwbkhjbjewhbd}
v(x,t)= \wK\big(x,q(t)\big) p(t)
\end{equation}
 with $(q,p)$ following the Hamiltonian dynamic  \eqref{eqkedmdledkemdlreg} (with $q(0)=X$) and $p(0)$ being a minimizer of $\Vk^r\big(p(0),X,Y\big)=$\eqref{eqkjlkjkejwkdjreg}.
Therefore the minimizers of \eqref{eqlsweedhsejdreg} and \eqref{eqlkjgehgddjedhjdreg} can be parameterized by their initial momentum $p(0)$, identified as a minimizer of $\Vk^r\big(p(0),X,Y\big)=$\eqref{eqkjlkjkejwkdjreg}.
 Furthermore, at those minima, the energies $\frac{1}{2}p \wK_r(q, q) p=\frac{1}{2}\dot{q}^T \wK_r(q, q)^{-1} \dot{q}=\frac{1}{2} (\|v\|_{\Gamma}^2+\frac{1}{r}\|\dot{q}-v(q,t)\|_{\X^N}^2) =\frac{1}{2} \|v\|_{{\H_{\Gamma_r}}}^2$ (writing $\|\cdot\|_{\H_{\Gamma_r}}$ for the RKHS norm defined by the kernel $\wK(x,x')+r \updelta(x-x')I$) are constant over $t\in [0,1]$ and equal to $\frac{\nu}{2}p^T(0) \wK_r\big(X,X\big) p(0)$.
\end{Theorem}

As in Sec.~\ref{seclkjdekjdhjd} the trajectory $q^1,\ldots,q^{L+1}$ of a minimizer of \eqref{eqlsedhjdreg} follows
\begin{equation}\label{ejkhdbejhdbdreg}
\begin{cases}
q^{s+1}&= q^s+\Delta t\, \wK_r(q^s,q^s) p^s\\
p^{s+1}&=p^s-\frac{\Delta t}{ 2} \partial_{q^{s+1}}\big((p^{s+1})^T \wK_r(q^{s+1},q^{s+1}) p^{s+1}\big)\,,
\end{cases}
\end{equation}
where $p^s$ is the momentum
\begin{equation}\label{eqkjelkbejdhbdreg}
p^s=\wK_r(q^s,q^s)^{-1} \frac{q^{s+1}-q^s}{\Delta t}\,.
\end{equation}
Write
\begin{equation}\label{eqkjhwbkdwehjvbdyreg}
\mathfrak{V}_L^r(p^1,X,Y):=\begin{cases}&\frac{\nu}{2}  \sum_{s=1}^L  (p^s)^T \wK_r(q^s,q^s) p^s\, \Delta t+
\ell\big(q^{L+1},Y\big)\\
&p^s=\eqref{eqkjelkbejdhbdreg}\text{ and }(q^s,p^s)\text{ follow }\eqref{ejkhdbejhdbdreg} \text{ with } q^1=X\,.
\end{cases}
\end{equation}

The proof of the following theorem is identical to that of Thm.~\ref{thmksbsahskd2g}.

\begin{Theorem}\label{thmksbsahskd2greg}
The minimum values of  \eqref{eqlktddsyeytdsefsreg}, \eqref{eqlsedhjdreg} and $\mathfrak{V}_L^r(p^1,X,Y)$ (in $p^1$) are identical.
\eqref{eqlktddsyeytdsefsreg}, \eqref{eqlsedhjdreg} and \eqref{eqkjhwbkdwehjvbdyreg} have minimizers.
$q^1,\ldots,q^{L+1}$ is a minimizer of \eqref{eqlsedhjdreg} if and only if $(q^s,p^s)$  (with $p^s$=\eqref{eqkjelkbejdhbdreg}) follows the discrete Hamiltonian map \eqref{ejkhdbejhdbdreg}, $q^1=X$ and $p^1$ is a minimizer of $\mathfrak{V}_L^r(p^1,X,Y)=$\eqref{eqkjhwbkdwehjvbdyreg}.
$v_1,\ldots,v_L$ is a minimizer of \eqref{eqlktddsyeytdsefsreg} if and only if
\begin{equation}\label{eqkjhbdejdhbdh}
v_s(x)=\wK( x, q^s)p^s=\eqref{eqkjdkedkjndreg}\,,
\end{equation}
 where $(q^s,p^s)$
 follows the discrete Hamiltonian map \eqref{ejkhdbejhdbdreg} with $q^1=X$ and $p^1$ is a minimizer of $\mathfrak{V}_L^r(p^1,X,Y)=$\eqref{eqkjhwbkdwehjvbdyreg}.
 Therefore the minimizers of \eqref{eqlktddsyeytdsefsreg} and \eqref{eqlsedhjdreg} can be parameterized by their initial momentum identified as a minimizer of $\mathfrak{V}_L^r(p^1,X,Y)=$\eqref{eqkjhwbkdwehjvbdyreg}.
 At those minima, the energies  $\frac{1}{2}(p^s)^T \wK_r(q^s,q^s) p^s$ and $\frac{1}{2}\|v_s\|_{\H_{\Gamma_r}}^2=\frac{1}{2}\big(\|v_s\|_{\Gamma}^2+\frac{1}{r} \|q^{s+1}-(I+v_s)(q^s)\|_{\X^N}^2\big)$  are equal and fluctuate by at most $\mathcal{O}(1/L)$ over $s\in \{1,\ldots,L\}$.
\end{Theorem}

\subsection{Continuity of minimal values}\label{subsecjkhgg6}
The following theorem does not have an equivalent in sec.~\ref{seclkjdekjdhjd} and \ref{seclkjdekjdhjdex} since it does not hold true without regularization.

\begin{Theorem}\label{thmwksjhgw76sgw}
The minimal values of \eqref{eqlsweedhsejdreg}, \eqref{eqkjlkjkejwkdjreg}, \eqref{eqlkjgehgddjedhjdreg}, \eqref{eqlktddsyeytdsefsreg}, \eqref{eqlsedhjdreg} and \eqref{eqkjhwbkdwehjvbdyreg} are continuous in $(X,Y)$.
\end{Theorem}
\begin{proof}
 By Thm.~\ref{thmksbsahshkdreg} it is then sufficient (for the discrete setting) to prove the continuity of the minimum value of \eqref{eqkjlkjkejwkdjreg} with respect to $(X,Y)$. By \cite[Thm.~1.4.1]{arino2006fundamental} if  $(q,p)$  follows the Hamiltonian dynamic  \eqref{eqkedmdledkemdlreg}
then, under Cond.~\ref{condnuggetreg}, $q(1)$ is continuous with respect to $p(0)$. The continuity of $\ell$ then implies that
of \eqref{eqkjlkjkejwkdjreg} with respect to $p(0)$ (with $q(1)$ being a function of $p(0)$).
Under Cond.~\ref{condnuggetreg}  $p(0)$ can be restricted to a compact set  in the  minimization of \eqref{eqkjlkjkejwkdjreg}. \cite[Lem.~5.3,5.4]{still2018lectures} then implies the continuity of the minimum value of \eqref{eqkjlkjkejwkdjreg} with respect to $(X,Y)$.
By Thm.~\ref{thmksbsahskd2greg} the  continuity of the minimum value of \eqref{eqlktddsyeytdsedhjdreg} follows from that
of $\mathfrak{V}_L^r(p^1,X,Y)=$\eqref{eqkjhwbkdwehjvbdyreg} which is continuous in all variables.  Under Cond.~\ref{condnuggetreg},  $p^1$ can be restricted to a compact set in the  minimization of \eqref{eqkjhwbkdwehjvbdyreg}. We conclude
by using \cite[Lem.~5.3,5.4]{still2018lectures}.
\end{proof}

\subsection{Convergence of minimal values and minimizers}\label{subskjkecgyf67f6reg}
Write $\Mk_L^r(X,Y)$ for the set of minimizers $p^1$ of $\Vk_L^r(p^1,X,Y)=$\eqref{eqkjhwbkdwehjvbdyreg}.
Write  $\Mk^r(X,Y)$ for the set of minimizers $p(0)$ of   $\Vk^r\big(p(0),X,Y\big)=$\eqref{eqkjlkjkejwkdjreg}.
The proof of the following theorem is identical to that of Thm.~\ref{thmhgw7gfdd}.

\begin{Theorem}\label{thmhgw7gfddreg}
The common minimal value of \eqref{eqlktddsyeytdsefsreg}, \eqref{eqlsedhjdreg} and \eqref{eqkjhwbkdwehjvbdyreg} converge, as $L\rightarrow \infty$,
towards the common minimal value of  \eqref{eqlsweedhsejdreg}, \eqref{eqkjlkjkejwkdjreg}, \eqref{eqlkjgehgddjedhjdreg}.
As $L\rightarrow \infty$, the set of adherence values of $\Mk_L^r(X,Y)$ is  $\Mk^r(X,Y)$.
Let $v_s^L$, $q^s_L$ and $p^s_L$ be sequences of minimizers of \eqref{eqlktddsyeytdsefsreg}, \eqref{eqlsedhjdreg} and \eqref{eqkjhwbkdwehjvbdyreg} indexed by the same sequence $p^1_L$ of initial momentum in $\Mk_L^r(X,Y)$ (as described in Thm.~\ref{thmksbsahskd2greg}). Then, the adherence points of the sequence $p^1_L$ are in $\Mk^r(X,Y)$ and if $p(0)$ is such a point
($p^1_L$ converges towards $p(0)$ along a subsequence $L_k$) then, along that subsequence: (1) The trajectory formed by interpolating the states $q^s_L \in \X^N$ converges to the trajectory formed by a minimizer of \eqref{eqlsweedhsejdreg} with initial momentum $p(0)$.
(2) For $t\in [0,1]$,  $(I+v_{\textrm{int}(tL)}^L)\circ \cdots \circ (I+v_1^L)(x)$ converges to $\phi^v(x,t)$=\eqref{eqflmp} where $v$ is a minimizer of \eqref{eqlkjgehgddjedhjdreg} with initial momentum $p(0)$.
Conversely if $p(0)\in \Mk^r(X,Y)$ then it is the limit of a sequence $p^{1}_L \in \Mk_L^r(X,Y)$ and the minimizers of \eqref{eqlktddsyeytdsefsreg}, \eqref{eqlsedhjdreg} and \eqref{eqkjhwbkdwehjvbdyreg} with initial momentum $p^1_L$ converge (in the sense given above) to the minimizers of \eqref{eqlsweedhsejdreg}, \eqref{eqkjlkjkejwkdjreg}, \eqref{eqlkjgehgddjedhjdreg} with initial momentum $p(0)$ (as described in Thm.~\ref{thmksbsahshkdreg}).
\end{Theorem}
It follows from Thm.~\ref{thmhgw7gfddreg} that, in the limit $L\rightarrow \infty$, a regularized warping regression solution to Problem \ref{pb828827hee}  approximates $f^\dagger$ with
$f^\ddagger=f\circ \phi^v(\cdot,1)$ where $f$ and $v$ are minimizers of
\begin{equation}\label{eqlkjgjhdbjehreg}
\begin{cases}
\text{Minimize } &\frac{\nu}{2}\,\int_0^1 \big(\|v\|_{\Gamma}^2+ \frac{1}{r}\|\dot{q}-v(q,t)\|_{\X^N}^2\big)\,dt\\&+
\lambda\,\big(\|f\|_{K}^2+\frac{1}{\rho} \|f(q(1))-Y'\|_{\Y^N}^2 \big)+\ell_\Y\big(Y',Y\big)\\
\text{over }&v \in C([0,1],\H_\Gamma),\,q\in C^1([0,1],\X^N),\,q(0)=X,\,f\in \H_K,\, Y'\in \Y^N\, .
\end{cases}
\end{equation}

\section{Deep residual Gaussian processes and generalization error estimates}\label{seckjhhehwd78w696d}

This section presents a natural \cite[Sec.~7\&17]{owhadi2019operator} extension of scalar-valued Gaussian processes to function-valued Gaussian processes (Subsec.~\ref{subsecuideydiueyd}).
This extension leads to deterministic (Cor.~\ref{coralkjhdedsbahkjedb}) and probabilistic (Subsec.~\ref{subsecprobahoeuwgediu}) error estimates\footnote{These error estimates are analogous to classical GPR error estimates and distinct from the usual ones found in
 the statistical learning literature based on a data distribution.
} for warping regression and idea registration. Minimizers of \eqref{eqlkjgehgddjedhjdB} (and its regularized variant \eqref{eqlkjgedehgddjededdeddddsdddsshjssdBN2intro}) have (as suggested for image registration in \cite[p.~4]{dupuis1998variational}) natural interpretations as MAP estimators of Brownian flows of diffeomorphisms \cite{baxendale1984brownian, kunita1997stochastic}, which we  extend  (Subsec.~\ref{subseckjhewg76}) as deep residual Gaussian processes (that can be interpreted as a continuous variant of deep Gaussian processes \cite{damianou2013deep}).

\subsection{Function-valued Gaussian processes}\label{subsecuideydiueyd}
The following definition of function-valued Gaussian processes is a natural extension of scalar-valued Gaussian fields as presented in \cite[Sec.~7\&17]{owhadi2019operator}.
\begin{Definition}\label{dejheekfklfhrf}
Let $K\,:\, \X\times \X\rightarrow \L(\Y)$ be an operator-valued kernel as in Sec. \ref{secovk}. Let $m$ be a function mapping $\X$ to $\Y$.
We call
$\xi\,:\,  \X\rightarrow \L(\Y,{\bf H})$ a function-valued Gaussian process if
 $\xi$ is a function mapping
$x \in \X$ to $\xi(x)\in \L(\Y,{\bf H})$ where ${\bf H}$ is a Gaussian space\footnote{That is a Hilbert space of centered Gaussian random variables, see \cite[Sec.~7\&17]{owhadi2019operator}.} and $\L(\Y,{\bf H})$ is the space of bounded linear operators from $\Y$ to ${\bf H}$. Abusing notations we write $\<\xi(x),y\>_\Y$ for $\xi(x)y$. We say that $\xi$ has mean $m$ and covariance kernel $K$ and write $\xi \sim \cN(m,K)$ if $\<\xi(x),y\>_\Y \sim \cN\big(m(x),y^T K(x,x)y\big)$ and
\begin{equation}
\Cov\big(\<\xi(x),y\>_\Y, \<\xi(x'),y'\>_\Y\big)=y^T K(x,x')y'\,.
\end{equation}
We say that $\xi$ is centered if it is of zero mean.
\end{Definition}
If $K(x,x)$ is trace class ($\Tr[K(x,x)]<\infty$) then $\xi(x)$  defines a measure  on $\Y$ (i.e. a $\Y$-valued random variable), otherwise it only defines a (weak) cylinder-measure in the sense of Gaussian fields (see \cite[Sec.~17]{owhadi2019operator}).
\begin{Theorem}
The distribution of a function-valued Gaussian process is uniquely determined by its mean and covariance kernel $K$. Conversely given  $m$ and $K$ there exists a function-valued Gaussian process having mean $m$ and covariance kernel $K$. In particular if $K$ has feature space $\F$ and map $\psi$, the $e_i$ form an orthonormal basis of $\F$,  and the $Z_i$ are i.i.d. $\cN(0,1)$ random variables, then
\begin{equation}
\xi=m +\sum_i Z_i \psi^T e_i
\end{equation}
is a function-valued GP with mean $m$ and covariance kernel $K$.
\end{Theorem}
\begin{proof}
The proof is classical, see \cite[Sec.~7\&17]{owhadi2019operator}. Note that the separability of
$\F$ ensures the existence of the $e_i$. Furthermore $\E\big[(\xi-m)(\xi-m)^T\big]=\psi^T \psi=K$.
\end{proof}

\subsection{Probabilistic error estimates for function-valued GP regression}\label{subsecjhewgdkedg}
The conditional covariance of the Gaussian process $\xi \sim \cN(m,K)$ (conditioned on the data $(X,Y)$) provides natural  a priori probabilistic error estimates for the testing data. The following theorem identifies this conditional covariance kernel.

\begin{Theorem}
Let $\xi$ be a centered function-valued GP with covariance kernel $K\,:\,\X\times \X\rightarrow \L(\Y)$. Let $X,Y\in \X^N\times \Y^N$. Let
$Z=(Z_1,\ldots,Z_N)$ be a random Gaussian vector, independent from $\xi$, with i.i.d. $\cN(0,\lambda I_\Y)$ entries ($\lambda\geq 0$ and $I_\Y$ is the identity map on $\Y$). Then $\xi$ conditioned on $\xi(X)+Z$ is a function-valued GP with mean
\begin{equation}\label{eqlkwjldkjhjedkhe}
\E\big[\xi(x)\big| \xi(X)+Z=Y\big]=K(x,X)\big(K(X,X)+\lambda I_\Y\big)^{-1} Y=\eqref{eqajkjwdhjbdjeh}
\end{equation}
and conditional covariance operator
\begin{equation}\label{eqkjwbldkdbkld}
K^\perp(x,x'):=K(x,x')-K(x,X)\big(K(X,X)+\lambda I_\Y\big)^{-1}K(X,x')\,.
\end{equation}
In particular, if $K$ is trace class, then
\begin{equation}\label{eqewjhlkefjlwkf}
\sigma^2(x):=\E\Big[\big\|\xi(x) -\E[\xi(x)| \xi(X)+Z=Y]\big\|_\Y^2\Big| \xi(X)+Z=Y\Big]=\Tr\big[K^\perp(x,x)\big]\,.
\end{equation}
\end{Theorem}
\begin{proof}
The proof is a generalization of the classical setting \cite[Sec.~7\&17]{owhadi2019operator}. Writing $\xi^T(x)y$ for $\<\xi(x),y\>_\Y$ observe that $y^T \xi(x)\xi^T(x') y=y^T K(x,x')y' $ implies $\E[\xi(x)\xi^T(x')]=K(x,x')$. Since $\xi$ and $Z$ share the same Gaussian space the  expectation of $\xi(x)$ conditioned on $\xi(X)+Z$ is $A \big(\xi(X)+Z\big)$ where $A$ is a linear map identified by
$0=\Cov\Big(\xi(x)-A \big(\xi(X)+Z\big), \xi(X)+Z\Big)=\E\big[ \xi(x)-A \big(\xi(X)+Z\big) \big(\xi^T(X)+Z^T\big)\big]=
K(x,X)-A\big(K(X,X)+\lambda I_\Y\big)$, which leads to $A=K(x,X)\big(K(X,X)+\lambda I_\Y\big)^{-1}$ and \eqref{eqlkwjldkjhjedkhe}. The conditional covariance is then given by
$K^\perp(x,x')=\E\Big[ \Big(\xi(x)-K(x,X)\big(K(X,X)+\lambda I_\Y\big)^{-1} \big(\xi(X)+Z\big)\Big) \Big(\xi(x')-K(x',X)\big(K(X,X)+\lambda I_\Y\big)^{-1} \big(\xi(X)+Z\big)\Big)^T \Big]$ which leads to \eqref{eqkjwbldkdbkld}.
\end{proof}

\subsection{Deterministic error estimates for function-valued Kriging}\label{subsecprobahoeuwgediu}
For  $\lambda=0$, $f(x)=$\eqref{eqlkwjldkjhjedkhe} is the optimal recovery solution \eqref{eqhgvygvgyv2} to Problem \ref{pb828827hee}.
For $\lambda>0$, $f(x)=$\eqref{eqlkwjldkjhjedkhe} is the ridge regression solution \eqref{eqajkjwdhjbdjeh} to Problem \ref{pb828827hee}. The following theorem shows that the standard deviation \eqref{eqewjhlkefjlwkf} provides deterministic a prior error bounds on the accuracy of the ridge regressor \eqref{eqlkwjldkjhjedkhe} to $f^\dagger$ in Problem \ref{pb828827hee}. Local error estimates such as \eqref{eqllbhjdebjdbd} are classical in Kriging \cite{wu1993local} where $\sigma^2(x)$ is known as the power function/kriging variance  (see also
\cite{owhadi2015bayesian}[Thm.~5.1] for applications to PDEs).

\begin{Theorem}\label{thmalkjhbahkjedb}
Let $f^\dagger$ be the unknown function of Problem \ref{pb828827hee} and let $f(x)=\eqref{eqlkwjldkjhjedkhe}=\eqref{eqajkjwdhjbdjeh}$ be its GPR/ridge regression solution. Let $K_\lambda:=K+\lambda I_\Y$. It holds true that
\begin{equation}\label{eqllbhjdebjdbd}
\big\|f^\dagger(x)-f(x) \big\|_\Y \leq \sigma(x) \|f^\dagger\|_K
\end{equation}
and
\begin{equation}\label{eqllbhjdebddjdbd}
\big\|f^\dagger(x)-f(x) \big\|_\Y \leq \sqrt{\sigma^2(x)+\lambda \operatorname{dim}(\Y)} \|f^\dagger\|_{K_\lambda}\,,
\end{equation}
where $\sigma(x)$ is the standard deviation \eqref{eqewjhlkefjlwkf}.
\end{Theorem}
\begin{proof}
 Let $y\in \Y$. Using the reproducing property \eqref{eqrepprop} and $Y=f^\dagger(X)$ we have
\[\begin{split}
y^T \big(f^\dagger(x)-f(x)\big)&=y^T f^\dagger(x)-y^T K(x,X)\big(K(X,X)+\lambda I_\Y\big)^{-1} f^\dagger(X)\\
&=\<f^\dagger, K(\cdot,x)y-K(\cdot,X)\big(K(X,X)+\lambda I_\Y\big)^{-1} K(X,x)y \>_K\,.
\end{split}
\]
Using Cauchy-Schwartz inequality, we deduce that
\begin{equation}\label{eqkjwkdddjbehdbd}
\Big|y^T \big(f^\dagger(x)-f(x)\big)\Big|^2\leq \|f^\dagger\|_K^2\,  y^T K^\perp(x,x) y
\end{equation}
where $K^\perp$ is the conditional covariance \eqref{eqkjwbldkdbkld}. Summing over $y$ ranging in basis of $\Y$ implies \eqref{eqllbhjdebjdbd}. The proof of \eqref{eqllbhjdebddjdbd} is similar, simply observe that
\[\begin{split}
y^T \big(f^\dagger(x)-f(x)\big)&=\<f^\dagger, K_\lambda(\cdot,x)y-K_\lambda(\cdot,X)\big(K(X,X)+\lambda I_\Y\big)^{-1} K(X,x)y \>_{K_\lambda}\\
&\leq \|f^\dagger\|_{H_\lambda} \big\|K_\lambda(\cdot,x)y-K_\lambda(\cdot,X)\big(K(X,X)+\lambda I_\Y\big)^{-1} K(X,x)y\big \|_{K_\lambda}\,,
\end{split}
\]
which implies
\begin{equation}\label{eqkjwkdjbehdbd}
\Big|y^T \big(f^\dagger(x)-f(x)\big)\Big|^2\leq \|f^\dagger\|_{K_\lambda}^2\,  \big(\lambda y^T y+ y^T K^\perp(x,x) y\big)\,.
\end{equation}
\end{proof}

\begin{Remark}
Since Thm.~\ref{thmalkjhbahkjedb} does not require $\X$ to be finite-dimensional, its estimates do not suffer from the curse of dimensionality, but from finding a good kernel for which both $\|f^\dagger\|_K$ and $y^T K^\perp(x,x) y$ are small (over $x$ sampled from the testing distribution). Indeed
both \eqref{eqllbhjdebjdbd} and \eqref{eqllbhjdebddjdbd} provide a priori deterministic error bounds on $f^\dagger-f$ depending on the RKHS norms $\|f^\dagger\|_K$ and $\|f^\dagger\|_{K_\lambda}$. Although these norms can be controlled in the PDE setting \cite{owhadi2015bayesian} via compact embeddings of Sobolev spaces, there is no clear strategy for obtaining a-priori bounds on these norms for
general machine learning problems\footnote{Although deep learning estimates derived from  Barron spaces \cite{barron1993universal, ma2019barron} have a priori Monte-Carlo (dimension independent) convergence rates, they also suffer from this problem since they require bounding the Barron norm of the target function.}.
\end{Remark}

\subsection{Deep residual Gaussian processes}\label{subseckjhewg76}
Write $\zeta$ for the centered GP (independent from $\xi$) defined by the quadratic norm $ \int_0^1 \|v(\cdot,t)\|_{\Gamma}^2\,dt$ on $L^2([0,1],{\H_\Gamma})$.
Recall \cite[Sec.~7\&17]{owhadi2019operator} that  $\zeta$ is  an isometry mapping $L^2([0,1],{\H_\Gamma})$ to a Gaussian space
(defined by $\int_0^1 \<\zeta, v\>_{\Gamma} \,dt\sim \cN(0, \int_0^1 \|v(\cdot,t)\|_{\Gamma}^2\,dt)$ for $v\in L^2([0,1],{\H_\Gamma})$).
The following proposition presents a construction/representation of the GP $\zeta$.

\begin{Proposition}\label{prophjdgfjhgwefdd}\label{propkjahbkdjhbed}
 Let $\psi$ and  $\F$ be a feature map and (separable) feature space for $\Gamma$.
Let the $e_i$ form an orthonormal basis of $\F$ and let the $B^i$ be independent one dimensional Brownian motions. Then
\begin{equation}
\zeta(x,t)=\sum_i \frac{dB^i_t}{dt}\psi^T(x) e_i
\end{equation}
is a representation of $\zeta$.
\end{Proposition}
\begin{proof}
Thm.~\ref{thmkjhkejhddjhd} implies that $v\in L^2([0,1],{\H_\Gamma})$ admits the representation\\
 $v(x,t)=\sum_i \alpha_i(t)\,\psi^T(x,t) e_i $ where the $\alpha_i$ are scalar-valued functions in $L^2([0,1],dt)$ such that
 $\sum_i \int_0^1 \alpha_i^2(t)\,dt=\int_0^1 \|v(\cdot,t)\|_\Gamma^2\,dt<\infty$. We conclude by observing that (using Thm.~\ref{thmkjhkejhddjhd} again)
$\int_0^1 \<\zeta, v\>_{\Gamma} \,dt=\sum_i \int_0^1 \alpha_i(t) dB^i_t \sim \cN(0, \sum_i \int_0^1 \alpha_i^2(t)\,dt)$.
\end{proof}

Let $\phi^\zeta$ be the solution of \eqref{eqflmp} with $v=\zeta$. We call this solution a {\bf deep residual Gaussian process}. Note that whereas deep Gaussian processes are defined by composing function-valued Gaussian processes \cite{damianou2013deep}, we define deep residual Gaussian processes as the flow of map of the stochastic dynamic system
\begin{equation}\label{eqjlkjdhedkd}
\dot{z}=\zeta(z,t)\text{ with } z(0)=x
\end{equation}
driven by the function-valued GP vector field $\zeta$. Evidently, the existence and uniqueness of solutions to  \eqref{eqjlkjdhedkd} require
the Cameron-Martin space of $\zeta$ to be sufficiently regular. As shown in the following proposition, this result is a simple consequence of the regularity of the feature map in the finite-dimensional setting.
\begin{Proposition}
 Using the notations of Prop.~\ref{propkjahbkdjhbed}, if $\F$ and $\X$ are finite-dimensional and if $\psi$ is uniformly Lipschitz continuous   then \eqref{eqjlkjdhedkd} (and therefore \eqref{eqflmp} with $v=\zeta$) has a unique strong solution.
\end{Proposition}
\begin{proof}
\eqref{eqjlkjdhedkd} corresponds to the finite-dimensional SDE $dz=\sum_i  \psi^T(z)e_i\, dB^i_t$ which is known \cite{kunita1997stochastic} to have unique strong solutions if $\psi$ is uniformly Lipschitz.
\end{proof}

 Let $\xi \sim \cN(0, K)$ (independent from $\zeta$).
$\xi\circ \phi^\zeta(\cdot,1)$ provides a probabilistic solution to Problem \eqref{pb828827hee} in the sense of the following proposition (whose proof is classical).

\begin{Proposition}\label{propkjhebkebjhewd}
Let $(v,f)$ minimize \eqref{eqlkjgehgddjedhjdB} with $\ell_\Y=$\eqref{eqjhguyghfuvf4}.
 Let  $\phi^v$ be the solution of \eqref{eqflmp} obtained from $v$. Then
\begin{equation}\label{eqjhbhjdeedbjdhb}
f^\ddagger(\cdot)=f\circ \phi^v(\cdot,1)
\end{equation}
 is a MAP estimator of $\xi\circ \phi^{\sqrt{\frac{\lambda}{\nu}}\zeta}(\cdot,1)$ given the information
\begin{equation}\label{eqjhbdekjdbdb}
\xi\circ \phi^{\sqrt{\frac{\lambda}{\nu}}\zeta}(X,1)+\sqrt{\lambda} Z=Y\,,
\end{equation}
where $Z=(Z_1,\ldots,Z_N)$ is a centered random Gaussian vector, independent from $\zeta$ and $\xi$, with i.i.d. $\cN(0, I_\Y)$ entries.
\end{Proposition}

The following proposition generalizes Prop.~\ref{propkjhebkebjhewd} to the regularized setting of Sec.~\ref{secregular}.

\begin{Proposition}
Let $\xi,\zeta,Z$ be as in Prop.~\ref{propkjhebkebjhewd}.
Write $\kappa$ for the centered GP defined\footnote{For $q\in L^2([0,1],\X^N)$, $\int_0^1 \kappa(t) u(t) \sim \cN(0,\int_0^1 \|u\|_{\X^N}^2\,dt)$ \cite[Sec.~7\&17]{owhadi2019operator}.} by the norm $\int_0^1 \|\cdot \|_{\X^N}^2$.
Let $z$ be the stochastic process defined as the solution of $\dot{z}=\sqrt{\frac{\lambda}{\nu}}\big(\zeta(z,t)+ \sqrt{r} \kappa\big)$ with initial value $z(0)=X$.
Let $(v,f)$ be a minimizer of  \eqref{eqlkjgjhdbjehreg}, and let
 $\phi^v$ be the solution of \eqref{eqflmp}. The regularized solution $f^\ddagger=f\circ \phi^v(\cdot,1)$ to Problem \eqref{pb828827hee}  is a MAP estimator of $\xi\circ \phi^{\sqrt{\frac{\lambda}{\nu}}\zeta}(\cdot,1)$ given the information
$
\xi\big(z(1)\big)+\sqrt{\lambda+\rho} Z=Y
$.
\end{Proposition}

\subsection{Generalization errors estimates for warping regression}
As in Subsec.~\ref{subsecjhewgdkedg} the conditional posterior distribution of $\xi\circ \phi^{\sqrt{\frac{\lambda}{\nu}}\zeta}(\cdot,1)$ (conditioned on \eqref{eqjhbdekjdbdb}) provides natural probabilistic error estimates on accuracy of a warping regression solution $f$ to Problem \ref{pb828827hee}.
We will now derive deterministic error estimates.

\begin{Corollary}\label{coralkjhdedsbahkjedb}
In the setting of Prop.~\ref{propdkjehbjdhb} it holds true that
\begin{equation}\label{eqllbhjdebjdbdb}
\big\|f^\dagger(x)-f\circ \phi^v(x,1) \big\|_\Y \leq \sigma(x) \|f^\dagger\|_{K^v}\,,
\end{equation}
and
\begin{equation}\label{eqlfrttdbdb}
\big\|f^\dagger(x)-f\circ \phi^v(x,1) \big\|_\Y \leq \sqrt{\sigma^2(x)+\lambda \operatorname{dim}(\Y)} \|f^\dagger\|_{K^v_\lambda}\,,
\end{equation}
with
\begin{equation}
\sigma^2(x):=\Tr\big[K^v(x,x)-K^v(x,X)\big(K^v(X,X)+\lambda I_\Y\big)^{-1}K^v(X,x)\big]\,.
\end{equation}
\end{Corollary}
\begin{proof}
Cor.~\ref{coralkjhdedsbahkjedb} is a direct consequence of Thm.~\ref{thmalkjhbahkjedb} and Prop.~\ref{propdkjehbjdhb}.
\end{proof}

\begin{Remark}
\eqref{eqllbhjdebjdbdb} and \eqref{eqlfrttdbdb} are a priori error estimate similar to those found in PDE numerical analysis. However, although
compactness and ellipticity can be used in PDE analysis \cite{owhadi2015bayesian} to bound $\|f^\dagger\|_{K^v}$ or  $\|f^\dagger\|_{K^v_\lambda}$,  such upper bounds are not available for general machine learning problems. Furthermore, a (deterministic) a posteriori analysis can only provide lower bounds on $\|f^\dagger\|_{K^v}$ and  $\|f^\dagger\|_{K^v_\lambda}$. Examples of such bounds are $\|f\|_{K^v}\leq \|f^\dagger\|_{K^v}$ and
\begin{equation}\label{eqlikwldjklkd}
\|f^\ddagger\|_{K^v_\lambda}\leq \|f^\dagger\|_{K^v_\lambda}\,,
\end{equation}
(implied by Prop.~\ref{coralkjhdedsbahkjedb} and  $\ell_\Y(f^\dagger(X),Y)=0$) for $f^\ddagger(\cdot)=f \circ \phi^v(\cdot,1)$.
Note that Prop.~\ref{propdkjehbjdhb} and Cor.~\ref{coralkjhdedsbahkjedb} do not make assumptions on $\phi^v$.
If $\phi^v$ is selected as a minimizer of
 \eqref{eqlkjgehgddjedhjd} then $f^\ddagger(\cdot)=f \circ \phi^v(\cdot,1)$ is a warping regression solution \eqref{eqkjelkdjendkjdn}
 to Problem \ref{pb828827hee}. Given the identity \eqref{eqnlkdekjddnjd}, in the limit $\nu \downarrow 0$, the variational formulation \eqref{eqlkjgehgddjedhjd}  seek to
  minimize the norm $\|f^\ddagger\|_{K^v_\lambda}$ (which acts in \eqref{eqlikwldjklkd}  as lower bound for the term $\|f^\dagger\|_{K^v_\lambda}$ appearing in the error bound \eqref{eqlfrttdbdb}). Ignoring the gap between $\|f^\dagger\|_{K^v_\lambda}$ and $\|f^\ddagger\|_{K^v_\lambda}$
   in \eqref{eqlikwldjklkd} warping regression seems to select a kernel $K^v(x,x')=K\big(\phi^v(x,1),\phi^v(x',1)\big)$ making  the bound \eqref{eqlfrttdbdb} as sharp as possible. The penalty $\nu\,\int_0^1 \frac{1}{2}\|v\|_{\Gamma}^2\,dt$ in \eqref{eqlkjgehgddjedhjd} can then be interpreted as a regularization term whose objective is to avoid a large gap in \eqref{eqlikwldjklkd} that could be created if $\phi^v(x,1)$ overfits the data.
\end{Remark}

\subsection{MAP  vs cross-validation estimation}\label{secmapvscv}
As discussed in Sec.~\ref{subseckjhewg76}, approximating $f^\dagger$ with \eqref{eqkjedjdiseu} where $f$ and the $v^s$ are identified by minimizing \eqref{eqlktddsyeytdsedhjd} is essentially a MAP estimation approach to the approximation of $f^\dagger$.
The need for solving a global optimization problem (across all the weights of all layers of the network at once and possibly involving backpropagation) can be circumvented by employing a cross-validation approach for identifying the $v^s$.
This strategy, described in \cite{owhadi2019kernel} and analyzed in \cite{chen2020consistency}, is equivalent to approximating $f^\dagger$ with $\E_{\xi \sim \cN(0,K^v)}\big[\xi(x)\mid\xi(X)=Y\big]$ where $K^v$ is the warped kernel $K(\phi_L(\cdot),\phi_L(\cdot))$ with $\phi_L=(I+v_L)\circ \cdots \circ (I+v_1)$ and  $v_s=$\eqref{eqkjdkedkjnd}.  The difference with the MAP approach is that the trajectory $q^s$ defining the $v^s$ in \eqref{eqkjdkedkjnd} is no longer obtained as the solution of the least action principle \eqref{eqkljedjnedjd} but as the flow of a cross-validation loss, i.e. $q^{s+1}=q^s-\delta \nabla_q \rho(q)\large|_{q=q^s}$ where $\rho(q)$ is the squared relative error (in the RKHS norm defined by $K$) between the $K$-interpolant of the data $(q,Y)$ and the $K$-interpolant of the data $(\pi q,\pi Y)$ where $\pi$ is a subsampling operator.

\subsection{MAP vs empirical Bayes estimation}\label{secmapvscveb}
Instead of approximating the underlying GPs appearing in warping regression and idea registration with their MAP estimators, one could employ their Empirical Bayes estimators.
In the warping regression setting of Sec.~\ref{secresnetb}, this is equivalent to approximating
 $f^\dagger$ in Problem \ref{pb828827hee} with \eqref{eqkjedjdiseu} where $\phi_L$ is given by \eqref{eqkjehdbehdhjbdphis} and \eqref{eqkjdkedkjnd}, and the $q^s$ are identified as minimizers of the following empirical Bayes variant of \eqref{eqlsedhjd}
 \begin{equation}\label{eqlsedhjdeb}
\begin{cases}
\text{Minimize } &\ell\big(q^{L+1},Y\big)\\&+\frac{\nu}{2}  \sum_{s=1}^L  \big((\frac{q^{s+1}-q^s}{\Delta t})^T \wK(q^s,q^s)^{-1} (\frac{q^{s+1}-q^s}{\Delta t})
+\log \det \wK(q^s,q^s)\big) \, \Delta t
\\
\text{over } &q^2,\ldots,q^{L+1} \in \X^N
\text{ with }q^1=X\,.
\end{cases}
\end{equation}
In the ResNet setting of Sec.~\ref{subresnetpar}, this is equivalent to  approximating
 $f^\dagger$  with \eqref{eqkjhebkjhedbd} with \eqref{eqjkehbwjhebdhdjq} replaced by
\begin{equation}\label{eqjkehbwjhebdhdjqeb}
\min_{\tilde{w},w_1,\ldots,w_L}\frac{\nu L}{2}\sum_{s=1}^L \big(\|w_s\|_{\L(\X\oplus \R,\X)}^2+ \log \det \wK(q^s,q^s) \big)+ \lambda \|\tilde{w}\|_{\L(\X\oplus \R,\Y)}^2+\ell_\Y\big(f\circ \phi_L(X),Y\big)\,,
\end{equation}
with $f(x)=\tilde{w} \bvarphi(x)$, $\Gamma(x,x')=\bvarphi^T(x) \bvarphi(x') I_\X$, $q^1=X$ and $q^{s+1}= (I+w_s \bvarphi)(q^s)$ for $2\leq s\leq L$.
Recall that although MAP has lower complexity, Empirical Bayes has better consistency (see \cite{dunlop2020hyperparameter} for a detailed analysis of differences between MAP and empirical Bayes).

\section{Further results}\label{secfres}

This section presents further results (some of which are only unpacked in the appendix).

\subsection{With feature maps and activation functions}\label{secmacfm}
Image registration is based on two main strategies \cite{hart2009optimal, allassonniere2005geodesic}: (1)  discretize $v\,:\, \X\times [0,1]\rightarrow \X$ on a space/time mesh and minimize \eqref{eqlkjgehgddjedhjd}; or (2) simulate the  Hamiltonian system \eqref{eqkedmdledkemdl}.
Although these strategies work well in computational anatomy where the dimension of $\X$ is 2 or 3, and the number of landmark points is small, they are not suitable for industrial-scale machine learning where the dimension of $\X$ and the number of data points is large.
When the feature spaces of the underlying kernels are finite-dimensional, then
a feature space representation of warping regression and idea registration
can overcome these limitations and lead to numerical schemes that include and generalize those currently used in deep learning.
The feature map approach to warping regression and idea registration is presented in Sec.~\ref{secfm}.

\subsection{Numerical experiments}\label{subsecnumexmp}
Sec.~\ref{secnumexp} presents numerical experiments on training warping regression networks with geodesic shooting (Sec.~\ref{seckdejhdjkdw}) and with feature maps (Sec.~\ref{secnumfm}). These experiments illustrate (1) the sparsity of the momentum representation of minimizers (2) testing error improvements due to the learned warping.

\subsection{Hydrodynamic limit of the empirical distribution}\label{subjhsecgyf67f6}
Consider the setting of  Sec.~\ref{subgfcsecgyf67f6}.
 \eqref{eqkedmdbhledkemdlN3} and  \eqref{eqkedmdledkemdlN} are, as in the ensemble analysis of gradient descent \cite{mei2018mean, rotskoff2018neural}, natural candidates for a hydrodynamic/mean-field limit analysis. Indeed,
using the change of variables $p_j=\frac{1}{N} \bar{p}_j$, \eqref{eqkedmdbhledkemdlN3} and the Hamiltonian system \eqref{eqkedmdledkemdlN} are equivalent to
\begin{equation}\label{eqkedmdledkemdlNN}
\begin{cases}
&\dot{q}_i =  \psi^T(q_i) \alpha \\
&\dot{\bar{p}}_i=-  \partial_x \big(  \bar{p}_i^T  \psi^T(x) \alpha \big)\Big|_{x=q_i}\,,
\end{cases}
\text{ with }\alpha=\frac{1}{N}\sum_{j=1}^N \psi(q_j) \bar{p}_j\,.
\end{equation}
and \eqref{eqtttc} is equivalent to
\begin{equation}\label{eqkhwbhiuwygde76}
\dot{\phi}^v(x,t)= \psi^T\big(\phi^v(x,t)\big)\,\alpha(t)\,.
\end{equation}
Let $\mu_N:=\frac{1}{N}\sum_{i=1}\updelta_{(q_i,\bar{p}_i)}$ be the empirical distribution of the particles $(q_i,\bar{p}_i)$.
Then by \eqref{eqkedmdledkemdlNN} the average of a test function $f(\tilde{q},\tilde{p})$ against $\mu_N$ obeys the dynamic
\begin{equation}
\frac{d}{dt} \mu_N[f]
=\mu_N\Big[ \partial_{\tilde{q}}f \psi^T(\tilde{q})- \partial_{\tilde{p}}f
\partial_x \big(  \tilde{p}^T  \psi^T(x) \big)\big|_{x=\tilde{q}}
\Big] \mu_N\big[\psi(\tilde{q})\tilde{p} \big]\,,
\end{equation}
which leads to the following theorem which is related to the corresponding optimality equation
known as Euler-Poincar\'{e} equation (EPDiff) \cite{holm2009euler, holm2005momentum}.
\begin{Theorem}\label{eqljdkjdkjdnk}
If, as $N\rightarrow \infty$,  $\mu_N$ and its first-order derivatives weakly converge towards $\mu$ then
 minimizers of \eqref{eqlkjgehgddjedhjd} converge to the solution of\\
$\dot{\phi}(x,t)=\psi^T\big(\phi^v(x,t)\big) \mu\big[\psi(\tilde{q})\tilde{p} \big]$ and
\begin{equation}
\partial_t \mu=  \Big[-\diiv_{\tilde{q}} \big(\mu \psi^T(\tilde{q})\big)
+\diiv_{\tilde{p}}\big(\mu \partial_x \big(  \tilde{p}^T  \psi^T(x) \big)\big|_{x=\tilde{q}} \big)\Big] \mu\big[\psi(\tilde{q})\tilde{p} \big]\,.
\end{equation}
\end{Theorem}
We refer to \cite{mflann22} for numerical experiments supporting the existence of the mean-limit discussed here (and in particular for an analysis of the mean-field limit of the parameter $\alpha$ defined in \eqref{eqkedmdledkemdlNN} as an empirical average).

\subsection{Multiresolution approach}\label{secspamulti}
In the setting of Subsec.~\ref{secspa}, following \cite{schafer2021compression, owhadi2019operator, schafer2021sparse}, we can
 interpret
 the number of particles represented by the components of $q$  as a notion of scale and  initiate a multiresolution description of the action $\A$ supporting
 the proposed interpretation of the momentum variables.
For $q\in C^1([0,1],\X^N)$ (with $N$ being an arbitrary integer) define $\A[q]$ as in \eqref{eqactq}.
For $q^1 \in C^1([0,1],\X^{N_1})$ and $q^2 \in C^1([0,1],\X^{N_2})$ write
$\A\big[(q^1,q^2)\big]$ for the action of the trajectory $t\rightarrow q(t):=(q^1(t),q^2(t))$ in
$\X^{N_1+N_2}$. Note that we have the following consistency relation.
\begin{Proposition}
 Let $q^1 \in C^1([0,1],\X^{N_1})$ and $X^2 \in \X^{N_2}$ be arbitrary. It holds true that
\begin{equation}
\A[q^1]=\inf_{q^2 \in C^1([0,1],\X^{N_2})\,:\, q^2(0)=X^2} \A[(q^1,q^2)]
\end{equation}
\end{Proposition}
\begin{proof}
 Observe that (as in Thm.~\ref{thmkljedjnedesdjd})
   $\inf_{v\in C([0,1],{\H_\Gamma})\,:\, \phi^v(q^1(0),t)=q^1(t)\,\forall t\in [0,1]} \int_0^1 \frac{1}{2}\|v\|_{\Gamma}^2\,dt$
    reduces to $\A[q^1]$ and is also equal to
 the minimum of  $$\inf_{v\in C([0,1],{\H_\Gamma})\,:\, \phi^v((q^1,q^2)(0),t)=(q^1,q^2)(t)\,\forall t\in [0,1]} \int_0^1 \frac{1}{2}\|v\|_{\Gamma}^2\,dt$$
 over $q_2 \in C^1([0,1],\X^{N_2})$ such that  $q_2(0)=X^2$.
\end{proof}

\begin{Proposition}
For an arbitrary trajectory $t\rightarrow q^1(t)$  let
$q^2$ be a minimizer of
\begin{equation}\label{eqkjhgkjguygy}
\nu \A[(q^1,q^2)]+\ell\big((Y^1,Y^2),(q^1,q^2)(1)\big)\,.
\end{equation}
Write $q:=(q^1,q^2)$, $Y=(Y^1,Y^2)$ and $p=(p^1,p^2):=\wK(q,q)^{-1}\dot{q}$. It holds true that $(q,p)$ is a solution of the dynamical system
\begin{equation}\label{eqkjhkjhhjjh}
\begin{cases}
\dot{q}^2&=\wK(q^2,q^1) p^1 + \wK(q^2,q^2) p^2\\
\dot{p}^2&=-\frac{1}{2}\partial_{q^2} p^T \wK(q,q)p\\
p^1 &= \wK(q^1,q^1)^{-1} (\dot{q}^1- \wK(q^1,q^2) p^2)
\end{cases}
\end{equation}
with the boundary condition
\begin{equation}\label{eqljdkejddj}
\nu\,p^2(1) + \partial_{q^2(1)}\ell(Y,q(1))=0\,.
\end{equation}
Furthermore, if  $\partial_{q^2(1)}\ell(Y,q(1))=0$  then
$
p^2(t)=0 \text{ for all }t\in [0,1]\,,
$
and \eqref{eqkjhkjhhjjh} reduces to $\dot{q}^2=\wK(q^2,q^1) \wK(q^1,q^1)^{-1} \dot{q}^1$ (which corresponds to \eqref{eqkjhgkjgguyugy}).
\end{Proposition}
\begin{proof}
The result follows by zeroing the Fr\'{e}chet derivative of \eqref{eqkjhgkjguygy} with respect to $q^2$.
Note that $(p^1,p^2):=\wK(q,q)^{-1}\dot{q}$ implies $\dot{q}^1=\wK(q^1,q^1) p^1 + \wK(q^1,q^2) p^2$ which allows us to
identify $p^1$ as in the third line of \eqref{eqkjhkjhhjjh}.
Note that if $\partial_{q^2(1)}\ell(Y,q(1))=0$ then  \eqref{eqljdkejddj} implies $p^2(1)=0$ and the second equation of \eqref{eqkjhkjhhjjh} implies
$p^2(t)=0$ for all $t\in [0,1]$.
\end{proof}

\subsection{Composing idea registration blocks}\label{subsecukygeuydd}
Composing idea registration blocks (Fig.~\ref{figideaformation3}) produces  input/output functions that have the exact functional structure of ANNs  and   enable their generalization to ANNs of continuous depth and acting on continuous (e.g., functional) spaces  (Sec.~\ref{seckehddjhevd} and \ref{secjhgwyde6dga}). In doing so we (1) prove the existence of minimizers for $L_2$ regularised ANNs/ResNets/CNNs (Thm.~\ref{thmkjhebedjhbdb}, \ref{thmejhekjhbfjhrbf} and \ref{thmejhekjhbfjhrbf2}) (2) characterize these minimizers as autonomous solutions of discrete Hamiltonian systems with discrete least action principles (3) derive the near-preservation of the norm of weights and biases in ResNet blocks (4) obtain their uniqueness given initial momenta (5) prove their convergence (in the sense of adherence values) in the infinite depth limit towards composed/nested idea registration characterized by continuous deformations flows in high dimensional RKHS spaces (6) deduce that training $L^2$-regularized ANNs could in principle be reduced to the determination of the weights and biases of the first layer (Subsec.~\ref{secfr}).

\subsection{Reduced equivariant multichannel (REM) kernels}\label{secremint}
The identification of ANNs as discretized composed idea registration flow maps implies that the search for good architectures for ANNs can be reduced to the search for good kernels for idea registration.
 We introduce (in Sec.~\ref{secrem}) reduced equivariant multichannel (REM) kernels (the equivariant component is a variant of  \cite{reisert2007learning}) and
 show that  CNNs  (and their ResNet variants) are particular instances of composed idea registration with REM kernels. REM kernels (1) enable the generalization of CNNs to arbitrary groups of transformations acting on arbitrary spaces, (2) preserve the relative pose information (see Rmk.~\ref{rmkpose}) across layers.

\section{Related work}\label{secrelwo}

\subsection{Deep kernel learning}\label{secresnet}
The deep learning approach to Problem \ref{pb828827hee} is to approximate $f^\dagger$ with the composition
$f:=f_L\circ \cdots \circ f_1$ of parameterized nonlinear functions $f_k\,:\, \X_k \rightarrow \X_{k+1}$ (with $\X_1:=\X$ and $\X_{L+1}:=\Y$) identified by minimizing the discrepancy between $f(X)$ and $Y$ via Stochastic Gradient Descent.
   \cite{bohn2019representer} proposes to generalize this approach to the nonparametric setting by introducing a representer theorem for the identification of
$(f_1,\ldots, f_L)$ as minimizers
of a loss of the form
\begin{equation}\label{eqkjdjkdnddd}
\sum_{k=1}^L \ell_k\big(\|f_k\|_{K_k}\big)+ \ell_{L+1}\big(f_L\circ \cdots \circ f_1(X),Y\big)
\end{equation}
\cite[Thm.~1]{bohn2019representer} reduces \eqref{eqkjdjkdnddd} to a finite-dimensional optimization problem.

\subsection{Computational anatomy and image registration}
Applying concepts from mechanics to classification/regression problems can be traced back to computational anatomy \cite{grenander1998computational} (and more broadly to image registration \cite{brown1992survey} and shape analysis \cite{younes2010shapes}) where ideas from elasticity and visco-elasticity are used to represent biological variability and create algorithms for the alignment of anatomical structures.
Joshi and Miller \cite{joshi1998large, joshi2000landmark} discovered that  minimizers of \eqref{eqkejdgggkejddj}
 admit a representer formula of the form  \eqref{eqtttc} which can be then
used to produce computationally tractable algorithms for shape analysis/regression   by
(1) minimizing a reduced loss of the form \eqref{eqlsweedhsejd} via gradient descent   \cite{joshi2000landmark, camion2001geodesic}
or  (2) via (geodesic) shooting algorithms obtained from the Hamiltonian perspective \cite{miller2006geodesic, vialard2012diffeomorphic}.
Therefore idea registration could be seen as a natural generalization of image registration in which image spaces are replaced by abstract feature spaces,  material points are replaced by data points, and smoothing kernels (Green's functions of differential operators) are replaced by REM kernels.
Although discretizing the material space is a viable and effective strategy \cite{charon2018metamorphoses} in the Large Deformation
Diffeomorphic Metric Mapping (LDDMM) model \cite{dupuis1998variational, beg2005computing} of image registration, the curse of dimensionality renders it prohibitive for general abstract spaces,
 which is why idea registration must (for efficiency) be implemented with feature maps.
Our regularization strategy for idea registration (proposed in sections \ref{subsecreg1} and \ref{secregular}) is a generalization of that of image registration \cite{micheli2008differential} and linked with the \emph{metamorphosis} of \cite{trouve2005metamorphoses}.
The optimal matching
of  functions and distributions via diffeomorphic
transformations of the ambient space has been the focus of diffeomorphic matching \cite{glaunes2004diffeomorphic, younes2019diffeomorphic}.

\subsection{Interplays between learning, inference, and numerical approximation}\label{subseckhgwgd7egd}
The error estimates discussed in Sec.~\ref{seckjhhehwd78w696d} are instances of interplays between numerical approximation, statistical inference, and learning, which are intimately connected through the common purpose of making estimations/predictions with partial information.
These confluences   (which are not new, see \cite{hennig2015probabilisticgiro, cockayne2019bayesian, owhadi2019statistical, owhadi2019operator} for reviews)
are not just objects of curiosity but constitute a pathway to simple solutions to fundamental problems in all three areas (e.g., solving PDEs as an inference/learning problem \cite{rico1993continuous, owhadi2015bayesian, owhadi2017multigrid, raissi2019physics} facilitates the discovery of efficient solvers with some degree of universality \cite{ schafer2021compression, schafer2021sparse}).
We also observe that the generalization properties of kernel methods (which, as stressed in   \cite{belkin2018understand}, are intimately related to the  generalization properties of ANNs \cite{zhang2016understanding}) can be quantified in a game-theoretic setting \cite{owhadi2019operator} through the observations \cite{owhadi2019operator} that
(1) regression with the kernel $K$ is minimax optimal when relative errors in  RKHS norm $\|\cdot\|_K$ are used as a loss, and (2) $\cN(0,K)$ is an optimal mixed strategy for the underlying adversarial game.

\subsection{ODE interpretations of ResNets}\label{subsecieuygwdd}
Although dynamical systems \cite{weinan2017proposal}, ODE \cite{haber2017stable, chen2018neural},  and diffeomorphism \cite{rousseau2018residual, owhadi2019kernel} interpretations of ResNets are not new, they were based on an Ansatz on the behavior of the trained weights. In this paper, we present the first rigorous proof of this Ansatz for weights trained with $L^2$ regularization.
The ODE interpretation of  ResNets has inspired the application of numerical approximation methods to the design and training of ANNs (including a shooting interpretation of Deep Learning \cite{vialard2020shooting}).
Motivated by the stability of very deep networks \cite{haber2017stable} proposes to derive ANN architectures from the symplectic integration of the Hamiltonian system
\begin{equation}\label{eqklkejdjd}
\begin{cases}
\dot{Y}&=\ba( W Z + b)\\
\dot{Z}&=-\ba( W Y + b)
\end{cases}
\end{equation}
 where, $Y$ and $Z$ are a partition of the features, $\ba$ is an activation function and $W(t)$ and $b(t)$ are time-dependent matrices and vectors acting as control parameters.
 Motivated by the reversibility of the network, \cite{chang2018reversible}  proposes to replace \eqref{eqklkejdjd} by a Hamiltonian system of the form
\begin{equation}\label{eqklkejjrjdjd}
\begin{cases}
\dot{Y}&=W_1^T \ba( W_1 Z + b_1)\\
\dot{Z}&=-W_2^T \ba( W_2 Y + b_2)
\end{cases}
\end{equation}
where $W_1$ and $W_2$ are time dependent convolution matrices acting as control parameters (in addition to $b_1$ and $b_2$).
Motivated by memory efficiency and explicit control of the speed vs. accuracy tradeoff \cite{chen2018neural}
proposes to use Pontryagin's adjoint sensitivity method for computing gradients with respect to the parameters of the network.
While  ResNets have been interpreted as solving an ODE of the form $\dot{x}=\ba(W x+b)$  \cite{weinan2017proposal,haber2017stable},
the feature space formulation of idea registration (Subsec.~\ref{subsecekuwege62}) suggests using ODEs of the form $\dot{x}=W \ba(x)+b$.
\cite{thorpe2018deep} was the first paper to the study the convergence of ResNets with trained weights and biases using $\Gamma$-
convergence.
\cite{LiChenetal17} formulates the forward
propagation as a controlled ODE,  obtains the adjoint equation via Hamiltonian dynamics, and uses
Pontryagin's maximum principle to derive a necessary (but not sufficient) condition for optimality.
 \cite{greydanus2019hamiltonian} draws inspiration from Hamiltonian mechanics to train models that learn and respect exact conservation laws in an unsupervised manner: the purpose is to learn the laws of physics, and
  instead of crafting the Hamiltonian by hand,  \cite{greydanus2019hamiltonian}  proposes parameterizing it with a
neural network and then learning it directly from data.

\subsection{Kernel Flows and deep learning without back-propagation}\label{subseckf}
In the setting of Sec.~\ref{secovk}, given an operator-valued kernel $K\,:\, \X\times \X\rightarrow \L(\Y)$, the Kernel Flows \cite{owhadi2019kernel, yoo2020deep, chen2020consistency, hamzi2021learning} solution to Problem \ref{pb828827hee} is, in its nonparametric version \cite{owhadi2019kernel}, to approximate $f^\dagger$ via ridge regression with a kernel of the form $\mathcal{K}_n(x,x')=K(\phi_n(x),\phi_n(x'))$.
 where $\phi_n\,:\, \X\rightarrow \X$ is a discrete flow learned from data via induction on $n$ with $\phi_0(x)=x$.
 This induction can be described as follows. Let $\wK\,:\, \X\times \X\rightarrow \L(\X)$ be a scalar operator-valued kernel.
 Set $q_n=\phi_n(X)$ and
 \begin{equation}\label{eqkljwbdkhehdhb}
 \phi_{n+1}(x)=\phi_n(x)+  \Delta t\,\wK\big(\phi_n(x), q_n \big) p_n\,,
 \end{equation}
 which is evidently a discretization of \eqref{eqflmp} with $v$ of the form \eqref{eqlhedjkekdkjreg}.
 To identify $p_n$ let $q_{n+1}=q_n+  \Delta t\,\wK\big(q_n, q_n \big) p_n$.
Select $(q_{n+1}',Y')$ as a random subset of the (deformed) data $(q_{n+1},Y)$, write $u_{q_{n+1}}$ for the $\wK$-interpolant of $(q_{n+1},Y)$,
$u_{q_{n+1}}'$ for the $\wK$-interpolant of $(q_{n+1}',Y)$, write $\rho(q_{n+1})=\|u_{q_{n+1}}-u_{q_{n+1}}'\|_{\wK}^2/\|u_{q_{n+1}}\|_{\wK}^2$ and identify $p_n$ in the gradient descent direction of
$\rho(q_{n+1})=\rho(q_n+  \Delta t\,\wK\big(q_n, q_n \big) p_n)$. Since no backpropagation is used to identify $p_n$, the numerical  evidence of the efficacy of this  strategy \cite{owhadi2019kernel} (interpretable as a variant of cross validation \cite{chen2020consistency}) suggests that deep learning could be performed by replacing backpropagation with forward cross-validation.

\subsection{Recent related papers}
We will now comment on two closely related recent papers posted after this manuscript (arXiv:2008.03920, Aug 2020):
``Momentum residual neural networks'' \cite{sander2021momentum}   (arXiv:2102.07870, Feb 2021) and ``Scaling Properties of Deep Residual Networks'' \cite{cohen2021scaling} (arXiv:2105.12245, May 2021).
 \cite{sander2021momentum} adds momentum variables to the ResNet ODE  Ansatz and interprets ResNets in the infinitesimal step size regime as a second-order ODE with increased representation properties and decreased memory footprint.
 In our proposed manuscript momentum variables and the second order ODE limit emerge as a byproduct of  $L^2$ small weights regularization.
\cite{cohen2021scaling} studies the scaling regime of ResNet weights trained by stochastic gradient descent (without small weights $L^2$ regularization in the loss) and their scaling with network depth through detailed numerical experiments. Several scaling regimes are observed in
\cite{cohen2021scaling} as a function of the regularity of the activation function. Using the notations of \eqref{eqjkehbwjhebdhdjq}, the regime associated with smooth activation functions corresponds to $\|w_s\|_{\L(\X\oplus \R,\X)}^2 \approx \mathcal{O}(\frac{1}{L})$, whereas the regime obtained with $L^2$ small weights regularization (studied here) corresponds to $\|w_s\|_{\L(\X\oplus \R,\X)}^2 \approx \mathcal{O}(\frac{1}{L^2})$.

\subsection{Difference with the original ResNet}
The  structure of the ResNet considered in this paper is
a particular case of the structure given in the original ResNet paper \cite{he2016deep} which encodes the
residual maps via a convolutional MLP with one hidden layer, which goes up in feature
space dimension and then back to the original space. In the setting of the proposed paper this difference is equivalent to replacing
 $v_s$ in \eqref{eqkjedjdiseu} with $v_b^s\circ v_a^s$ and $\|v_s\|_{\Gamma}^2$ in \eqref{eqlktddsyeytdsedhjd} with $\|v_s^a\|_{\Gamma}^2+\|v_s^b\|_{\Gamma}^2$. Viewing  $v^a_s$  as a hyperparameter for a warped kernel of the form $\Gamma^b(v^a_s(\cdot),v^a_s(\cdot))$, it follows that this difference is equivalent to replacing the
 the kernel $\Gamma$ by a parameterized kernel $\Gamma_{v^a_s}$ whose trained parameters $v^a_s$ depend on data and possibly $s$.
Another difference with the original ResNet architecture is the explicit $L^2$ small weights regularization, which in the continuous limit implies the preservation of the topology of the input space \cite{dupont2019augmented}.

\section{The elephant in the dark deep learning room and the shape of ideas}\label{seckejhdbeyudydgor}
Seeking to develop a theoretical understanding of deep learning can be compared to attempting to describe an elephant in a dark room \cite{rumi}. Rephrasing \cite{rumi}, ResNets  \cite{he2016deep} look like discretized ODEs \cite{weinan2017proposal, haber2017stable, chen2018neural, owhadi2019kernel}, the generalization properties of ANNs \cite{zhang2016understanding} feel like those of kernel methods \cite{belkin2018understand, jacot2018neural, owhadi2019kernel}, the functional form of ANNs is akin to that of deep kernels \cite{wilson2016deep},
there seems to be a natural relation between ANNs and deep Gaussian processes \cite{damianou2013deep}. Training ANNs with backpropagation seems to be related to the type of constrained minimization algorithms used in optimal control  \cite{lecun1988theoretical}.
The identification of  CNNs, and ResNets as algorithms obtained from the discretization of a GP generalization of image registration problems suggests that (1) ResNets are essentially image registration/computational anatomy algorithms generalized to abstract high dimensional spaces and (2) ideas do have shape and forming ideas can be expressed as manipulating their form in abstract RKHS spaces.
Evidently, this identification opens the possibility of (1) analyzing  deep learning the perspectives of
shape analysis \cite{younes2010shapes} and Computational Graph Completion \cite{owhadi2021computational},
(2) identifying good architectures by programming good kernels  \cite{owhadi2019kernelb} through computational graphs \cite{owhadi2021computational}.
  Although
 it is difficult to visualize shapes in high dimensional spaces,
we suspect that
deep-learning breaks the curse of dimensionality by (implicitly) employing kernels (such as REM kernels) exploiting\footnote{The corresponding RKHS norm of the target function should be small.
Although Barron space error estimates \cite{barron1993universal, ma2019barron} and RKHS error estimates (Thm.~\ref{thmalkjhbahkjedb}) do not depend on dimension, they rely on bounding the Barron/RKHS norm of the target function, which is the difficulty to be addressed.}
 universal patterns/structures in the shape of the data (e.g., the compositional nature of the world and its invariants under transformations).
Therefore analyzing ANNs in the setting of completing computational graphs with GPs \cite{owhadi2021computational} and
understanding interplays between learning and shapes/forms in high dimensional spaces may help us see ``\emph{the whole of the beast}'' \cite{rumi}. A \emph{beast} that bears some intriguing similarities with Plato's theory of forms\footnote{According to Plato's theory of forms, (1) {\it ``Ideas'' or ``Forms'', are the non-physical essences of all things, of which, objects and matter, in the physical world, are merely imitations} (\url{https://en.wikipedia.org/wiki/Theory_of_forms}) and (2)
{\it
The world can be divided
into two worlds, the visible and the intelligible.
We grasp the visible world with our senses.
The intelligible world we can only grasp  with our
 mind, it is comprised of the forms
\ldots Only the forms are objects of knowledge  because
only they possess the eternal, unchanging truth that the mind, not the
senses, must apprehend.
} (Randy Aust,
\url{https://www.youtube.com/watch?v=A7xjoHruQfY}).} \cite{Platoforms}.

\subsection*{Acknowledgments}
The author gratefully acknowledges support from the Air Force Office of Scientific Research under award number FA9550-18-1-0271 (Games for Computation and Learning) and MURI award number FA9550-20-1-0358 (Machine Learning and Physics-Based Modeling and Simulation). Thanks to Clint Scovel for a careful readthrough with detailed comments and feedback and to two anonymous referees for detailed comments and suggestions.

\bibliographystyle{plain}
\bibliography{merged,RPS,extra,kmd}

\def\cprime{$'$} \def\cprime{$'$} \def\cprime{$'$} \def\cprime{$'$}
  \def\cprime{$'$}
\begin{thebibliography}{100}

\bibitem{akian2022learning}
Jean-Luc Akian, Luc Bonnet, Houman Owhadi, and {\'E}ric Savin.
\newblock Learning" best" kernels from data in gaussian process regression.
  with application to aerodynamics.
\newblock {\em Journal of Computational Physics}, 470, 2022.

\bibitem{allassonniere2005geodesic}
St{\'e}phanie Allassonni{\`e}re, Alain Trouv{\'e}, and Laurent Younes.
\newblock Geodesic shooting and diffeomorphic matching via textured meshes.
\newblock In {\em International Workshop on Energy Minimization Methods in
  Computer Vision and Pattern Recognition}, pages 365--381. Springer, 2005.

\bibitem{alvarez2012kernels}
Mauricio~A Alvarez, Lorenzo Rosasco, Neil~D Lawrence, et~al.
\newblock Kernels for vector-valued functions: A review.
\newblock {\em Foundations and Trends{\textregistered} in Machine Learning},
  4(3):195--266, 2012.

\bibitem{an1996effects}
Guozhong An.
\newblock The effects of adding noise during backpropagation training on a
  generalization performance.
\newblock {\em Neural computation}, 8(3):643--674, 1996.

\bibitem{arino2006fundamental}
Julien Arino.
\newblock Fundamental theory of ordinary differential equations.
\newblock {\em Lecture Notes. University of Manitoba}, 2006.

\bibitem{bajgiran2021uncertainty}
Hamed~Hamze Bajgiran, Pau~Batlle Franch, Houman Owhadi, Clint Scovel, Mahdy
  Shirdel, Michael Stanley, and Peyman Tavallali.
\newblock Uncertainty quantification of the 4th kind; optimal posterior
  accuracy-uncertainty tradeoff with the minimum enclosing ball.
\newblock {\em arXiv preprint arXiv:2108.10517}, 2021.

\bibitem{rumi}
Coleman Barks.
\newblock The {E}ssential {R}umi.
\newblock In {\em Elephant in the Dark}, volume 252. HarperSanFrancisco, 1995.

\bibitem{barron1993universal}
Andrew~R Barron.
\newblock Universal approximation bounds for superpositions of a sigmoidal
  function.
\newblock {\em IEEE Transactions on Information theory}, 39(3):930--945, 1993.

\bibitem{baxendale1984brownian}
Peter Baxendale.
\newblock Brownian motions in the diffeomorphism group i.
\newblock {\em Compositio Mathematica}, 53(1):19--50, 1984.

\bibitem{beg2005computing}
M~Faisal Beg, Michael~I Miller, Alain Trouv{\'e}, and Laurent Younes.
\newblock Computing large deformation metric mappings via geodesic flows of
  diffeomorphisms.
\newblock {\em International journal of computer vision}, 61(2):139--157, 2005.

\bibitem{belkin2021fit}
Mikhail Belkin.
\newblock Fit without fear: remarkable mathematical phenomena of deep learning
  through the prism of interpolation.
\newblock {\em arXiv preprint arXiv:2105.14368}, 2021.

\bibitem{belkin2018understand}
Mikhail Belkin, Siyuan Ma, and Soumik Mandal.
\newblock To understand deep learning we need to understand kernel learning.
\newblock {\em arXiv preprint arXiv:1802.01396}, 2018.

\bibitem{blanes2017concise}
Sergio Blanes and Fernando Casas.
\newblock {\em A concise introduction to geometric numerical integration}.
\newblock CRC press, 2017.

\bibitem{bohn2019representer}
Bastian Bohn, Christian Rieger, and Michael Griebel.
\newblock A representer theorem for deep kernel learning.
\newblock {\em Journal of Machine Learning Research}, 20(64):1--32, 2019.

\bibitem{brown1992survey}
Lisa~Gottesfeld Brown.
\newblock A survey of image registration techniques.
\newblock {\em ACM computing surveys (CSUR)}, 24(4):325--376, 1992.

\bibitem{bruveris2011momentum}
Martins Bruveris, Fran{\c{c}}ois Gay-Balmaz, Darryl~D Holm, and Tudor~S Ratiu.
\newblock The momentum map representation of images.
\newblock {\em Journal of nonlinear science}, 21(1):115--150, 2011.

\bibitem{bruveris2017completeness}
Martins Bruveris and Fran{\c{c}}ois-Xavier Vialard.
\newblock On completeness of groups of diffeomorphisms.
\newblock {\em Journal of the European Mathematical Society}, 19(5):1507--1544,
  2017.

\bibitem{camion2001geodesic}
Vincent Camion and Laurent Younes.
\newblock Geodesic interpolating splines.
\newblock In {\em International workshop on energy minimization methods in
  computer vision and pattern recognition}, pages 513--527. Springer, 2001.

\bibitem{carreira2014distributed}
Miguel Carreira-Perpinan and Weiran Wang.
\newblock Distributed optimization of deeply nested systems.
\newblock In {\em Artificial Intelligence and Statistics}, pages 10--19. PMLR,
  2014.

\bibitem{casetti1996riemannian}
Lapo Casetti, Cecilia Clementi, and Marco Pettini.
\newblock Riemannian theory of {H}amiltonian chaos and {L}yapunov exponents.
\newblock {\em Physical Review E}, 54(6):5969, 1996.

\bibitem{chan2015pcanet}
Tsung-Han Chan, Kui Jia, Shenghua Gao, Jiwen Lu, Zinan Zeng, and Yi~Ma.
\newblock Pcanet: A simple deep learning baseline for image classification?
\newblock {\em IEEE transactions on image processing}, 24(12):5017--5032, 2015.

\bibitem{chang2018reversible}
Bo~Chang, Lili Meng, Eldad Haber, Lars Ruthotto, David Begert, and Elliot
  Holtham.
\newblock Reversible architectures for arbitrarily deep residual neural
  networks.
\newblock In {\em Thirty-Second AAAI Conference on Artificial Intelligence},
  2018.

\bibitem{charon2018metamorphoses}
Nicolas Charon, Benjamin Charlier, and Alain Trouv{\'e}.
\newblock Metamorphoses of functional shapes in {S}obolev spaces.
\newblock {\em Foundations of Computational Mathematics}, 18(6):1535--1596,
  2018.

\bibitem{chen2018neural}
Tian~Qi Chen, Yulia Rubanova, Jesse Bettencourt, and David~K Duvenaud.
\newblock Neural ordinary differential equations.
\newblock In {\em Advances in neural information processing systems}, pages
  6571--6583, 2018.

\bibitem{chen2021solving}
Yifan Chen, Bamdad Hosseini, Houman Owhadi, and Andrew~M Stuart.
\newblock Solving and learning nonlinear pdes with gaussian processes.
\newblock {\em Journal of Computational Physics}, 447, 2021.

\bibitem{chen2020consistency}
Yifan Chen, Houman Owhadi, and Andrew Stuart.
\newblock Consistency of empirical bayes and kernel flow for hierarchical
  parameter estimation.
\newblock {\em Mathematics of Computation}, 90(332):2527--2578, 2021.

\bibitem{choromanska2019beyond}
Anna Choromanska, Benjamin Cowen, Sadhana Kumaravel, Ronny Luss, Mattia
  Rigotti, Irina Rish, Paolo Diachille, Viatcheslav Gurev, Brian Kingsbury,
  Ravi Tejwani, et~al.
\newblock Beyond backprop: Online alternating minimization with auxiliary
  variables.
\newblock In {\em International Conference on Machine Learning}, pages
  1193--1202. PMLR, 2019.

\bibitem{cockayne2019bayesian}
Jon Cockayne, Chris~J Oates, Timothy~John Sullivan, and Mark Girolami.
\newblock Bayesian probabilistic numerical methods.
\newblock {\em SIAM Review}, 61(4):756--789, 2019.

\bibitem{cohen2021scaling}
Alain-Sam Cohen, Rama Cont, Alain Rossier, and Renyuan Xu.
\newblock Scaling properties of deep residual networks.
\newblock {\em arXiv preprint arXiv:2105.12245}, 2021.

\bibitem{cohen2016group}
Taco Cohen and Max Welling.
\newblock Group equivariant convolutional networks.
\newblock In {\em International conference on machine learning}, pages
  2990--2999, 2016.

\bibitem{damianou2013deep}
Andreas Damianou and Neil Lawrence.
\newblock Deep {G}aussian processes.
\newblock In {\em Artificial Intelligence and Statistics}, pages 207--215,
  2013.

\bibitem{dunlop2020hyperparameter}
Matthew~M Dunlop, Tapio Helin, and Andrew~M Stuart.
\newblock Hyperparameter estimation in bayesian map estimation:
  parameterizations and consistency.
\newblock {\em The SMAI journal of computational mathematics}, 6:69--100, 2020.

\bibitem{dupont2019augmented}
Emilien Dupont, Arnaud Doucet, and Yee~Whye Teh.
\newblock Augmented neural odes.
\newblock {\em Advances in Neural Information Processing Systems}, 32, 2019.

\bibitem{dupuis1998variational}
Paul Dupuis, Ulf Grenander, and Michael~I Miller.
\newblock Variational problems on flows of diffeomorphisms for image matching.
\newblock {\em Quarterly of applied mathematics}, pages 587--600, 1998.

\bibitem{ma2019barron}
Weinan E, Chao Ma, and Lei Wu.
\newblock Barron spaces and the compositional function spaces for neural
  network models.
\newblock {\em arXiv preprint arXiv:1906.08039}, 2019.

\bibitem{fishbaugh2013geodesic}
James Fishbaugh, Marcel Prastawa, Guido Gerig, and Stanley Durrleman.
\newblock Geodesic image regression with a sparse parameterization of
  diffeomorphisms.
\newblock In {\em International Conference on Geometric Science of
  Information}, pages 95--102. Springer, 2013.

\bibitem{glaunes2004diffeomorphic}
Joan Glaunes, Alain Trouv{\'e}, and Laurent Younes.
\newblock Diffeomorphic matching of distributions: A new approach for
  unlabelled point-sets and sub-manifolds matching.
\newblock In {\em Proceedings of the 2004 IEEE Computer Society Conference on
  Computer Vision and Pattern Recognition, 2004. CVPR 2004.}, volume~2, pages
  II--II. IEEE, 2004.

\bibitem{grenander1998computational}
Ulf Grenander and Michael~I Miller.
\newblock Computational anatomy: An emerging discipline.
\newblock {\em Quarterly of applied mathematics}, 56(4):617--694, 1998.

\bibitem{greydanus2019hamiltonian}
Samuel Greydanus, Misko Dzamba, and Jason Yosinski.
\newblock Hamiltonian neural networks.
\newblock {\em Advances in Neural Information Processing Systems}, 32, 2019.

\bibitem{gulcehre2016noisy}
Caglar Gulcehre, Marcin Moczulski, Misha Denil, and Yoshua Bengio.
\newblock Noisy activation functions.
\newblock In {\em International conference on machine learning}, pages
  3059--3068, 2016.

\bibitem{haasdonk2005invariance}
Bernard Haasdonk, A~Vossen, and Hans Burkhardt.
\newblock Invariance in kernel methods by {H}aar-integration kernels.
\newblock In {\em Scandinavian Conference on Image Analysis}, pages 841--851.
  Springer, 2005.

\bibitem{haber2017stable}
Eldad Haber and Lars Ruthotto.
\newblock Stable architectures for deep neural networks.
\newblock {\em Inverse Problems}, 34(1):014004, 2017.

\bibitem{hairer2003geometric}
Ernst Hairer, Christian Lubich, and Gerhard Wanner.
\newblock Geometric numerical integration illustrated by the
  {S}t{\"o}rmer--{V}erlet method.
\newblock {\em Acta numerica}, 12:399--450, 2003.

\bibitem{hairer2006geometric}
Ernst Hairer, Christian Lubich, and Gerhard Wanner.
\newblock {\em Geometric numerical integration: structure-preserving algorithms
  for ordinary differential equations}, volume~31.
\newblock Springer Science \& Business Media, 2006.

\bibitem{hamzi2021simple}
B~Hamzi, R~Maulik, and H~Owhadi.
\newblock Simple, low-cost and accurate data-driven geophysical forecasting
  with learned kernels.
\newblock {\em Proceedings of the Royal Society A}, 477(2252):20210326, 2021.

\bibitem{hamzi2021learning}
Boumediene Hamzi and Houman Owhadi.
\newblock Learning dynamical systems from data: A simple cross-validation
  perspective, part i: Parametric kernel flows.
\newblock {\em Physica D: Nonlinear Phenomena}, 421:132817, 2021.

\bibitem{han2019mean}
Jiequn Han, Qianxiao Li, et~al.
\newblock A mean-field optimal control formulation of deep learning.
\newblock {\em Research in the Mathematical Sciences}, 6(1):1--41, 2019.

\bibitem{hart2009optimal}
Gabriel~L Hart, Christopher Zach, and Marc Niethammer.
\newblock An optimal control approach for deformable registration.
\newblock In {\em 2009 IEEE Computer Society Conference on Computer Vision and
  Pattern Recognition Workshops}, pages 9--16. IEEE, 2009.

\bibitem{he2016deep}
Kaiming He, Xiangyu Zhang, Shaoqing Ren, and Jian Sun.
\newblock Deep residual learning for image recognition.
\newblock In {\em Proceedings of the IEEE conference on computer vision and
  pattern recognition}, pages 770--778, 2016.

\bibitem{hennig2015probabilisticgiro}
Philipp Hennig, Michael~A Osborne, and Mark Girolami.
\newblock Probabilistic numerics and uncertainty in computations.
\newblock {\em Proceedings of the Royal Society A: Mathematical, Physical and
  Engineering Sciences}, 471(2179):20150142, 2015.

\bibitem{holm2009euler}
Darryl Holm, Alain Trouv{\'e}, and Laurent Younes.
\newblock The euler-poincar{\'e} theory of metamorphosis.
\newblock {\em Quarterly of Applied Mathematics}, 67(4):661--685, 2009.

\bibitem{holm2005momentum}
Darryl~D Holm and Jerrold~E Marsden.
\newblock Momentum maps and measure-valued solutions (peakons, filaments, and
  sheets) for the epdiff equation.
\newblock In {\em The breadth of symplectic and Poisson geometry}, pages
  203--235. Springer, 2005.

\bibitem{holmstrom1992using}
Lasse Holmstrom and Petri Koistinen.
\newblock Using additive noise in back-propagation training.
\newblock {\em IEEE transactions on neural networks}, 3(1):24--38, 1992.

\bibitem{huang2017densely}
Gao Huang, Zhuang Liu, Laurens Van Der~Maaten, and Kilian~Q Weinberger.
\newblock Densely connected convolutional networks.
\newblock In {\em Proceedings of the IEEE conference on computer vision and
  pattern recognition}, pages 4700--4708, 2017.

\bibitem{jacot2018neural}
Arthur Jacot, Franck Gabriel, and Cl{\'e}ment Hongler.
\newblock Neural tangent kernel: Convergence and generalization in neural
  networks.
\newblock In {\em Advances in neural information processing systems}, pages
  8571--8580, 2018.

\bibitem{joshi1998large}
Sarang~C Joshi.
\newblock {\em Large deformation diffeomorphisms and Gaussian random fields for
  statistical characterization of brain sub-manifolds}.
\newblock PhD thesis, Washington University, 1998.

\bibitem{joshi2000landmark}
Sarang~C Joshi and Michael~I Miller.
\newblock Landmark matching via large deformation diffeomorphisms.
\newblock {\em IEEE transactions on image processing}, 9(8):1357--1370, 2000.

\bibitem{kadri2016operator}
Hachem Kadri, Emmanuel Duflos, Philippe Preux, St{\'e}phane Canu, Alain
  Rakotomamonjy, and Julien Audiffren.
\newblock Operator-valued kernels for learning from functional response data.
\newblock {\em The Journal of Machine Learning Research}, 17(1):613--666, 2016.

\bibitem{kunita1997stochastic}
Hiroshi Kunita.
\newblock {\em Stochastic flows and stochastic differential equations},
  volume~24.
\newblock Cambridge university press, 1997.

\bibitem{lecun1995convolutional}
Yann LeCun, Yoshua Bengio, et~al.
\newblock Convolutional networks for images, speech, and time series.
\newblock {\em The handbook of brain theory and neural networks},
  3361(10):1995, 1995.

\bibitem{lecun2015deep}
Yann LeCun, Yoshua Bengio, and Geoffrey Hinton.
\newblock Deep learning.
\newblock {\em {N}ature}, 521(7553):436--444, 2015.

\bibitem{lecun1999object}
Yann LeCun, Patrick Haffner, L{\'e}on Bottou, and Yoshua Bengio.
\newblock Object recognition with gradient-based learning.
\newblock In {\em Shape, contour and grouping in computer vision}, pages
  319--345. Springer, 1999.

\bibitem{lecun1988theoretical}
Yann LeCun, D~Touresky, G~Hinton, and T~Sejnowski.
\newblock A theoretical framework for back-propagation.
\newblock In {\em Proceedings of the 1988 connectionist models summer school},
  volume~1, pages 21--28. CMU, Pittsburgh, Pa: Morgan Kaufmann, 1988.

\bibitem{LiChenetal17}
Qianxiao Li, Long Chen, Cheng Tai, and Weinan E.
\newblock Maximum principle based algorithms for deep learning.
\newblock {\em J. Mach. Learn. Res.}, 18:Paper No. 165, 29, 2017.

\bibitem{li1975existence}
Tien-Yien Li.
\newblock Existence of solutions for ordinary differential equations in
  {B}anach spaces.
\newblock {\em Journal of Differential Equations}, 18(1):29--40, 1975.

\bibitem{marsden2013introduction}
Jerrold~E Marsden and Tudor~S Ratiu.
\newblock {\em Introduction to mechanics and symmetry: a basic exposition of
  classical mechanical systems}, volume~17.
\newblock Springer Science \& Business Media, 2013.

\bibitem{marsden2001discrete}
Jerrold~E Marsden and Matthew West.
\newblock Discrete mechanics and variational integrators.
\newblock {\em Acta Numerica}, 10:357--514, 2001.

\bibitem{brittlnessmachinelearning}
Mike McKerns.
\newblock Mystic: a framework for predictive science; {S}ci{P}y 2013
  presentation;
  \url{https://www.youtube.com/watch?v=o-nwSnLC6DU&feature=youtu.be&t=74}.

\bibitem{mei2018mean}
Song Mei, Andrea Montanari, and Phan-Minh Nguyen.
\newblock A mean field view of the landscape of two-layer neural networks.
\newblock {\em Proceedings of the National Academy of Sciences},
  115(33):E7665--E7671, 2018.

\bibitem{micchelli2005kernels}
Charles~A Micchelli and Massimiliano Pontil.
\newblock Kernels for multi--task learning.
\newblock In {\em Advances in neural information processing systems}, pages
  921--928, 2005.

\bibitem{micheli2008differential}
Mario Micheli.
\newblock {\em The differential geometry of landmark shape manifolds: metrics,
  geodesics, and curvature}.
\newblock PhD thesis, Brown University, 2008.

\bibitem{micheli2012sectional}
Mario Micheli, Peter~W Michor, and David Mumford.
\newblock Sectional curvature in terms of the cometric, with applications to
  the riemannian manifolds of landmarks.
\newblock {\em SIAM Journal on Imaging Sciences}, 5(1):394--433, 2012.

\bibitem{miller2002metrics}
Michael~I Miller, Alain Trouv{\'e}, and Laurent Younes.
\newblock On the metrics and {E}uler-{L}agrange equations of computational
  anatomy.
\newblock {\em Annual review of biomedical engineering}, 4(1):375--405, 2002.

\bibitem{miller2006geodesic}
Michael~I Miller, Alain Trouv{\'e}, and Laurent Younes.
\newblock Geodesic shooting for computational anatomy.
\newblock {\em Journal of mathematical imaging and vision}, 24(2):209--228,
  2006.

\bibitem{muller2004gamma}
Stefan M{\"u}ller and Michael Ortiz.
\newblock On the $\gamma$-convergence of discrete dynamics and variational
  integrators.
\newblock {\em Journal of Nonlinear Science}, 14(3):279--296, 2004.

\bibitem{nelsen2020random}
Nicholas~H Nelsen and Andrew~M Stuart.
\newblock The random feature model for input-output maps between {B}anach
  spaces.
\newblock {\em arXiv preprint arXiv:2005.10224}, 2020.

\bibitem{OwhScobook2018}
H.~Owhadi and C.~Scovel.
\newblock {\em Operator Adapted Wavelets, Fast Solvers, and Numerical
  Homogenization, from a game theoretic approach to numerical approximation and
  algorithm design}.
\newblock Cambridge Monographs on Applied and Computational Mathematics.
  Cambridge University Press, 2019.

\bibitem{owhadiyoutube20}
Houman Owhadi.
\newblock {\em Do ideas have shape? Plato's theory of forms as the continuous
  limit of artificial neutral networks}.
\newblock Fields Institute, Youtube.
\newblock \url{https://youtu.be/Oyl-4EqNz5o}.

\bibitem{owhadi2015bayesian}
Houman Owhadi.
\newblock Bayesian numerical homogenization.
\newblock {\em Multiscale Modeling \& Simulation}, 13(3):812--828, 2015.

\bibitem{owhadi2017multigrid}
Houman Owhadi.
\newblock Multigrid with rough coefficients and multiresolution operator
  decomposition from hierarchical information games.
\newblock {\em SIAM Review}, 59(1):99--149, 2017.

\bibitem{owhadi2021computational}
Houman Owhadi.
\newblock Computational graph completion.
\newblock {\em Research in the Mathematical Sciences}, 9(2):1--33, 2022.

\bibitem{owhadi2013brittleness}
Houman Owhadi and Clint Scovel.
\newblock Brittleness of {B}ayesian inference and new {S}elberg formulas.
\newblock {\em Communications in Mathematical Sciences}, 14(1):83--145, 2016.

\bibitem{owhadi2017qualitative}
Houman Owhadi and Clint Scovel.
\newblock Qualitative robustness in {B}ayesian inference.
\newblock {\em ESAIM: Probability and Statistics}, 21:251--274, 2017.

\bibitem{owhadi2019operator}
Houman Owhadi and Clint Scovel.
\newblock {\em Operator-Adapted Wavelets, Fast Solvers, and Numerical
  Homogenization: From a Game Theoretic Approach to Numerical Approximation and
  Algorithm Design}, volume~35.
\newblock Cambridge University Press, 2019.

\bibitem{owhadi2019statistical}
Houman Owhadi, Clint Scovel, and Florian Sch{\"a}fer.
\newblock Statistical numerical approximation.
\newblock {\em Notices of the AMS}, 2019.

\bibitem{owhadi2015brittlenessb}
Houman Owhadi, Clint Scovel, and Tim Sullivan.
\newblock On the brittleness of {B}ayesian inference.
\newblock {\em SIAM Review}, 57(4):566--582, 2015.

\bibitem{owhadi2015brittlenessa}
Houman Owhadi, Clint Scovel, Tim Sullivan, et~al.
\newblock Brittleness of bayesian inference under finite information in a
  continuous world.
\newblock {\em Electronic Journal of Statistics}, 9(1):1--79, 2015.
\newblock arXiv:1304.6772 (April 2013).

\bibitem{owhadi2019kernelb}
Houman Owhadi, Clint Scovel, and Gene~Ryan Yoo.
\newblock {\em Kernel Mode Decomposition and the programming of kernels}.
\newblock Springer, 2021.
\newblock arXiv:1907.08592.

\bibitem{owhadi2019kernel}
Houman Owhadi and Gene~Ryan Yoo.
\newblock Kernel flows: From learning kernels from data into the abyss.
\newblock {\em Journal of Computational Physics}, 389:22--47, 2019.

\bibitem{OwZh:2007a}
Houman Owhadi and Lei Zhang.
\newblock Metric-based upscaling.
\newblock {\em Communications on Pure and Applied Mathematics: A Journal Issued
  by the Courant Institute of Mathematical Sciences}, 60(5):675--723, 2007.

\bibitem{perrin1999modelling}
O~Perrin and P~Monestiez.
\newblock Modelling of non-stationary spatial structure using parametric radial
  basis deformations.
\newblock In {\em {G}eoENV II—Geostatistics for Environmental Applications},
  pages 175--186. Springer, 1999.

\bibitem{Platoforms}
Plato.
\newblock {\em The Republic}, volume {VII}.
\newblock 375 {BCE}.

\bibitem{rahimi2008random}
Ali Rahimi and Benjamin Recht.
\newblock Random features for large-scale kernel machines.
\newblock In {\em Advances in neural information processing systems}, pages
  1177--1184, 2008.

\bibitem{raissi2019physics}
Maziar Raissi, Paris Perdikaris, and George~E Karniadakis.
\newblock Physics-informed neural networks: A deep learning framework for
  solving forward and inverse problems involving nonlinear partial differential
  equations.
\newblock {\em Journal of Computational Physics}, 378:686--707, 2019.

\bibitem{reisert2007learning}
Marco Reisert and Hans Burkhardt.
\newblock Learning equivariant functions with matrix valued kernels.
\newblock {\em Journal of Machine Learning Research}, 8(Mar):385--408, 2007.

\bibitem{rico1993continuous}
Ramiro Rico-Martinez and Ioannis~G Kevrekidis.
\newblock Continuous time modeling of nonlinear systems: A neural network-based
  approach.
\newblock In {\em IEEE International Conference on Neural Networks}, pages
  1522--1525. IEEE, 1993.

\bibitem{rotskoff2018neural}
Grant~M Rotskoff and Eric Vanden-Eijnden.
\newblock Neural networks as interacting particle systems: Asymptotic convexity
  of the loss landscape and universal scaling of the approximation error.
\newblock {\em Stat}, 1050:22, 2018.

\bibitem{rousseau2018residual}
Fran{\c{c}}ois Rousseau and Ronan Fablet.
\newblock Residual networks as geodesic flows of diffeomorphisms.
\newblock {\em arXiv preprint arXiv:1805.09585}, 2018.

\bibitem{sabour2017dynamic}
Sara Sabour, Nicholas Frosst, and Geoffrey~E Hinton.
\newblock Dynamic routing between capsules.
\newblock In {\em Advances in neural information processing systems}, pages
  3856--3866, 2017.

\bibitem{sampson1992nonparametric}
Paul~D Sampson and Peter Guttorp.
\newblock Nonparametric estimation of nonstationary spatial covariance
  structure.
\newblock {\em Journal of the American Statistical Association},
  87(417):108--119, 1992.

\bibitem{sander2021momentum}
Michael~E Sander, Pierre Ablin, Mathieu Blondel, and Gabriel Peyr{\'e}.
\newblock Momentum residual neural networks.
\newblock {\em arXiv preprint arXiv:2102.07870}, 2021.

\bibitem{schafer2021sparse}
Florian Sch\"{a}fer, Matthias Katzfuss, and Houman Owhadi.
\newblock Sparse cholesky factorization by kullback--leibler minimization.
\newblock {\em SIAM Journal on Scientific Computing}, 43(3):A2019--A2046, 2021.

\bibitem{schafer2021compression}
Florian Sch\"{a}fer, TJ~Sullivan, and Houman Owhadi.
\newblock Compression, inversion, and approximate pca of dense kernel matrices
  at near-linear computational complexity.
\newblock {\em Multiscale Modeling \& Simulation}, 19(2):688--730, 2021.

\bibitem{schmidt2003bayesian}
Alexandra~M Schmidt and Anthony O'Hagan.
\newblock Bayesian inference for non-stationary spatial covariance structure
  via spatial deformations.
\newblock {\em Journal of the Royal Statistical Society: Series B (Statistical
  Methodology)}, 65(3):743--758, 2003.

\bibitem{shirdel2021deep}
Mahdy Shirdel, Reza Asadi, Duc Do, and Micheal Hintlian.
\newblock Deep learning with kernel flow regularization for time series
  forecasting.
\newblock {\em arXiv preprint arXiv:2109.11649}, 2021.

\bibitem{mflann22}
Alexandre Smirnov, Boumediene Hamzi, and Houman Owhadi.
\newblock Mean-field limits of trained weights in deep learning: A dynamical
  systems perspective.
\newblock {\em RG.2.2.26186.24007}, 2022.

\bibitem{srivastava2014dropout}
Nitish Srivastava, Geoffrey Hinton, Alex Krizhevsky, Ilya Sutskever, and Ruslan
  Salakhutdinov.
\newblock Dropout: a simple way to prevent neural networks from overfitting.
\newblock {\em The journal of machine learning research}, 15(1):1929--1958,
  2014.

\bibitem{steinwart2008support}
Ingo Steinwart and Andreas Christmann.
\newblock {\em Support vector machines}.
\newblock Springer Science \& Business Media, 2008.

\bibitem{still2018lectures}
Georg Still.
\newblock Lectures on parametric optimization: An introduction.
\newblock {\em Optimization Online}, 2018.

\bibitem{szegedy2013intriguing}
Christian Szegedy, Wojciech Zaremba, Ilya Sutskever, Joan Bruna, Dumitru Erhan,
  Ian Goodfellow, and Rob Fergus.
\newblock Intriguing properties of neural networks.
\newblock {\em arXiv preprint arXiv:1312.6199}, 2013.

\bibitem{tao2016explicit}
Molei Tao.
\newblock Explicit symplectic approximation of nonseparable {H}amiltonians:
  Algorithm and long time performance.
\newblock {\em Physical Review E}, 94(4):043303, 2016.

\bibitem{teixeira2005strong}
Eduardo~V Teixeira.
\newblock Strong solutions for differential equations in abstract spaces.
\newblock {\em Journal of Differential Equations}, 214(1):65--91, 2005.

\bibitem{thorpe2018deep}
Matthew Thorpe and Yves van Gennip.
\newblock Deep limits of residual neural networks.
\newblock {\em arXiv preprint arXiv:1810.11741}, 2018.

\bibitem{trouve1998diffeomorphisms}
Alain Trouv{\'e}.
\newblock Diffeomorphisms groups and pattern matching in image analysis.
\newblock {\em International journal of computer vision}, 28(3):213--221, 1998.

\bibitem{trouve2005metamorphoses}
Alain Trouv{\'e} and Laurent Younes.
\newblock Metamorphoses through lie group action.
\newblock {\em Foundations of computational mathematics}, 5(2):173--198, 2005.

\bibitem{vialard2020shooting}
Fran{\c{c}}ois-Xavier Vialard, Roland Kwitt, Susan Wei, and Marc Niethammer.
\newblock A shooting formulation of deep learning.
\newblock {\em Advances in Neural Information Processing Systems}, 33, 2020.

\bibitem{vialard2012diffeomorphic}
Fran{\c{c}}ois-Xavier Vialard, Laurent Risser, Daniel Rueckert, and Colin~J
  Cotter.
\newblock Diffeomorphic 3d image registration via geodesic shooting using an
  efficient adjoint calculation.
\newblock {\em International Journal of Computer Vision}, 97(2):229--241, 2012.

\bibitem{weinan2017proposal}
E~Weinan.
\newblock A proposal on machine learning via dynamical systems.
\newblock {\em Communications in Mathematics and Statistics}, 5(1):1--11, 2017.

\bibitem{west2004variational}
Matthew West.
\newblock {\em Variational integrators}.
\newblock PhD thesis, California Institute of Technology, 2004.

\bibitem{wilson2016deep}
Andrew~Gordon Wilson, Zhiting Hu, Ruslan Salakhutdinov, and Eric~P Xing.
\newblock Deep kernel learning.
\newblock In {\em Artificial intelligence and statistics}, pages 370--378,
  2016.

\bibitem{wu1993local}
Zong-min Wu and Robert Schaback.
\newblock Local error estimates for radial basis function interpolation of
  scattered data.
\newblock {\em IMA journal of Numerical Analysis}, 13(1):13--27, 1993.

\bibitem{yoo2020deep}
Gene~Ryan Yoo and Houman Owhadi.
\newblock Deep regularization and direct training of the inner layers of neural
  networks with kernel flows.
\newblock {\em Physica D: Nonlinear Phenomena}, 426:132952, 2021.

\bibitem{younes1998computable}
Laurent Younes.
\newblock Computable elastic distances between shapes.
\newblock {\em SIAM Journal on Applied Mathematics}, 58(2):565--586, 1998.

\bibitem{younes2010shapes}
Laurent Younes.
\newblock {\em Shapes and diffeomorphisms}, volume 171.
\newblock Springer, 2010.

\bibitem{younes2019diffeomorphic}
Laurent Younes.
\newblock Diffeomorphic matching.
\newblock In {\em Shapes and Diffeomorphisms}, pages 291--346. Springer, 2019.

\bibitem{zammit2019deep}
Andrew Zammit-Mangion, Tin Lok~James Ng, Quan Vu, and Maurizio Filippone.
\newblock Deep compositional spatial models.
\newblock {\em arXiv preprint arXiv:1906.02840}, 2019.

\bibitem{zhang2016understanding}
Chiyuan Zhang, Samy Bengio, Moritz Hardt, Benjamin Recht, and Oriol Vinyals.
\newblock Understanding deep learning requires rethinking generalization.
\newblock {\em arXiv preprint arXiv:1611.03530}, 2016.

\end{thebibliography}

\newpage
\section*{Appendix}

\section{Operator-valued kernels}\label{secovk}
Through this manuscript we  employ (with slight variations) the setting of operator-valued kernels introduced in \cite{kadri2016operator} (as a generalization of vector-valued kernels \cite{alvarez2012kernels}). This section provides a short reminder on operator-valued kernels.

\subsection{A short reminder}\label{subseciuwge67d}
Let $\X$ and $\Y$ be as in Subsec.~\ref{subsett1}.
\begin{Definition}
We call
$
K\,:\, \X\times \X\rightarrow \L(\Y)
$
  an {\bf operator-valued kernel} if
\begin{enumerate}
\item
$K$ is Hermitian, i.e.
\begin{equation}
K(x,x')=K(x',x)^T\text{ for }x,x'\in \X\,,
\end{equation}
writing $A^T$ for the adjoint of the operator $A$ with respect to $\<\cdot,\cdot\>_\Y$, and
\item non-negative, i.e.
 \begin{equation}
 \sum_{i,j=1}^m   \<y_i, K(x_i,x_j) y_j\>_\Y\geq 0 \text{ for }(x_i, y_i)\in \X\times \Y,\,m\in \mathbb{N}\,.
 \end{equation}
 \end{enumerate}
We call $K$ non-degenerate if $\sum_{i,j=1}^m   \<y_i, K(x_i,x_j) y_j\>_\Y= 0$ implies $y_i=0$ for all $i$ whenever  $x_i\not=x_j$ for $i\not=j$.
\end{Definition}
The following definition provides a simple example of
operator-valued kernels obtained from scalar-valued kernels.
\begin{Definition}\label{deflkekdlkjd}
We say that $K\,:\, \X\times \X\rightarrow \L(\Y)$ is {\bf scalar} if $K(x,x')=k(x,x') I_\Y$  (writing $I_\Y$ for the identity operator on $\Y$) for some scalar-valued  kernel $k\,:\, \X\times \X\rightarrow \R$, i.e.
\begin{equation}
\<y, K(x,x') y'\>_\Y= k(x,x') \<y,y'\>_\Y\text{ for }x,x'\in \X \text{ and }y,y'\in\Y\,.
\end{equation}
\end{Definition}

\begin{Example}
$\X=\R^{d_\X}$, $\<x,x'\>_\X=x^T x'$, $\Y=\R^{d_\Y}$ and $\<y,y'\>_\Y=y^T y' $ are prototypical examples.
For ease of presentation, we will continue using the notation $\<y,y'\>_\Y=y^T y'$ even when $\Y$ is arbitrary.
\end{Example}

Each non-degenerate, locally bounded and separately continuous  operator-valued kernel $K$ (which we will refer to as a Mercer's kernel) is in one to one correspondence with a  reproducing kernel Hilbert space $\H_K$ of continuous functions $f\,:\, \X\rightarrow \Y$  obtained \cite[Thm.~1,2]{kadri2016operator}  as the closure of the linear span of functions $z\rightarrow K(z,x)y$ ($(x,y)\in \X\times \Y$) with respect to the inner product $\<\cdot, \cdot\>_K$ identified by the reproducing property
\begin{equation}\label{eqrepprop}
\<f, K(\cdot,x) y\>_K=\<f(x),y\>_\Y
\end{equation}

\subsection{Feature maps}\label{subsecfeaturemaps}
Let $\F$ be a  separable Hilbert space (with inner product $\<\cdot,\cdot\>_\F$ and norm $\|\cdot\|_\F$) and let
$\psi\,:\, \X \rightarrow \L(\Y,\F)$ be a continuous function mapping $\X$ to the space of bounded linear operators from
$\Y$ to $\F$.
\begin{Definition}
We say that $\F$ and $\psi\,:\, \X \rightarrow \L(\Y,\F)$ are a \emph{feature space} and  a \emph{feature map} for the kernel $K$ if, for all $(x,x',y,y')\in \X^2 \times \Y^2$,
\begin{equation}
y^T K(x,x')y'= \<\psi(x) y,\psi(x') y'\>_\F\,.
\end{equation}
\end{Definition}
Write $\psi^T(x)$, for the adjoint of $\psi(x)$ defined as the linear function mapping $\F$ to $\Y$ satisfying
\begin{equation}
\<\psi(x) y, \alpha\>_\F=\< y, \psi^T(x)\alpha\>_\Y
\end{equation}
 for
 $x,y,\alpha \in \X\times \Y \times \F$. Note that $\psi^T\,:\, \X\rightarrow \L(\F,\Y)$ is therefore a function mapping $\X$ to the space of bounded linear functions from $\F$ to $\Y$. Writing $\alpha^T \alpha':=\<\alpha,\alpha'\>_\F$ for the inner product in $\F$  we can ease our notations by writing
 \begin{equation}\label{eqkledkjdkejdd}
 K(x,x')= \psi^T(x) \psi(x')\,
\end{equation}
 which is consistent with the finite-dimensional setting and $y^T K(x,x') y'=(\psi(x) y)^T (\psi(x')y')$ (writing $y^T y'$ for the inner product in $\Y$).
 For $\alpha \in \F$ write $\psi^T \alpha$ for the function  $\X\rightarrow \Y$ mapping $x\in \X$ to the element $y\in \Y$ such that
 \begin{equation}
 \<y',y\>_\Y=\<y',\psi^T(x)\alpha\>_\Y=\<\psi(x) y',\alpha\>_\F \text{ for all }y'\in \Y\,.
 \end{equation}
 We can, without loss of generality, restrict $\F$ to be the range of $(x,y)\rightarrow \psi(x)y$ so that the RKHS $\H_K$ defined by
 $K$ is the (closure of) linear space spanned by $\psi^T \alpha $ for $\alpha \in \F$.
Note that the reproducing property \eqref{eqrepprop} implies that for $\alpha\in \F$
\begin{equation}
\<\psi^T(\cdot) \alpha, \psi^T(\cdot) \psi(x) y\>_K=\<\psi^T(x) \alpha,y\>_\Y=\<\alpha, \psi(x) y\>_\F
\end{equation}
for all $x,y\in \X\times \Y$, which leads to the following theorem.
\begin{Theorem}\label{thmkjhkejhddjhd}
The RKHS $\H_K$ defined by the kernel \eqref{eqkledkjdkejdd} is the linear span of $\psi^T \alpha$ over  $\alpha \in \F $ such that $\|\alpha\|_\F< \infty$. Furthermore, $\<\psi^T(\cdot) \alpha, \psi^T(\cdot) \alpha'\>_K=\<\alpha,\alpha'\>_\F$ and
\begin{equation}
\|\psi^T(\cdot) \alpha\|^2_K=\|\alpha\|^2_\F
\text{ for }\alpha,\alpha'\in \F\,.
\end{equation}
\end{Theorem}

\section{Existence, uniqueness, and convergence of minimizers}\label{seclkjdekjdhjdex}

This section presents existence, uniqueness, and convergence results on the minimizers of \eqref{eqlktddsyeytdsedhjd}  and \eqref{eqlkjgehgddjedhjdB}.

\subsection{Existence/identification of minimizers and  energy preservation}\label{subseckjjhgdejwedh}
Although it is simple to show the existence of minimizers for \eqref{eqlsweedhsejd}, \eqref{eqkjlkjkejwkdj} and \eqref{eqlkjgehgddjedhjd} (see Thm.~\ref{thmksbsahshkd} below), we will not attempt to identify sufficient conditions for their uniqueness since
pathological landmark matching examples \cite{micheli2012sectional} suggest that, even with smooth kernels, these
minimizers may not be unique\footnote{For a simple example, consider a rigid pendulum spinning about the origin. Starting from the stable equilibrium point (pendulum down, zero velocity), consider the problem of finding a minimal energy initial momentum arriving at the unstable equilibrium (pendulum up) with zero velocity. This problem is analogous to minimizing \eqref{eqlsweedhsejd} and has two solutions. This lack of uniqueness is also related to the notions of \emph{nonconjugate} solutions \cite[Def.~7.4.4]{marsden2013introduction} in classical mechanics and \emph{conjugate} points \cite[p.~198]{marsden2013introduction}  in the study of geodesics.
This non-uniqueness is also illustrated in
\cite{micheli2008differential} and  connected with the fact that the corresponding
Riemannian metric has negative and positive curvature. Although minimizers are not unique, they are unique almost everywhere by Ekeland's variational principle in Sobolev spaces \cite{bruveris2017completeness}.
}.

\begin{Theorem}\label{thmksbsahshkd}
The minimum values of \eqref{eqlsweedhsejd}, \eqref{eqkjlkjkejwkdj} and \eqref{eqlkjgehgddjedhjd} are identical.
\eqref{eqlsweedhsejd}, \eqref{eqkjlkjkejwkdj} and \eqref{eqlkjgehgddjedhjd} have minimizers.
$q$ is a minimizer of \eqref{eqlsweedhsejd} if and only if $(q,p)$ ($p=\wK(q,q)^{-1}\dot{q}$) follows the Hamiltonian dynamic  \eqref{eqkedmdledkemdl} (with $q(0)=X$) and $p(0)=\wK\big(q(0),q(0)\big)^{-1} \dot{q}(0)$ is a minimizer of $\Vk\big(p(0),X,Y\big)=$\eqref{eqkjlkjkejwkdj}.
$v$ is a minimizer of \eqref{eqlkjgehgddjedhjd} if and only if $v(x,t)= \wK\big(x,q(t)\big) p(t)$ with $(q,p)$ following the Hamiltonian dynamic  \eqref{eqkedmdledkemdl} (with $q(0)=X$) and $p(0)$ being a minimizer of $\Vk\big(p(0),X,Y\big)=$\eqref{eqkjlkjkejwkdj}.
Therefore the minimizers of \eqref{eqlsweedhsejd} and \eqref{eqlkjgehgddjedhjd} can be parameterized by their initial momentum identified as a minimizer of $\Vk\big(p(0),X,Y\big)=$\eqref{eqkjlkjkejwkdj}.
 Furthermore, at those minima, the energies $\frac{1}{2}\dot{q}^T \wK(q, q)^{-1} \dot{q}$ and
$ \frac{1}{2} \|v\|_{\Gamma}^2$ are constant over $t\in [0,1]$ and equal to $\frac{\nu}{2}p^T(0) \wK\big(X,X\big) p(0)$.
\end{Theorem}
\begin{proof}
Given Thm.~\ref{thmlllakbhhd} and Thm.~\ref{thmkljedjnedesdjd} we only need to prove the existence of minimizers for $\Vk\big(p(0),X,Y\big)=$\eqref{eqkjlkjkejwkdj}.
Let $B_\rho:=\{p(0)\in \X^N\mid p(0)^T p(0) \leq \rho^2\}$. Since $\X$ is finite-dimensional this ball is compact. Since $\ell$ is continuous in $q(1)$ and $q(1)$ is continuous in $p(0)$ \cite[Thm.~1.4.1.]{arino2006fundamental}, \eqref{eqkjlkjkejwkdj} must have a minimizer in $B_\rho$.
Since $\ell$ is positive, Cond.~\ref{condeqlkedjehd7d}.(1) implies that \eqref{eqkjlkjkejwkdj} diverges towards infinity as $p(0)^T p(0)\rightarrow \infty$. It follows that, for $\rho$ large enough, $B_\rho$ contains at least one
global minimizer  of \eqref{eqkjlkjkejwkdj} and all global minimizers are contained in $B_\rho$. Note that by Thm.~\ref{thml2kjej2ke} and the representation \eqref{eqtttc} it holds true that a minimizer  $v$ of \eqref{eqlkjgehgddjedhjd} must be an element of  $C([0,1],\H_\Gamma)$.
\end{proof}

We will now show the existence of minimizers for  \eqref{eqlkjeddjedhjd} and \eqref{eqlsedhjd}.
\cite[Thm.~3.3]{west2004variational} implies\footnote{Such results are part of the discrete mechanics literature on discretized least action principles. General accuracy results could also be derived from \cite[Sec.~2]{marsden2001discrete} and \cite[p.~114]{blanes2017concise} and $\Gamma$-convergence results could be derived from \cite{muller2004gamma}.} (under the regularity conditions \ref{condeqlkedjehd7d} on $\wK$) that the trajectory $q^1,\ldots,q^{L+1}$ of a minimizer of the discrete least action principle \eqref{eqlsedhjd} follows  a first-order\footnote{Higher order  symplectic partitioned Runge Kutta discretizations \cite[Sec.~2.6.5]{hairer2006geometric} of  \eqref{eqkedmdledkemdl} would lead to numerical schemes
akin to Densely Connected Networks  \cite{huang2017densely}.} symplectic integrator for the Hamiltonian system \eqref{eqkedmdledkemdl}. Introducing the momentum variables \eqref{eqkjelkbejdhbd}
  this discrete integrator is \eqref{ejkhdbejhdbd}.
Write
\begin{equation}\label{eqkjhwbkdwehjvbdy}
\mathfrak{V}_L(p^1,X,Y):=\begin{cases}&\frac{\nu}{2}  \sum_{s=1}^L  (p^s)^T \wK(q^s,q^s) p^s\, \Delta t+
\ell\big(q^{L+1},Y\big)\\
&p^s=\eqref{eqkjelkbejdhbd}\text{ and }(q^s,p^s)\text{ follow }\eqref{ejkhdbejhdbd} \text{ with } q^1=X\,.
\end{cases}
\end{equation}

\begin{Theorem}\label{thmksbsahskd2g}
The minimum values of  \eqref{eqlkjeddjedhjd}, \eqref{eqlsedhjd} and $\mathfrak{V}_L(p^1,X,Y)$ (in $p^1$) are identical.
\eqref{eqlkjeddjedhjd}, \eqref{eqlsedhjd} and \eqref{eqkjhwbkdwehjvbdy} have minimizers.
$q^1,\ldots,q^{L+1}$ is a minimizer of \eqref{eqlsedhjd} if and only if $(q^s,p^s)$  (with $p^s$=\eqref{eqkjelkbejdhbd}) follows the discrete Hamiltonian map \eqref{ejkhdbejhdbd}, $q^1=X$ and $p^1$ is a minimizer of $\mathfrak{V}_L(p^1,X,Y)=$\eqref{eqkjhwbkdwehjvbdy}.
$v_1,\ldots,v_L$ is a minimizer of \eqref{eqlkjeddjedhjd} if and only if $v_s(x)=\Delta t\,\wK( x, q^s)p^s=$\eqref{eqkjdkedkjnd} where $(q^s,p^s)$
 follows the discrete Hamiltonian map \eqref{ejkhdbejhdbd} with $q^1=X$ and $p^1$ is a minimizer of $\mathfrak{V}_L(p^1,X,Y)=$\eqref{eqkjhwbkdwehjvbdy}.
 Therefore the minimizers of \eqref{eqlkjeddjedhjd} and \eqref{eqlsedhjd} can be parameterized by their initial momentum identified as a minimizer of $\mathfrak{V}_L(p^1,X,Y)=$\eqref{eqkjhwbkdwehjvbdy}.
 At those minima, the energies  $\frac{1}{2}(p^s)^T \wK(q^s,q^s) p^s$ and $\frac{1}{2}\|v_s\|_{\Gamma}^2$  are equal and fluctuate by at most $\mathcal{O}(1/L)$ over $s\in \{1,\ldots,L\}$.
\end{Theorem}
\begin{proof}
By Thm.~\ref{thmeqkljedjnedjd} we only need to prove the result for \eqref{eqkjhwbkdwehjvbdy}.
Since (under Cond.~\ref{condeqlkedjehd7d}) $\mathfrak{V}_L(p^1,X,Y)$ diverges  towards infinity as $(p^1)^T p^1\rightarrow \infty$ and since $\mathfrak{V}_L(p^1,X,Y)$ is continuous, as in proof of Thm.~\ref{thmksbsahshkd}, for $\rho$ large enough,
\eqref{eqkjhwbkdwehjvbdy} must have a global minimizer in $B_\rho:=\{p^1\in \X^N\mid (p^1)^T p^1 \leq \rho^2\}$ and
 all global minimizers must be contained in $B_\rho$.
By Thm.~\ref{thmeqkljedjnedjd} $\|v_s\|_{\Gamma}^2$ is equal to $(p^s)^T \wK(q^s,q^s) p^s$ where $(q^s,p^s)$ is obtained from the (first-order) symplectic and variational integrator \eqref{ejkhdbejhdbd} for the Hamiltonian system \eqref{eqkedmdledkemdl}. The near energy preservation then follows  from \cite[Thm.~8.1]{hairer2006geometric} (derived from  the fact that symplectic integrators simulate a nearby mechanical system) and the order of accuracy of \eqref{ejkhdbejhdbd}.
\end{proof}

\subsection{Convergence of minimal values and minimizers}\label{subskjkecgyf67f6} Due to the possible lack of uniqueness of warping regression solutions (discussed in Subsec.~\ref{subseckjjhgdejwedh}), the convergence of minimizers must be indexed based on their initial momentum parametrization as described in Thm.~\ref{thmksbsahshkd} and Thm.~\ref{thmksbsahskd2g}.
Write $\Mk_L(X,Y)$ for the set of minimizers $p^1$ of $\Vk_L(p^1,X,Y)=$\eqref{eqkjhwbkdwehjvbdy}.
Write  $\Mk(X,Y)$ for the set of minimizers $p(0)$ of   $\Vk\big(p(0),X,Y\big)=$\eqref{eqkjlkjkejwkdj}.

\begin{Theorem}\label{thmhgw7gfdd}
The  minimal value of \eqref{eqlkjeddjedhjd}, \eqref{eqlsedhjd} and \eqref{eqkjhwbkdwehjvbdy} converge, as $L\rightarrow \infty$,
towards the minimal value of  \eqref{eqlsweedhsejd}, \eqref{eqkjlkjkejwkdj}, \eqref{eqlkjgehgddjedhjd}.
As $L\rightarrow \infty$, the set of adherence values\footnote{Writing $\cl\, A$ for the closure of a set $A$, $\cap_{L'\geq 1} \cl \cup_{L\geq L'} \Mk_L(X,Y)=\Mk(X,Y)$.} of $\Mk_L(X,Y)$ is  $\Mk(X,Y)$.
Let $v_s^L$, $q^s_L$ and $p^1_L$ be sequences of minimizers of \eqref{eqlkjeddjedhjd}, \eqref{eqlkjeddjedhjd} and \eqref{eqkjhwbkdwehjvbdy} indexed by the same sequence $p^1_L$ of initial momentum in $\Mk_L(X,Y)$ (as described in Thm.~\ref{thmksbsahskd2g}). Then, the adherence points of the sequence $p^1_L$ are in $\Mk(X,Y)$ and if $p(0)$ is such a point
($p^1_L$ converges towards $p(0)$ along a subsequence $L_k$) then, along that subsequence: (1) The trajectory formed by interpolating the states $q^s_L \in \X^N$ converges to the trajectory formed by a minimizer of \eqref{eqlsweedhsejd} with initial momentum $p(0)$.
(2) For\footnote{Write $\textrm{int}(tL)$ for the integer part of $tL$.} $t\in [0,1]$,  $(I+v_{\textrm{int}(tL)}^L)\circ \cdots \circ (I+v_1^L)(x)$ converges to $\phi^v(x,t)$=\eqref{eqflmp} where $v$ is a minimizer of \eqref{eqlkjgehgddjedhjd} with initial momentum $p(0)$.
Conversely if $p(0)\in \Mk(X,Y)$ then it is the limit of a sequence $p^{1}_L \in \Mk_L(X,Y)$ and the minimizers of \eqref{eqlkjeddjedhjd}, \eqref{eqlsedhjd} and \eqref{eqkjhwbkdwehjvbdy} with initial momentum $p^1_L$ converge (in the sense given above) to the minimizers of \eqref{eqlsweedhsejd}, \eqref{eqkjlkjkejwkdj}, \eqref{eqlkjgehgddjedhjd} with initial momentum $p(0)$ (as described in Thm.~\ref{thmksbsahshkd}).
\end{Theorem}
\begin{proof}
By Thm.~\ref{thmksbsahskd2g} and Thm.~\ref{thmksbsahshkd} the convergence of minimum values follows from that of $\mathfrak{V}_L(p^1,X,Y)$  towards that of $\Vk\big(p(0),X,Y\big)$.
As shown in the proof of Thm.\ref{thmksbsahshkd} and Thm.~\ref{thmksbsahskd2g}, the initial momenta minimizing
$\mathfrak{V}_L$ and $\Vk$ are contained in a compact set $B_\rho$ (independent from $L$).
The uniform convergence (for $p^1=p(0)\in B_\rho$, $q^1=q(0)=X$) of the solution of the  integrator \eqref{ejkhdbejhdbd} towards the solution of the Hamiltonian system \eqref{eqkedmdledkemdl} implies that
$\lim_{L\rightarrow \infty} \min_{p^1\in B_\rho} \mathfrak{V}_L(p^1,X,Y)= \min_{p^1\in B_\rho} \lim_{L\rightarrow \infty} \mathfrak{V}_L(p^1,X,Y)=\min_{p^1\in B_\rho} \Vk\big(p^1,X,Y\big)$. Which proves the convergence of minimum values.
Similarly the uniform convergence (over $p^1\in B_\rho$) of $\mathfrak{V}_L(p^1,X,Y)$ towards $\Vk\big(p^1,X,Y\big)$ (also obtained from the uniform convergence of the solution of \eqref{ejkhdbejhdbd} in $B_\rho$) implies that
the set of adherence values of $\Mk_L(X,Y)$ is  $\Mk(X,Y)$.
Let $p^1_L$ be a sequence in $\Mk_L(X,Y)$  and let  $p(0)\in \Mk(X,Y)$ be one of its adherence points.
The convergence of $p^1_L$ towards $p(0)$ (along a subsequence $L_k$) and the uniform convergence (along $L_k$) of
the solution of \eqref{ejkhdbejhdbd} towards the solution of  \eqref{eqkedmdledkemdl} implies (1).
(1) and the representation $v_s=\wK( x, q^s)p^s=$\eqref{eqkjdkedkjnd} of Thm.~\ref{thmksbsahskd2g} imply
that $(I+v_{\textrm{int}(tL)})\circ \cdots \circ (I+v_1)(x)$ converges to $z(t)$ where $z$ is the
  solution of the ODE
\begin{equation}\label{eqkjhgkjgguyugy}
\dot{z}= \wK(z,q) p \text{ with } z(0)=x\,,
\end{equation}
where $(q,p)$  follows the Hamiltonian dynamic  \eqref{eqkedmdledkemdl} with initial value $q(0)=X$ and $p(0)$.
We conclude that (2) holds true by the representation $v(x,t)= \wK(x,q) p$ of Thm.~\ref{thmksbsahshkd} and by observing that
 \eqref{eqkjhgkjgguyugy} is a characteristic curve  of  \eqref{eqttjjkjtc}.
The proof of the remaining (conversely) portion of the theorem is identical.
\end{proof}

Thm.~\ref{thmhgw7gfdd} implies that minimizers of \eqref{eqlktddsyeytdsedhjd} converge towards minimizers of \eqref{eqlkjgehgddjedhjdB} in the sense Cor.~\ref{corwjkdkdb736d}.

\section{With feature maps and activation functions}\label{secfm}
This section presents the feature map approach to warping regression and idea registration. Our analysis includes feature maps defined via activation functions (Sec.~\ref{subsecekuwege62}). This generalization allows us to establish a direct connection with ANN-based methods, which can be analyzed as kernel-based methods characterized by a finite-dimensional feature-map parameterized by the inner layers of the network.
If the dimension of the feature space is finite then,  the feature-map approach is more efficient that the feature-space approach in terms of complexity (at the cost of some loss in expressivity). In particular, the memory required to store the trained network is independent from the number $N$ of data points.

\subsection{With feature maps}\label{subseckejdhkejd}

\subsubsection{Warping regression}\label{subsejhjhgy}
Let $\F$ and  $\psi\,:\, \X \rightarrow \L(\X,\F)$ be a feature space and map associated with the kernel $\Gamma$ of the RKHS ${\H_\Gamma}$ in \eqref{eqlkjeddjedhjd} and \eqref{eqlkjgehgddjedhjd}. We will now work under Cond.~\ref{condnuggetreg}.
 The following theorems reformulate the least action principles and Hamiltonian dynamics of Sec.~\ref{secregular}  in the feature map setting of Subsec.~\ref{subsecfeaturemaps}.

\begin{Theorem}
 $\alpha_1,\ldots,\alpha_L,q^1,\ldots,q^{L+1}$ minimize
\begin{equation}\label{eqlkjeddjedhjdBN}
\begin{cases}
\text{Minimize } &\frac{\nu}{2} L\sum_{s=1}^L \big( \|\alpha_s\|_{\F}^2+ \frac{1}{r} \|q^{s+1}-q^s -
\psi^T(q^s) \alpha_s\|_{\X^N}^2\big)+\ell\big(q^{L+1},Y\big)\\
\text{over }&\alpha_1,\ldots,\alpha_L \in \F,\,  q^2,\ldots,q^{L+1} \in \X^N,\, q^1=X\,,
\end{cases}
\end{equation}
if and only if the $v_s=\psi^T \alpha_s$, $q^s$  minimize \eqref{eqlktddsyeytdsefsreg}. Furthermore $\|\alpha_s\|_{\F}^2+ \frac{1}{r} \|q^{s+1}-q^s -
\psi^T(q^s) \alpha_s\|_{\X^N}^2$ fluctuates by at most $\mathcal{O}(1/L)$ over $s$.
\end{Theorem}
\begin{proof}
The proof is a simple consequence of Thm.~\ref{thmkjhkejhddjhd} and Thm.~\ref{thmksbsahskd2greg}.
\end{proof}

\begin{Theorem}\label{thmkjehddjhdjeh}
$\alpha \in C([0,1],\F)$ and $q\in C^1([0,1],\X^N)$ minimize
\begin{equation}\label{eqlkjgehgddjedhjdBN}
\begin{cases}
\text{Minimize } &\frac{\nu}{2}\int_0^1 \big(\|\alpha(t)\|_{\F}^2+ \frac{1}{r} \|\dot{q}(t) -
\psi^T(q(t)) \alpha(t)\|_{\X^N}^2\big)\,dt+\ell\big(q(1),\phi^v(X,1)\big)\\
\text{over }& \alpha \in C([0,1],\F),\, q\in C^1([0,1],\X^N),\,q(0)=X\,,
\end{cases}
\end{equation}
if and only if $v(\cdot,t)=\psi^T(\cdot) \alpha(t)$ and $q(t)$ minimize \eqref{eqlkjgehgddjedhjdreg}. Furthermore, at the minimum,
$\|\alpha(t)\|_{\F}^2+ \frac{1}{r} \|\dot{q}(t) -
\psi^T(q(t)) \alpha(t)\|_{\X^N}^2$ is constant over $t\in [0,1]$.
\end{Theorem}
\begin{proof}
The proof is a simple consequence of Thm.~\ref{thmkjhkejhddjhd} and Thm.~\ref{thmksbsahshkdreg}.
\end{proof}

Let  $\F_2$ and  $\psi_2\,:\, \X \rightarrow \L(\Y,\F_2)$
be a feature map and space associated with the kernel $K$ in the loss \eqref{eqlktdelldreg}. The following is a direct corollary of Thm.~\ref{thmkjehddjhdjeh}.
\begin{Corollary}\label{corjhgjgkyu}
The  maps  $v=\psi^T \alpha$ and $f=\psi_2^T \alpha_2$ obtained from
minimizing
\begin{equation}\label{eqlkjgedehgddjedhjssdBN}
\begin{cases}
\text{Minimize } &\frac{\nu}{2}\,\int_0^1 \big(\|\alpha(t)\|_{\F}^2+ \frac{1}{r}\|\dot{q}(t)-\psi^T(q(t)) \alpha(t)\|_{\X^N}^2\big)\,dt\\&+
\lambda\,\big(\|\alpha_2\|_{\F_2}^2+\frac{1}{\rho} \|\psi_2^T(q(1))\alpha_2-Y'\|_{\Y^N}^2 \big)+\ell_\Y\big(Y',Y\big)\\
\text{over }& \alpha \in C([0,1],\F),\, \alpha_2 \in \F_2,\,q\in C^1([0,1],\X^N),\,q(0)=X,\, Y'\in \Y^N\, .
\end{cases}
\end{equation}
 are identical to those obtained by minimizing \eqref{eqlkjgjhdbjehreg}.
\end{Corollary}

\subsection{With feature maps of scalar operator-valued kernels}
We will now describe warping regression with the feature maps of scalar operator-valued kernels.

\subsubsection{Feature maps of scalar operator-valued kernels}\label{secscalfeatmap}
We will first describe the feature spaces and maps of scalar operator-valued kernels in the setting of Sec~\ref{secovk}. Let $K(x,x')=k(x,x') I_\Y$ be as in Definition \ref{deflkekdlkjd} and write  $\Fk$ and $\varphi\,:\,\X\rightarrow \Fk$ for a feature space and map associated with the scalar-valued kernel $k$.
Write $\L(\Fk,\Y)$ for the space of bounded linear operators from  $\Fk$ to $\Y$.
For $\beta \in \Fk$ and $y\in \Y$ write $  y\beta^T \in \L(\Fk,\Y)$  for the outer product between $y$ and $\beta$  defined as the linear function mapping $\beta'\in \Fk$ to $y \<\beta,\beta'\>_\Fk \in \Y$.

\begin{Theorem}
A feature space of the operator-valued kernel $K$ is  $\F:=\L(\Fk,\Y)$ and its feature map is defined by
\begin{equation}\label{eqlkjehdlkjeh}
\psi(x) y=y \varphi^T(x) \text{ for }(x,y)\in \X\times \Y\,.
\end{equation}
Furthermore,
\begin{equation}
\psi^T(x)\alpha= \alpha \varphi(x) \text{ for }x\in \X \text{ and }\alpha \in \L(\Fk,\Y)\,,
\end{equation}
and,
\begin{equation}\label{knkejdnkjd}
\|\alpha \varphi(\cdot)\|_K^2=\|\alpha\|^2_\F=\|\alpha\|^2_{\L(\Fk,\Y)}=\Tr[\alpha^T \alpha]\,.
\end{equation}
where $\Tr$ is the trace operator.
\end{Theorem}
\begin{proof}
Write  $\psi$ and $\F$ for a feature map/space associated with $K$.
The identity
\begin{equation}
y^T K(x,x') y'=\< \psi(x) y, \psi(x') y'\>_\F=  \<\varphi(x),\varphi(x')\>_\Fk \<y,y'\>_\Y\,,
\end{equation}
implies \eqref{eqlkjehdlkjeh} and the identification of $\F$ with $\L(\Fk,\Y)$ endowed with inner product
\begin{equation}\label{eqlklkjwndd}
\< \sum_{i,j} c_{i,j} y_i \beta_j^T  , \sum_{i',j'} c_{i',j'}' y_{i'} (\beta_{j'}')^T  \>_\F=  \sum_{i,i',j,j'} c_{i,j}c_{i',j'}' \<\beta_j,\beta_{j'}'\>_\Fk \<y_i,y_{i'}'\>_\Y\,.
\end{equation}
 Thm.~\ref{thmkjhkejhddjhd} implies the first identity in  \eqref{knkejdnkjd}.
\eqref{eqlklkjwndd} combined with the matrix representation of $\alpha\in \L(\Fk,\Y)$ over bases of $\Y$ and $\Fk$ imply the last equality in \eqref{knkejdnkjd}.
\end{proof}

\subsubsection{Warping regression}

Consider the setting of \eqref{eqlkjgedehgddjedhjssdBN} in the situation where $\Gamma$ and $K$  are scalar, i.e. $\Gamma(z,z')=k(z,z') I_\X$ and $K(z,z')=k_2(z,z') I_\Y$ and write
$\Fk$, $\Fk_2$, $\varphi$ and $\varphi_2$ for  feature spaces and maps associated with $k$ and $k_2$. The following proposition follows from Subsec.~\ref{secscalfeatmap}.

\begin{Proposition}
The maps $v(\cdot,t)=\alpha(t)\varphi(\cdot) $ and $f=\alpha_2\varphi_2$ obtained by minimizing
\begin{equation}\label{eqlkjwwqehgddjedhjssdBN}
\begin{cases}
\text{Minimize } &\frac{\nu}{2}\,\int_0^1 \big(\|\alpha(t)\|_{\L(\Fk,\X)}^2+ \frac{1}{r}\|\dot{q}(t)-\alpha(t)\varphi(q(t)) \|_{\X^N}^2\big)\,dt\\&+
\lambda\,\big(\|\alpha_2\|_{\F_2}^2+\frac{1}{\rho} \|\alpha_2\varphi_2(q(1))-Y'\|_{\Y^N}^2 \big)+\ell_\Y\big(Y',Y\big)\\
\text{over }& \alpha \in C([0,1],\F),\, \alpha_2 \in \F_2,\,q\in C^1([0,1],\X^N),\,q(0)=X,\, Y'\in \Y^N\, .
\end{cases}
\end{equation}
are identical to those obtained by minimizing \eqref{eqlkjgedehgddjedhjssdBN}.
\end{Proposition}

 \begin{figure}[h!]
	\begin{center}
			\includegraphics[width= \textwidth]{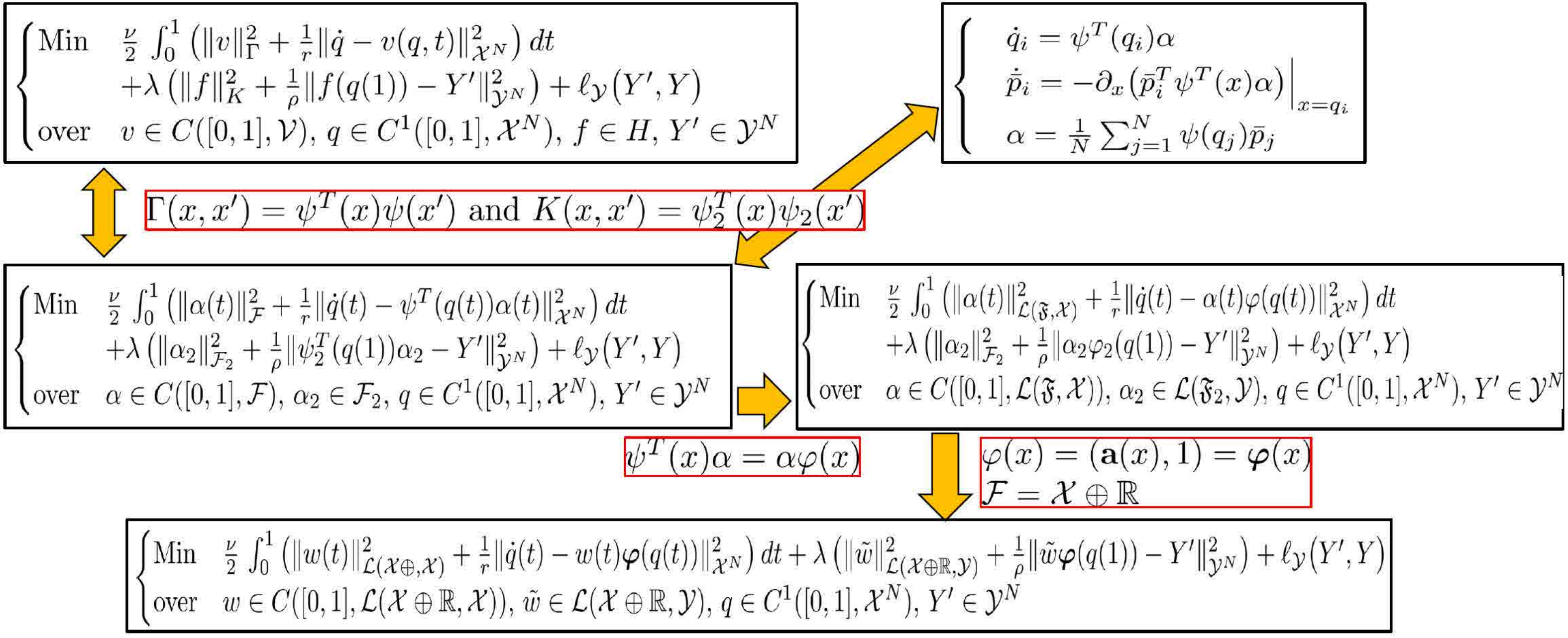}
		\caption{Warping regression with feature maps and activation functions.}\label{figmr2}
	\end{center}
\end{figure}

\subsection{With activation functions}\label{subsecekuwege62}
Let $\Fk=\Fk' \oplus \Fk''$ where $\Fk'$ and $\Fk''$ are $\<\cdot,\cdot\>_\Fk$-orthogonal separable Hilbert sub-spaces of $\Fk$.
Let $A\in \L(\X,\Fk')$ be a bounded linear operator from $\X$ to $\Fk'$ such that\footnote{ $\operatorname{dim}(\Fk')\geq \operatorname{dim}(\X)$ suffices for the existence of such an $A$.} $A^T A=I$, $c\in \Fk''$ such that $c^T c=1$, and  $\varphi\,:\, \X\rightarrow \Fk' \oplus \Fk''$ defined by
\begin{equation}\label{eqkjdekjddkjsdor}
\varphi(x)= A \ba (x) +c\,,
\end{equation}
where $\ba\,:\,\X\rightarrow \X$, is an arbitrary nonlinear {\bf activation function}.
\begin{Proposition}\label{propactfunct}
The operator-valued kernel defined by \eqref{eqkjdekjddkjsdor} is $\Gamma(x,x') = (\ba(x)^T \ba(x')+1) I$.  In particular $\Gamma$ satisfies Cond.~\ref{condnuggetreg} if $x\rightarrow \ba(x)$ and its first and second order partial derivatives are continuous and uniformly bounded.
\end{Proposition}

\begin{Remark}\label{rmkkjhevgcvcg78}
We call $\ba$ an {\bf elementwise nonlinearity} if $\ba(x)=\sum_{i=1}^{d_\X} e_i \bar{\ba}(x_i)$  for $x=\sum_{i=1}^{d_\X} x_i e_i \in \X$ where $e_1,\ldots,e_{d_\X}$ is some basis of $\X$ and $\bar{\ba}$ is a scalar-valued (nonlinear) function. To satisfy the
regularity requirements of Prop.~\ref{propactfunct} we must then assume that $z\rightarrow \bar{\ba}(z), \partial_z \bar{\ba}(z), \partial_z^2 \bar{\ba}(z)$  are continuous and uniformly bounded.
 Observe that the sigmoid nonlinearity $\bar{\ba}(z)=\tanh(z)$ satisfies these regularity conditions.
  Although
 ReLU ($\bar{\ba}(z)=\max(z,0)$) is not bounded and lacks the required regularity one could use a bounded and smoothed version such as the following variant of  softplus $\bar{\ba}(z)=  \ln(1 + e^z)/\big(1+ \epsilon \ln(1 + e^z) \big) $ which behaves like ReLU for $z\in (-\infty,1/\epsilon)$ and $0<\epsilon\ll 1 $. For ease of presentation, we will also write $\ba$ for $\bar{\ba}$ when $\ba$ is an elementwise nonlinearity.
\end{Remark}

Prop.~\ref{propactfunct} implies that, as long as $A$ and $c$ are unitary, their particular choice has no influence on the kernel $\Gamma$. We will therefore from now on, in the setting of activation functions, select $\F=\X\oplus \R$ ($\Fk'=\X$ and $\Fk''=\R$) and use the identity matrix/vector for $A$ and $c$. \eqref{eqkjdekjddkjsdor} can then be written
\begin{equation}\label{eqkjdekjddkjsd}
\varphi(x)= \bvarphi (x) \text{ with } \bvarphi(x)=(\ba(x),1)\,,
\end{equation}
and $\ba$ is, from now on, assumed  to satisfy the regularity conditions of Prop.~\ref{propactfunct}.
Similarly we select $\Fk_2=\X \oplus \R$ and
\begin{equation}\label{eqkjdekjddkjsd2}
\varphi_2(x)=  \bvarphi (x) \text{ with } \bvarphi(x)=(\ba(x),1)\,.
\end{equation}
Note that the operator-valued kernel $K\,:\, \X\times \X\rightarrow \L(\Y)$ defined by \eqref{eqkjdekjddkjsd2} is
\begin{equation}\label{eqkj2lehdbljebd}
K(x,x')=\bvarphi^T(x) \bvarphi(x')=(\ba^T(x) \ba(x')+1) I_\Y\,.
\end{equation}
Also note that for $\tilde{w}\in \L(\X \oplus \R,\Y)$ we have $\tilde{w} \bvarphi(x)= W \ba(x)+ b$ where the {\bf weight} $W\in \L(\X,\Y)$ is defined by
 $W z=\tilde{w} (z,0) $ for $z\in \X$ and the {\bf bias} $b\in \Y$ is defined by  $b=\tilde{w}(0,1)$.
 Therefore \eqref{eqkjdekjddkjsd2} allows us to incorporate weights and biases into a single variable $\tilde{w}$.

 Write $\|\cdot\|_{\L(\X,\Y)}$ for the Frobenius norm on $\L(\X,\Y)$. The following theorem shows that warping regression with feature maps  \eqref{eqkjdekjddkjsd} and \eqref{eqkjdekjddkjsd2} can be expressed as a ResNet with scaled/strong $L_2$ regularization on weights and biases.
The following two theorems are straightforward, and Fig.~\ref{figmr2} summarises the results of this section.

\begin{Theorem}\label{thmidjheyd88h2dd}
If $\varphi$ and $\varphi_2$ are as in  \eqref{eqkjdekjddkjsd} and \eqref{eqkjdekjddkjsd2} then
    $v(\cdot,t)=w(t) \bvarphi(\cdot)$, $f=\tilde{w} \bvarphi(\cdot)$ and $q,Y'$  obtained by minimizing
\begin{equation}\label{eqlkjgedehgddjedddsdddsshjssdBN}
\begin{cases}
\text{Min} &\frac{\nu}{2}\,\int_0^1 \big(\|w(t)\|_{\L(\X\oplus \R,\X)}^2+ \frac{1}{r}\|\dot{q}(t)-
w(t) \bvarphi(q(t)) \|_{\X^N}^2\big)\,dt+\\&
\lambda\,\big(\|\tilde{w}\|_{\L(\X\oplus \R,\Y)}^2+\frac{1}{\rho} \|\tilde{w}\bvarphi(q(1))-Y'\|_{\Y^N}^2 \big)+\ell_\Y\big(Y',Y\big)\\
\text{over}& w\in C\big([0,1],\L(\X\oplus\R,\X)\big),\, \tilde{w}\in \L(\X\oplus\R,\Y),\\&
q\in C^1([0,1],\X^N),\,q(0)=X,\, Y'\in \Y^N\, ,
\end{cases}
\end{equation}
are identical to those obtained by
 minimizing \eqref{eqlkjwwqehgddjedhjssdBN}. Therefore \eqref{eqlkjgedehgddjedddsdddsshjssdBN} has minimizers and
   if $w, q$ are minimizers of \eqref{eqlkjgedehgddjedddsdddsshjssdBN} then the energy $\frac{1}{2}\big(\|w(t)\|_{\L(\X\oplus\R,\X)}^2+ \frac{1}{r}\|\dot{q}(t)-
w(t) \bvarphi(q(t)) \|_{\X^N}^2\big)$ is constant over $t\in [0,1]$.
\end{Theorem}
\begin{proof}
The proof is straightforward. Use Thm.~\ref{thmksbsahshkdreg} for the existence of minimizers and energy preservation.
\end{proof}

\begin{Theorem}\label{thmidjheyd88h2ddde}
If $\varphi$ and $\varphi_2$ are as in  \eqref{eqkjdekjddkjsd} and \eqref{eqkjdekjddkjsd2} then
  $\phi=(I+v_L)\circ \cdots \circ(I+v_1)$, $v_s=w^s \bvarphi(\cdot)$, $f=\tilde{w} \bvarphi(\cdot)$ and $q^s, Y'$, obtained by minimizing
\begin{equation}\label{eqlkjgedehgddjededdeddddsdddsshjssdBN2}
\begin{cases}
\text{Min} &\frac{\nu L}{2} \sum_{s=1}^L \big(\|w^s\|_{\L(\X\oplus \R,\X)}^2
+\frac{1}{r} \|q^{s+1}-q^s -
w^s \bvarphi(q^s)\|_{\X^N}^2
\big)
+\\&
\lambda\,\big(\|\tilde{w}\|_{\L(\X\oplus \R,\Y)}^2+\frac{1}{\rho} \|\tilde{w}\bvarphi(q^{L+1})-Y'\|_{\Y^N}^2 \big)+\ell_\Y\big(Y',Y\big)\\
\text{over}& w^s\in\L(\X\oplus \R,\X),\, \tilde{w}\in \L(\X\oplus \R,\Y),\,q^s\in \X^N,\, Y'\in \Y^N\, ,q^1=X\,,
\end{cases}
\end{equation}
are identical to those obtained by
 minimizing \eqref{eqlkjeddjedhjdBN} and, as $L\rightarrow \infty$, converge (in the sense of the adherence values as in Subsec.~\ref{subskjkecgyf67f6reg}), towards those obtained by minimizing \eqref{eqlkjgehgddjedhjd}. Furthermore, \eqref{eqlkjgedehgddjededdeddddsdddsshjssdBN2} has minimizers and if
  the $w^s, q^s$ are minimizers of \eqref{eqlkjgedehgddjededdeddddsdddsshjssdBN2}  then the energy $\frac{1}{2}\big(\|w^s\|_{\L(\X\oplus \R,\X)}^2
+\frac{1}{r} \|q^{s+1}-q^s -
w^s \bvarphi(q^s)\|_{\X^N}^2\big)$ fluctuates by at most $\mathcal{O}(1/L)$ over $s\in \{1,\ldots,L\}$.
\end{Theorem}
\begin{proof}
The proof of the is straightforward. Convergence follows from Thm.~\ref{thmhgw7gfddreg}.
Near energy preservation and existence follow from Thm.~\ref{thmksbsahskd2greg}.
\end{proof}

 \begin{figure}[h!]
	\begin{center}
			\includegraphics[width= \textwidth]{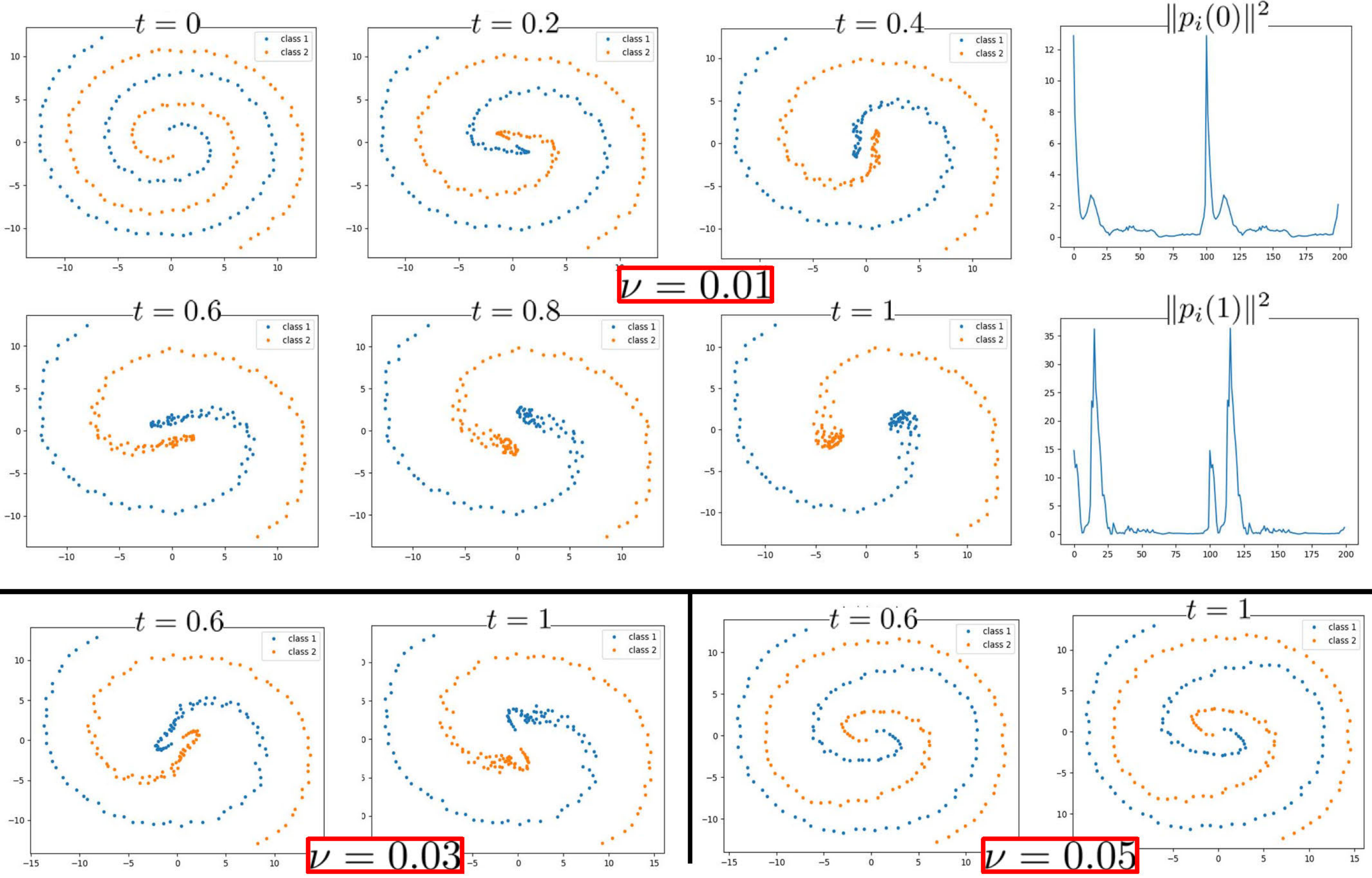}
		\caption{Swiss roll data set. Locations of the points $q_i(t)$ for $\nu=0.01$ (top) $\nu=0.03$ (bottom left) and $\nu=0.05$ (bottom right).
The sparsity of the momentum variable is theoretically explained in
Sec.~\ref{secspa}.}\label{figspiral2}
	\end{center}
\end{figure}

\section{Numerical experiments}\label{secnumexp}

This section presents numerical experiments on training warping regression networks with geodesic shooting (Sec.~\ref{seckdejhdjkdw}) and with feature maps (Sec.~\ref{secnumfm}). These experiments illustrate (1) the sparsity of the momentum representation of minimizers (2) testing error improvements due to the learned warping.

\subsection{Numerical implementation as geodesic shooting}\label{seckdejhdjkdw}
Given the picture depicted in Fig.~\ref{figmr}, the geodesic shooting solution to Problem \ref{pb828827hee}
is summarized in the pseudo-algorithm  \ref{alggtpshootsol}.
Except for the structure of the end loss $\ell$,
this method is similar to the one introduced for computational anatomy \cite{miller2006geodesic, allassonniere2005geodesic} (where the constraint $\dot{q}=\wK(q,q)p$ may also be relaxed \cite{camion2001geodesic}).
This section will implement this strategy on the Swiss roll dataset to illustrate the impact of the value of $\nu$  on the deformation of the space and the sparsity of momentum variables.

\begin{algorithm}[h]
\caption{Shooting solution to Problem \ref{pb828827hee}}\label{alggtpshootsol}
\begin{algorithmic}[1]
\STATE Define the loss $\ell$  via \eqref{eqhgvygvgyv} or \eqref{eqledhehdiudh}.
\STATE Discretize the Hamiltonian system \eqref{eqkedmdledkemdl} with a stable and accurate integrator.
\STATE Minimize \eqref{eqkjlkjkejwkdj} (via gradient descent and the discretized Hamiltonian system) to identify the initial momentum $p(0)$.
\STATE Approximate $f^\dagger$ with
$f^\ddagger(\cdot)=f\circ \phi(\cdot,1)$ where $\phi$ is obtained from the numerical integration of \eqref{eqttjjkjtc} (using the solution of the discretized Hamiltonian system with optimal initial momentum $p(0)$) and $f$ is obtained as
the minimizer of $\ell\big(\phi(X,1),Y\big)$ in  \eqref{eqhgvygvgyv} or \eqref{eqledhehdiudh}.
\end{algorithmic}
\end{algorithm}

\subsubsection{Geometric integration}\label{subsecgyf67hhbvdf6}

The Hamiltonian system \eqref{eqkedmdledkemdl} is characterized by structural and geometric invariants (the canonical symplectic form, volumes in the phase space, the energy, etc.)  \cite{marsden2013introduction}. Symplectic integrators \cite{hairer2006geometric} have been developed to approximate the continuous system while exactly (e.g., for the symplectic form) or nearly (e.g., for the energy) preserving these invariants. The main idea of these integrators is to simulate a nearby discrete mechanical system rather than a nearby discrete ODE. Within this class of symplectic integrators, explicit ones are preferred for their computational efficiency/tractability.
Since the Hamiltonian \eqref{eqkedmdledkemdl} is nonseparable, classical symplectic integrators \cite{hairer2006geometric} such as St\"{o}rmer-Verlet  are implicit \cite{hairer2003geometric}. Although the Euler-Lagrange scheme associated with the  discrete least action principle \eqref{eqlsedhjd} is symplectic (since it is variational \cite{marsden2001discrete}), it is also implicit and therefore difficult to simulate.
In this paper we will simply discretize the Hamiltonian system with the Leapfrog method \cite{hairer2006geometric} as follows:
\begin{equation}\label{eqkhkwjheddh}
\begin{cases}
p &\leftarrow p-\frac{h}{2}  \partial_q\big(\frac{1}{2}p^T \wK(q,q)p\big)\\
q &\leftarrow q + h\,\wK(q,q)p\\
p &\leftarrow p-\frac{h}{2}  \partial_q\big(\frac{1}{2}p^T \wK(q,q)p\big)\,.
\end{cases}
\end{equation}
For simplicity, although \eqref{eqkhkwjheddh} is not symplectic for our non-separable system \eqref{eqkedmdledkemdl}, it is explicit, time-reversible
and sufficiently stable for our example is.

\begin{Remark}
M. Tao has recently introduced \cite{tao2016explicit} an ingenious method for deriving explicit symplectic integrators for general non-separable Hamiltonian systems that could be employed for \eqref{eqkedmdledkemdl}.
Tao's idea is to consider the augmented Hamiltonian system
\begin{equation}\label{eqjkehdjedd}
\bar{\Hf}(q,p,\bar{q}.\bar{p})=\Hf(q,\bar{p})+\Hf(\bar{q},p)+\omega (\|q-\bar{q}\|_2^2+\|p-\bar{p}\|_2^2)
\end{equation}
 in which the first two terms are  copies of the original system with mixed-up
positions and momenta and the last term is  an artificial restraint ($\omega$ is a constant  controlling the binding of the two copies).
Discretizing \eqref{eqjkehdjedd} via Strang splitting leads to explicit symplectic integrators of arbitrary even order.
\end{Remark}

 \begin{figure}[h]
 \captionsetup{width=\linewidth}
	\begin{center}
			\includegraphics[width= \textwidth]{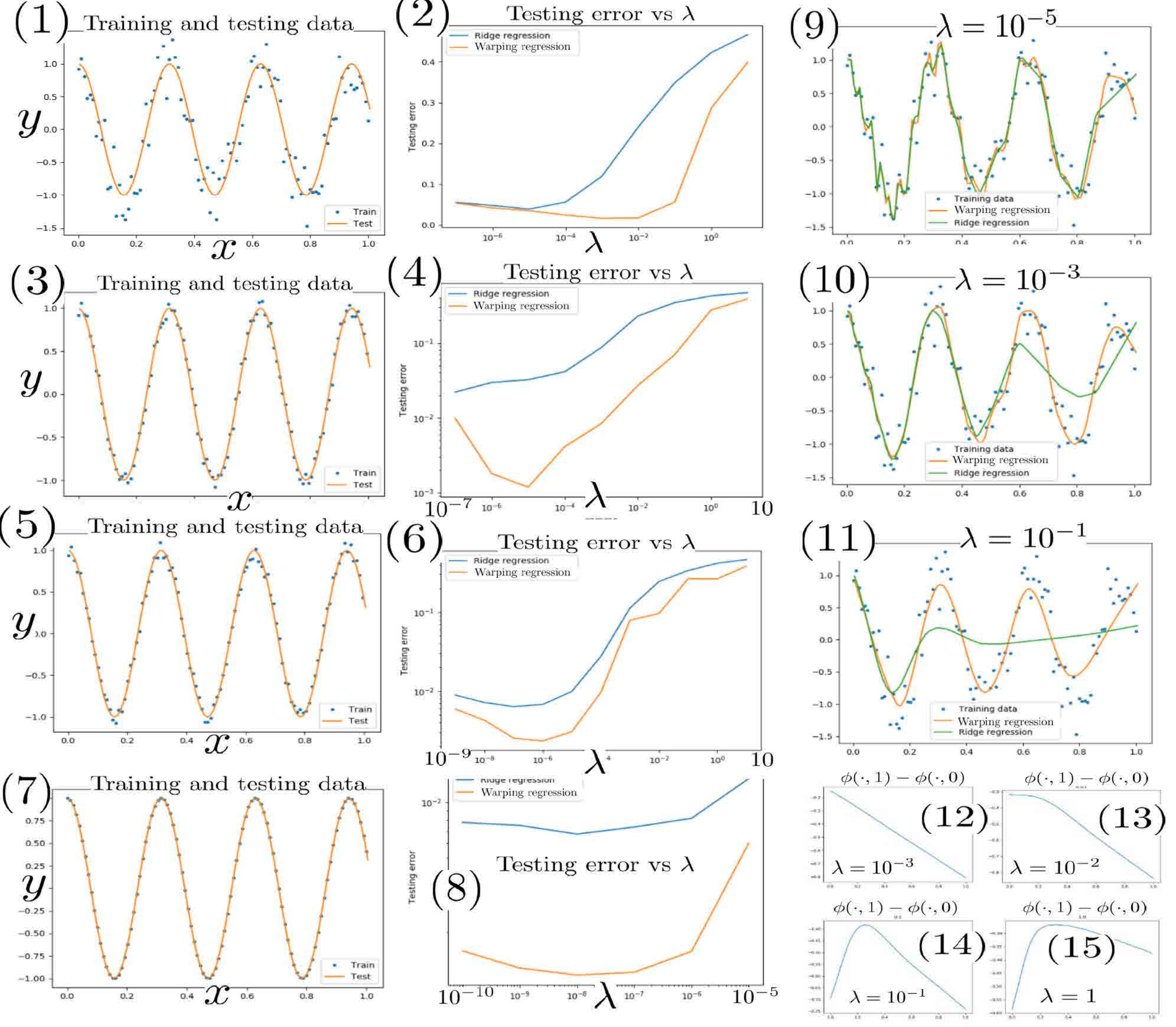}
		\caption{Warping regression vs. ridge regression. (1,3,5,7) Target function and noisy training data with $\sigma_z=1$ for (1), $\sigma_z=0.2$ for (3,5) and $\sigma_z=0$ for (7). (2,4,6,8) testing errors vs. $\lambda$ corresponding to the left column for ridge regression and warping regression. The $y$-axis of (2) is in linear scale. The $y$-axis of (4,6,8) is in log scale. (9,10,11) ridge and mechanical regressors corresponding to (1) for $\lambda=10^{-5}, 10^{-3}, 10^{-1}$. (12,13,14,15) $\phi(\cdot,1)-\phi(\cdot,0)$ corresponding to warping regression for (1) for $\lambda=10^{-3}, 10^{-2}, 10^{-1},1$.}\label{figmrgr1d}
	\end{center}
\end{figure}

 \begin{figure}[h]
 \captionsetup{width=\linewidth}
	\begin{center}
			\includegraphics[width= \textwidth]{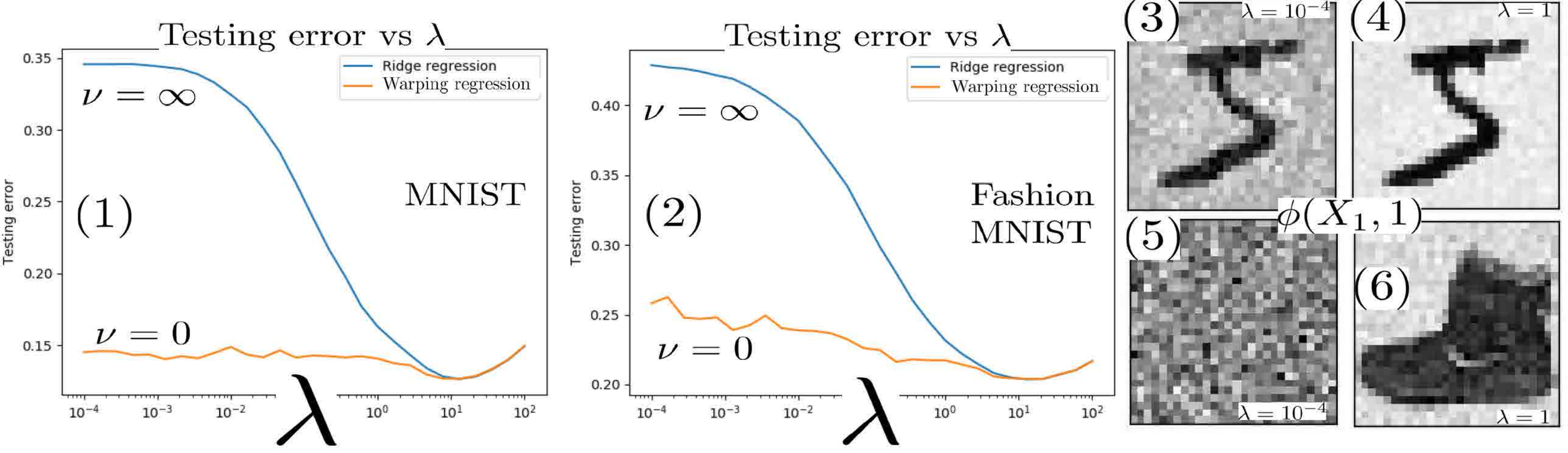}
		\caption{Warping regression vs. ridge regression. (1), (2) Testing errors vs. $\lambda$ for MNIST and Fashion-MNIST for ridge regression ($\nu=\infty$) and warping regression (with $\nu=0$). (3-6) $\phi(X_1,1)$ (3,4) MNIST (5,6) Fashion-MNIST (3,5) $\lambda=10^{-4}$ (4,6) $\lambda=1$. The larger than usual testing errors (around $12\%$ for MNIST and $20\%$ for Fashion-MNIST) are due to the fact that we are using $1000$ points (instead of the usual $60,000$) for training.
The architecture is presented in the introduction of Sec.~12.2. We use randomized feature maps and equivariant feature maps as done implicitly with convolutional networks (Sec.~\ref{subseciuigiwuegew6gd}), which explains the lower accuracy when compared to CNNs. We chose to train on a small subsample of MNIST to ensure low computational complexity (without sub-batching).
}\label{figmrmnist1}
	\end{center}
\end{figure}

\subsubsection{Swiss roll dataset}\label{subsefcfcdcgyf67f6}
We implement the pseudo-algorithm  \eqref{alggtpshootsol} for the
Swiss roll dataset illustrated in Fig.~\ref{figspiral2}. We use the optimal
 recovery loss \eqref{eqhgvygvgyv} to define $\ell$ and $f$.
 For this example $\X=\R^2$, $\Y=\R$, $N=200$, $Y_i=+1$ for the first 100 points, and $Y_i=-1$ for the remaining $100$ points.
$\wK$ is a separable Gaussian kernel (with a nugget $r$) of the form $\wK(z,z')=(k(z,z')+r) I$ with $k(z,z')=e^{-|z-z'|^2/s^2}$ for $z,z'\in \X$, $s=5$ and $r=0.1$. We  simply take $K$ to be the scalar kernel $k(z,z')+r$ with the same parameters as for $\wK$.
Fig. \ref{figspiral2} shows the locations of the points $q_i(t)$ for $i=1,\ldots,200$ which is a solution of the numerical discretization of the Hamiltonian system \eqref{eqkedmdledkemdl} with the Leapfrog method \eqref{eqkhkwjheddh}
and  $h=0.2$. This Hamiltonian system is initialized with the momentum $p(0)$ identified by minimizing  \eqref{eqkjlkjkejwkdj} via gradient descent for three different values of $\nu$ of the regularizing parameter balancing, in \eqref{eqlkjeddjedhjd}, the
RKHS norm of the deformation of the space with that of the regressor $f$. Note that as $\nu$ increases, a greater penalty is placed on that deformation, and the points $q_i(1)$ remain closer to their original position. On the other hand, for a small value of $\nu$, the space will deform to a greater degree to minimize the RKHS norm of the regressor. Fig.~\ref{figspiral2} is also showing the norm of the entries of initial momentum $p(0)$ and final momentum $p(1)$. As discussed in Subsec.~\ref{secspa}, the domination of a few entries supports the suggestion that momentum variables promote sparsity in the representation of the regressor.

\subsection{With feature maps}\label{secnumfm}
In the following experiments we use the variational formulation \eqref{eqlkjwwqehgddjedhjssdBN} with, $r=\rho=0$, $\ell_\Y(Y',Y)=\|Y'-Y\|_{\Y^N}^2$  and use  random features\footnote{Using random features improves computational complexity without incurring significant loss in accuracy, see \cite{rahimi2008random, nelsen2020random}.} to construct $\varphi$ and $\varphi_2$.
We select $\varphi(x)=\ba(W x+b)$ and $\varphi_2(x)=\ba(W^2 x+b^2)$
with $\ba(\cdot)=\max(\cdot,0)$, $W\in \R^{\dim(\Fk)\times \dim(\X)}$, $b\in \R^{\dim(\Fk)}$,
$W^2\in \R^{\dim(\Fk_2)\times \dim(\X)}$, $b^2\in \R^{\dim(\Fk_2)}$.
All the entries of $W, W^2, b, b^2$ are independent and we select
$W_{i,j}, W^2_{i,j}\sim (1.5/\sqrt{\dim(\X)}) \,\cN(0, 1)$ and $b_i, b_i^2\sim 0.1\, \cN(0,1)$.

\subsubsection{One dimensional regression}
To goal of this experiment is to approximate the function  $f^\dagger(x)=\cos(20 x)$ in the interval $[0,1]$ from the observation of $N=100$ data (training) points $(X_i,Y_i)$ where $X_i=i/100$, $Y_i=\cos(20 X_i)+\sigma_z Z_i$ and the $Z_i$ are i.i.d. random variables uniformly distributed in $[-0.5,0.5]$.
Here $\X=\Y=\R$ and we also use $100$ points $(x_i,y_i)_{1\leq i \leq 100}$ to compute testing errors (we take
$x_i=i/100-1/200$ and $y_i=f^\dagger(x_i)$.
We select $\Fk=\R^{200}$ and $\Fk_2=\R^{800}$.
Fig.~\ref{figmrgr1d} compares classical ridge regression ($\nu=\infty$) with warping regression with $\nu=0$.
Note that warping regression has significantly smaller testing errors than ridge regression over a broad range of values for $\lambda$.

\subsubsection{MNIST and Fashion MNIST}
For this experiment we use the MNIST and Fashion MNIST datasets. We use $N=1000$ points $(X_i,Y_i)$ for training and $10000$ points for testing. $f^\dagger$ maps a $28\times 28$ image $X_i \in \R^{28\times 28}$ to a one hot-vector $Y_i\in \R^{10}$ ($Y_{i,j}=1$ if the class of $X_i$ is $j$ and $Y_{i,j}=0$ otherwise).
We select $\Fk=\R^{784}$ and $\Fk_2=\R^{800}$.
Fig.~\ref{figmrmnist1} compares classical ridge regression ($\nu=\infty$) with warping regression with $\nu=0$.
Warping regression has significantly smaller testing errors than ridge regression over a broad range of values for $\lambda$, and the deformation of the space $\phi(\cdot,1)$ seems to regularize the classification problem.

\section{Reduced equivariant multi-channel (REM) kernels and feature maps}\label{secrem}
This section introduces the type of structured kernels implicitly associated with CNNs. These kernels preserve the relative pose information ( across layers and enable the generalization of CNNs to arbitrary groups of transformations acting on arbitrary spaces.

\subsection{Reduced kernels}\label{secohhdvk}
In the setting of Sec.~\ref{secovk},  let  $K\,:\,\X\times \X\rightarrow \L(\Y)$ be an operator-valued kernel and
 let  $P\,:\,\X\rightarrow \X$ and $R\,:\,\Y\rightarrow \Y$ be
  linear projections.
Note that $R K(P x,Px') R \,:\, \X\times \X\rightarrow   \L(R \Y) $ is also an operator-valued kernel.
The following proposition generalizes \eqref{eqhgvygvgyv} and \eqref{eqhgvygvgyv2} to partial measurements on the inputs and outputs of the unknown function $f^\dagger$ in Problem \ref{pb828827hee}.

\begin{Proposition}
Using the relative error in $\|\cdot\|_{K}$-norm as a loss, the minimax optimal recovery of an unknown function
$f^\dagger\in \H_K$ given $R f^\dagger(P X)=Z$ (with $X:=(X_1,\ldots,X_N)\in \X^N$ and $Z:=(Z_1,\ldots,Z_N)\in (R\Y)^N$) is the minimizer of
\begin{equation}\label{eqhgdsvygvgyvee}
\begin{cases}
\text{Minimize }&\|f\|_{K}\\
\text{subject to }&R f(P X)=Z \,,
\end{cases}
\end{equation}
which admits the representation
\begin{equation}\label{eqhgvydedsgdvgyv2}
 f(\cdot)= K(\cdot,P X) R  (R K (P X,P X) R)^{-1} Z\text{ with }\|f\|_K^2=Z^T (R K (P X,P X) R)^{-1} Z\,,
\end{equation}
where $R K (P X,P X) R$ is the $N\times N$ block-operator matrix with entries $R K (P X_i,P X_j) R$ and $K(\cdot,P X) R$ is the N-vector with entries $K(\cdot,P X) R$.
\end{Proposition}
\begin{proof}
The proof of minimax optimality of the minimizer of \eqref{eqhgdsvygvgyvee} is similar to that of \cite[Thm.~12.4,12.5]{owhadi2019operator}. The representation \eqref{eqhgvydedsgdvgyv2} follows by observing that
that $R f(P X_i)=Z_i$ and that $f$ is $\<\cdot,\cdot\>_K$-orthogonal to the set of $g\in \H_K$ such that $R g(P X_i)=0$ (since $f$ has the representation $f=\sum_i K(\cdot,P X_i) R V_i$ with $V_i\in \Y$ and
$\<K(\cdot,P X_i) R V_i, g\>_K=\<g(P X_i),R V_i\>_\Y=\<R g(P X_i),V_i\>_\Y=0$ via the reproducing identity).
\end{proof}

\subsection{Equivariant multi-channel kernels}\label{secqker}
We will now present a generalization of the equivariant kernels of \cite{reisert2007learning}.

\subsubsection{The unitary group of transformations on the base space}
Let $\Xf$ be a separable Hilbert space.
Let $\G$ be a (compact, possibly finite)
group of linear unitary transformations acting on $\Xf$: $g\in \G$ maps $\Xf$ to $\Xf$, $\G$ is closed under composition, $\G$ contains the identity map $i_d$, $g\in \G$ has an inverse $g^{-1}$ such that $g g^{-1}=g^{-1} g =i_d$, and
$\<g x, g x'\>_\Xf=\<x,x'\>_\Xf$ for $g\in \G$ and $x,x'\in \Xf$ (i.e. $g^T=g^{-1}$ where $g^T$ is the adjoint of $g$).
Write $dg$ for the Haar measure associated with $\G$ and $|\G|:=\int_\G dg$ for the volume of the group ($|\G|=\operatorname{Card}(\G)$ when the group is finite) and assume $\G$ to be unimodular ($dg$ is invariant under both the left and right action of the group, i.e. $\int_{\G}f(g)\,dg=\int_{ \G}f(gg')\,dg=\int_{\G}f(g'g)\,dg$ for $g'\in \G$). Write $\E_\G$ for the expectation with respect to the probability distribution induced by
$dg/|\G|$ on $\G$.

\subsubsection{Extension to multiple channels}
Let $c$ be a strictly positive integer called the number of channels. Let $\Xf^c$ be the $c$-fold product space of $\Xf$ endowed with the scalar product defined by $\<x,x'\>_{\Xf^c}:=\sum_{i=1}^c \<x_i,x_i'\>_\Xf$ for $x=(x_1,\ldots,x_c)\in \Xf^c$ and $x'\in \Xf^c$.
The action of the group $\G$ can be naturally diagonally be extended to $\Xf^c$ by
\begin{equation}
g (x_1,\ldots,x_c):=(g x_1,\ldots, g x_c) \text{ for } g\in \G\text{ and } (x_1,\ldots,x_c)\in \Xf^c\,.
\end{equation}
Note that $\G$ remains unitary on $\Xf^c$ ($\<g x, g x'\>_{\Xf^c}=\<x,x'\>_{\Xf^c}$).

\subsubsection{Equivariant multi-channel kernels}
We will now introduce equivariant multi-channel kernels in the setting of Subsec.~\ref{subseciuwge67d}.
\begin{Definition}\label{defnkjdhbejhdb}
Let
$\X=\Xf^{c_1}$ and $\Y=\Xf^{c_2}$ with  $c_1,c_2\in \mathbb{N}^*$.
We say that an operator-valued kernel $K\,:\, \X\times \X\rightarrow \L(\Y)$ is
 {\bf $\G$-equivariant} if
\begin{equation}
K( g x, g' x')=g K(x, x') (g')^T \text{ for all }g, g'\in \G\,.
\end{equation}
Similarly we say that a function $f\,:\, \X\rightarrow \Y$ is $\G$-equivariant if
\begin{equation}\label{eqhvhgvhddeeg}
f(g x)=g f(x) \text{ for all }(x,g)\in \X\times \G\,.
\end{equation}
\end{Definition}
Set $\X=\Xf^{c_1}$ and $\Y=\Xf^{c_2}$ as in Def.~\ref{defnkjdhbejhdb}.
\begin{Proposition}\label{propeqk}
Given a (possibly non-equivariant) kernel $K\,:\, \X\times \X\rightarrow \L(\Y)$,
 \begin{equation}\label{eqjhehdddebdshjdb}
 K^\G(x,x'):=\frac{1}{|\G|^2}\int_{\G^2} g^T K(g x,g'x') g'\,dg\,dg':=\E_{\G^2}\big[g^T K(g x,g'x') g'\big]\,,
 \end{equation}
is a  $\G$-equivariant kernel $K^\G\,:\, \X\times \X\rightarrow \L(\Y)$.
\end{Proposition}
\begin{proof}
The proof  is similar to that of \cite[Prop.~2.2]{reisert2007learning}. Simply
observe that for $\bar{g}, \bar{g}'\in \G$,\\ $\E_{\G^2}\big[g^T K( g \bar{g}x,  g' \bar{g}'x') g'\big]=
\bar{g}\E_{\G^2}\big[(g \bar{g})^T K( g \bar{g}x,  g' \bar{g}'x') g' \bar{g}' \big] (\bar{g}')^T$ .
\end{proof}
We say that $K\,:\, \X\times \X\rightarrow \L(\Y)$ is {\bf  $\G$-invariant}\footnote{Given a non-invariant  kernel $K$ Haar integration can also be used \cite{haasdonk2005invariance} to the derive the invariant kernel
$\E_{\G^2}\big[ K(g x,g'x')\big]$.} if $K(g x,g' x')=K(x,x')$ for $(x,x',g,g')\in (\X)^2\times \G^2$.
We say that $K\,:\, \X\times \X\rightarrow \L(\Y)$ is {\bf weakly $\G$-invariant} if $K(g x,g x')=K(x,x')$ for $(x,x',g)\in \X^2\times \G$.

\begin{Remark}
If $K\,:\, \X\times \X\rightarrow \L(\Y)$ is scalar and weakly $\G$-invariant then
  \begin{equation}\label{eqkjddjddnedjkdn}
  K^\G(x,x')=\E_{\G}\big[ K( x,g'x') g'\big]\,,
  \end{equation}
 since $\E_{\G^2}\big[ g^T K(g x,g'x') g'\big] =  \E_{\G^2}\big[  K( x,g^T g'x') g^T g' \big] =\E_{\G}\big[  K( x,g'x') g'\big]$.
\eqref{eqkjddjddnedjkdn}  matches the construction of \cite{reisert2007learning} for $c_1=c_2=1$.
 \end{Remark}

The interpolant \eqref{eqhgvygvgyv2} of the data $X_i,Y_i$ with a  $\G$-equivariant kernel $K$ (1) is a equivariant function (satisfies $f(gx)=g f(x)$) and (2) is equal to the interpolant of the enriched data $(g X_i, g Y_i)_{g\in \G, 1\leq i \leq N}$ with $K$. However, although interpolating with an equivariant kernel implicitly enriches the data, interpolating the enriched data $(g X_i, g Y_i)_{g\in \G, 1\leq i \leq N}$ with a non-equivariant kernel $K$ does not guarantee  the equivariance \eqref{eqhvhgvhddeeg} of the interpolant \eqref{eqhgvygvgyv2}. Furthermore, we have the following variant of  \cite[Thm.~2.8]{reisert2007learning}.
\begin{Theorem}
If  $K$ is scalar and weakly $\G$-invariant then
the minimizer of \eqref{eqhgvygvgyv} with the constraint that $f$ must also be $\G$-equivariant is
$f^\G(\cdot):= K^\G(\cdot,X)  K^\G(X,X)^{-1} Y$.
\end{Theorem}
\begin{proof}
By construction $f^\G$ satisfies the constraints of \eqref{eqhgvygvgyv} and is $\G$-equivariant.
To show that $f^\G$ is the minimizer simply observe that $\<f^\G, u\>_K=0$
 if $u \in \H_K$ is $\G$-equivariant and satisfies $u(X)=0$. Indeed   (writing $V:=K^\G(X,X)^{-1} Y$)
$\<f^\G, u\>_K= \sum_i \E_{\G}\big[ \< K( \cdot,g'X_i) g' V_i, u\>_K\big]=0$ since (by the reproducing identity)
$\<  K( \cdot,g'X_i) g' V_i, u\>_K=\<u(g'X_i),g' V_i\>_\Y=\<g' u(X_i),g' V_i\>_\Y=0$.
\end{proof}

\begin{Remark}\label{rmkpose}
Let $(x,g^\dagger) \in \X\times \G$ and $y=g^\dagger x$ and consider the problem of recovering $g^\dagger$ (which we refer to as the relative pose between $x$ and $y$) from the observation of $K(x,x)$, $K(y,y)$ and $K(x,y)$.
Although this problem is impossible if $K$ is   $\G$-invariant (since $K(x,x)=K(y,y)=K(x,y)$), it remains solvable if $K$ is   $\G$-equivariant (since $K(x,y)=K(x,x)g^T$ and $g=(K(x,y))^T K(x,x)^{-1}$). Therefore, contrary to invariant kernels \cite{haasdonk2005invariance}, equivariant kernels preserve the relative pose  information between
objects \cite[Sec.~2.3]{reisert2007learning}. The notion of equivariance has been used in deep learning to
design convolutional neural networks \cite{lecun1999object} on non-flat manifolds \cite{cohen2016group} and for preserving
intrinsic  part/whole spatial relationship in image recognition
 \cite{ sabour2017dynamic}.
 \end{Remark}

 The following theorem shows that interpolants/regressors obtained from warping regression with an equivariant kernel are also equivariant.
 \begin{Theorem}
 Let $\H_\Gamma$ be the RKHS defined by a $\G$-equivariant kernel $\Gamma\,:\,\X\times \X \rightarrow \L(\X)$. Then for $v\in C([0,1],\H_\Gamma)$
 obtained as a minimizer of \eqref{eqlkjgehgddjedhjd} or \eqref{eqlkjgehgddjedhjdreg}, the solution $\phi^v$ of \eqref{eqflmp} is also $\G$-equivariant in the sense that
 \begin{equation}
 \phi^v(g z,t)=g\phi^v(z,t) \text{ for all }(z, g,t)\in \X\times \G \times [0,1]\,.
 \end{equation}
 \end{Theorem}
\begin{proof}
The proof follows from Thm.~\ref{thmkljedjnedesdjd} and Thm.~\ref{thmkljedjnedesdjdreg} by continuous induction on $t$. Indeed  \eqref{eqtttc} implies that
$\dot{\phi}^v( g z,t)=g  \dot{\phi}^v(  z,t)$ as long as $\phi^v( g z,t)=g  \phi^v(  z,t)$.
\end{proof}

 \begin{figure}[h!]
	\begin{center}
			\includegraphics[width= \textwidth]{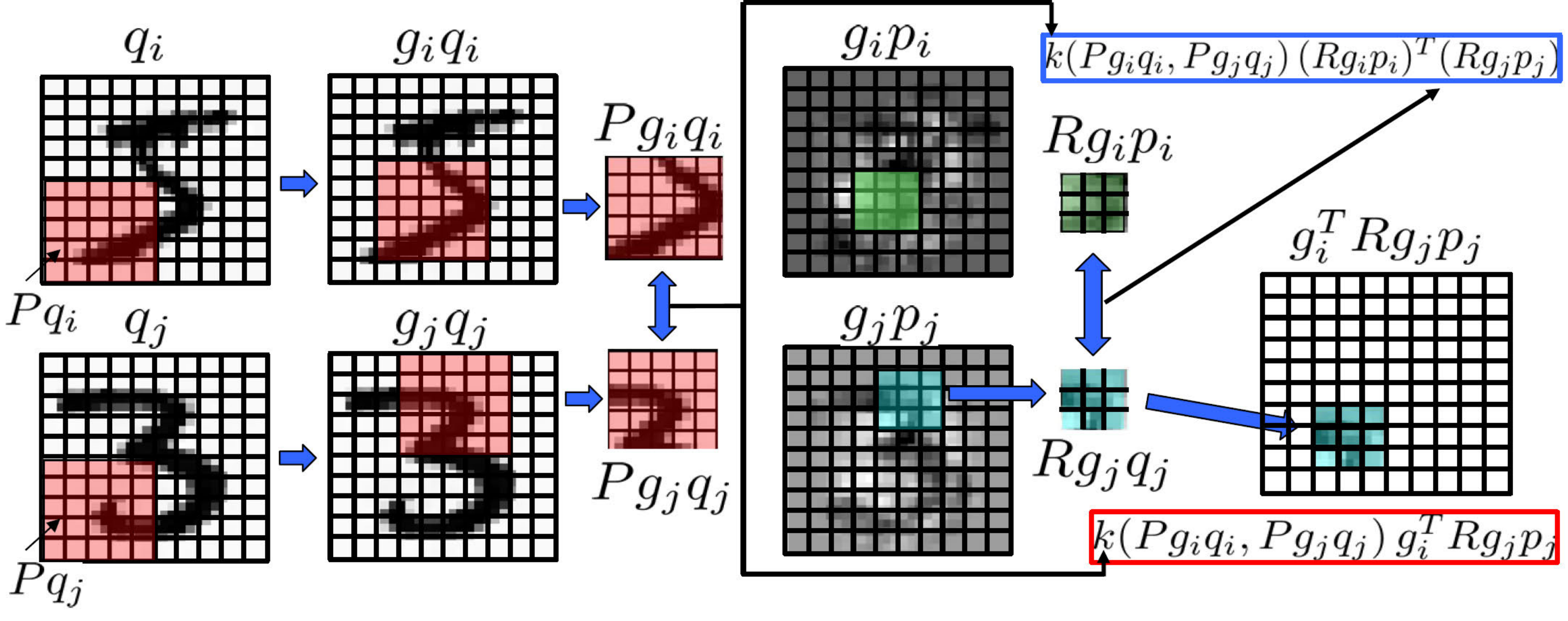}
		\caption{Warping regression with REM  kernels.}\label{figmr3}
	\end{center}
\end{figure}

\subsection{REM kernels}
In the setting of Sec.~\ref{secqker}, let $R$ and $P$ be linear projections from $\Xf$ onto closed linear subspaces of $\Xf$.
Extend the action of $P$  to $\X=\Xf^{c_1}$ by
$P(x_1,\ldots,x_{c_1})=(Px_1,\ldots, P x_{c_1})$. Similarly extend the action of  $R$ to $\Y=\Xf^{c_2}$.
Observe that, given an operator-valued kernel $K\,:\, P \X\times P\X\rightarrow \L(R \Y)$,
\begin{equation}\label{propeqked}
\bar{K}(x,x'):=R K(P x, P x') R
\end{equation}
is an operator-valued kernel\footnote{Note that $K$ can also be identified as the reduced kernel of $\bar{K}$. When the dimension of $P\X$ is low then the interpolation of functions mapping $P\X$ to  $R\X$ does not suffer from the curse of dimensionality.} $\bar{K}\,:\, \X\times \X\rightarrow \L( Y)$.
Prop.~\ref{propeqk} implies that
\begin{equation}\label{propddeeqked}
C(x,x'):=\bar{K}^\G(x,x')=\E_{\G^2}\big[g^T R K(P g x, P g' x')R g'\big]
\end{equation}
is an equivariant operator-valued kernel $C\,:\, \X\times \X\rightarrow \L( \Y)$. We call \eqref{propddeeqked} a REM (reduced equivariant multi-channel) kernel.

Fig.~\ref{figmr3} illustrates the Hamiltonian system \eqref{eqkedmdledkemdl} for $c_1=c_2=1$,
with
$\Gamma=C=$\eqref{propddeeqked} obtained from  a scalar kernel $K(x,x')=k(x,x') I_{R\Y}$ (where the arguments of $k$ are $5\times 5$ images) and the (finite) group of translations $\G$ on periodized $10\times 10$ images. Note that $P q_i$ projects the image $q_i$ to its bottom left $5\times 5$ sub-image and
$R p_i$ projects the image $p_i$ to its bottom left $3\times 3$ sub-image.
$P g_i q_i$  translates $q_i$ by $g_i$ before the projection $P$ which is equivalent to projecting $q_i$ onto the $g_i^T$ translation of the original $5\times 5$ patch. $\sum_j k(P g_i q_i, P g_j q_j) g_i ^T R g_j p_j$ creates a $10\times 10$ image adding (over $j$) the $q_i^T$ translates of sub-images $g_i ^T R g_j p_j$ weighted by $k(P g_i q_i, P g_j q_j)$. We will now show that this is equivalent to performing a weighted convolution, and  convolutional neural networks \cite{lecun1995convolutional} can be recovered as the feature map version of this algorithm.

 \begin{figure}[h!]
	\begin{center}
			\includegraphics[width= \textwidth]{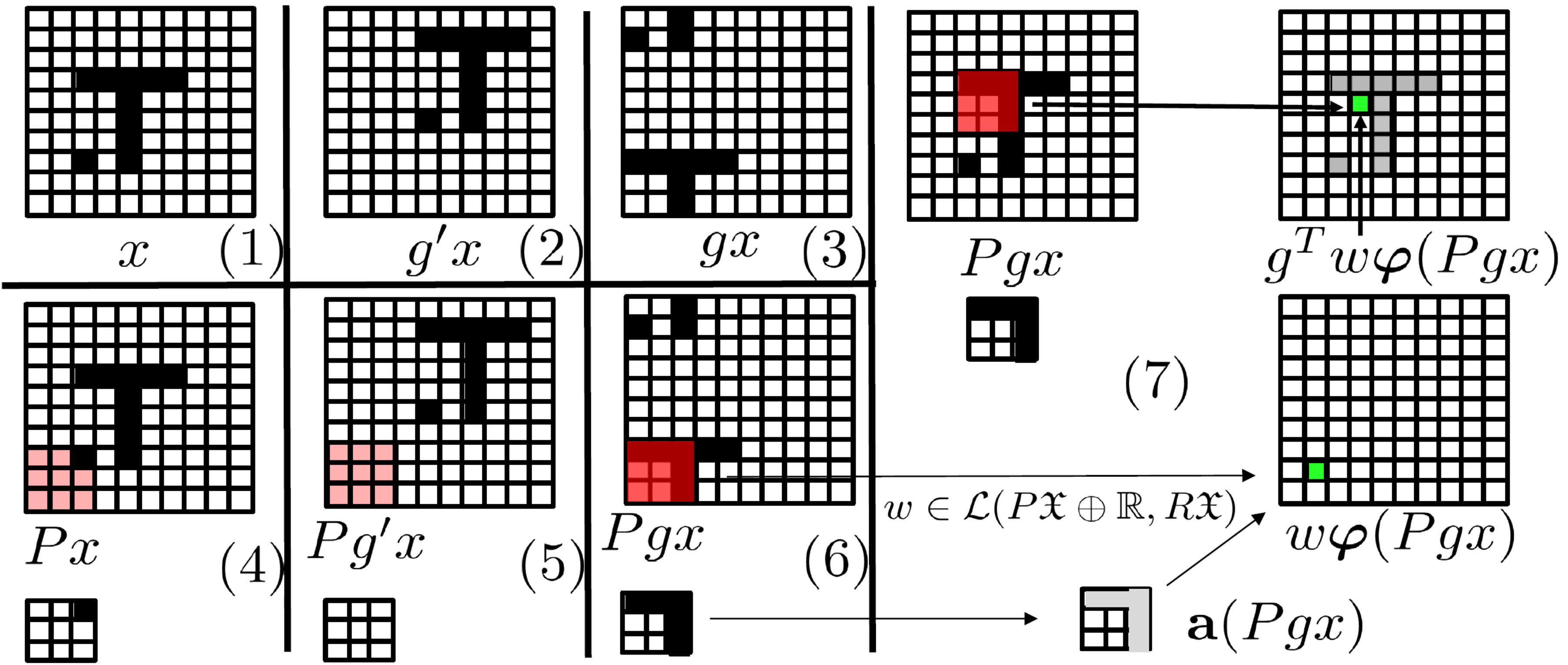}
		\caption{REM feature maps.}\label{figconv1}
	\end{center}
\end{figure}

\subsection{REM  feature maps}\label{subsechj23hdyudd}
Let $\F$ and $\psi\,:\, \X\rightarrow \L(\Y,\F)$ be a
 feature space and map associated with the kernel $K$ in \eqref{propddeeqked}. Then
$C(x,x')=\E_{\G^2}\big[g^T R \psi^T(P g x) \psi(P g' x')R g'\big]$
implies that $C$ has feature space $\F$ and feature map $\Psi$ defined by
\begin{equation}\label{eqkldejkdhlekhs}
\Psi(x) y=\E_\G\big[ \psi(P g x) R g y \big]\,.
\end{equation}
If $K$ is a scalar kernel as in Subsec.~\ref{secscalfeatmap} with feature space/map $\Fk$ and $\varphi\,:\, \X\rightarrow \Fk$, then
$\F=\L(\Fk,R\Y)$ and
 $\psi(x) R y=R y \varphi^T(x)$ imply $\Psi(x) y=\E_\G  \big[   R g y  \varphi^T(P g x) \big]$ and (for $\alpha \in \F$)
\begin{equation}\label{eqjhwbekjbdjhedbd}
\Psi^T(x) \alpha=\E_\G  \big[   g^T \alpha \varphi(P g x) \big]\,.
\end{equation}
If $\varphi$ is obtained from an elementwise nonlinearity activation function as in \eqref{eqkjdekjddkjsd} and Rmk.~\ref{rmkkjhevgcvcg78}
then  $\Fk=P\X\oplus \R$, and for $\alpha=w\in \L(P\X\oplus \R, R \Y)$ we have
\begin{equation}\label{eqjbhhebddbdbgw}
\Psi^T(x) \alpha=\E_\G  \big[   g^T (w \bvarphi(P g x)) \big]\,.
\end{equation}
We call  \eqref{eqkldejkdhlekhs}, \eqref{eqjhwbekjbdjhedbd} and \eqref{eqjbhhebddbdbgw} REM (reduced equivariant multi-channel) feature maps.

Fig.~\ref{figconv1} shows the action of \eqref{eqjbhhebddbdbgw}. In that illustration $c_1=c_2=1$, the elements of $\Xf$ are $10\times 10$ images, and $\G$ is the group of translations acting on $10\times 10$ images with periodic boundaries as shown in subimages (1-3). $P x$ projects the $10\times 10$ image $x$ onto the lower left $3\times 3$ sub-image (by zeroing out the pixels outside that left corner). The action of $P$ on $x$ and the translation of $x$ by $g'$ and $g$ are
  illustrated in subimages (4-6). Note that  translating $x$ by $g$ before applying $P$ is equivalent to translating the action of $P$ as illustrated in subimage (7) and commonly done in CNNs. $R x$ projects the $10\times 10$ image onto the green pixel at the bottom right of subimage (7). In the setting of CNNs $w\in \L(P \Xf\oplus \R, R\Xf )$ is one convolutional patch incorporating a $3\times 3$ weight matrix $W$ and a $1\times 3$ vector $b$ and  computing $g^T w \bvarphi(P g x)$ is equivalent to obtaining the value of the green pixel on the top right of subimage (7) by computing  $W \ba(P gx)+b$.

 \begin{figure}[h!]
	\begin{center}
			\includegraphics[width=0.8 \textwidth]{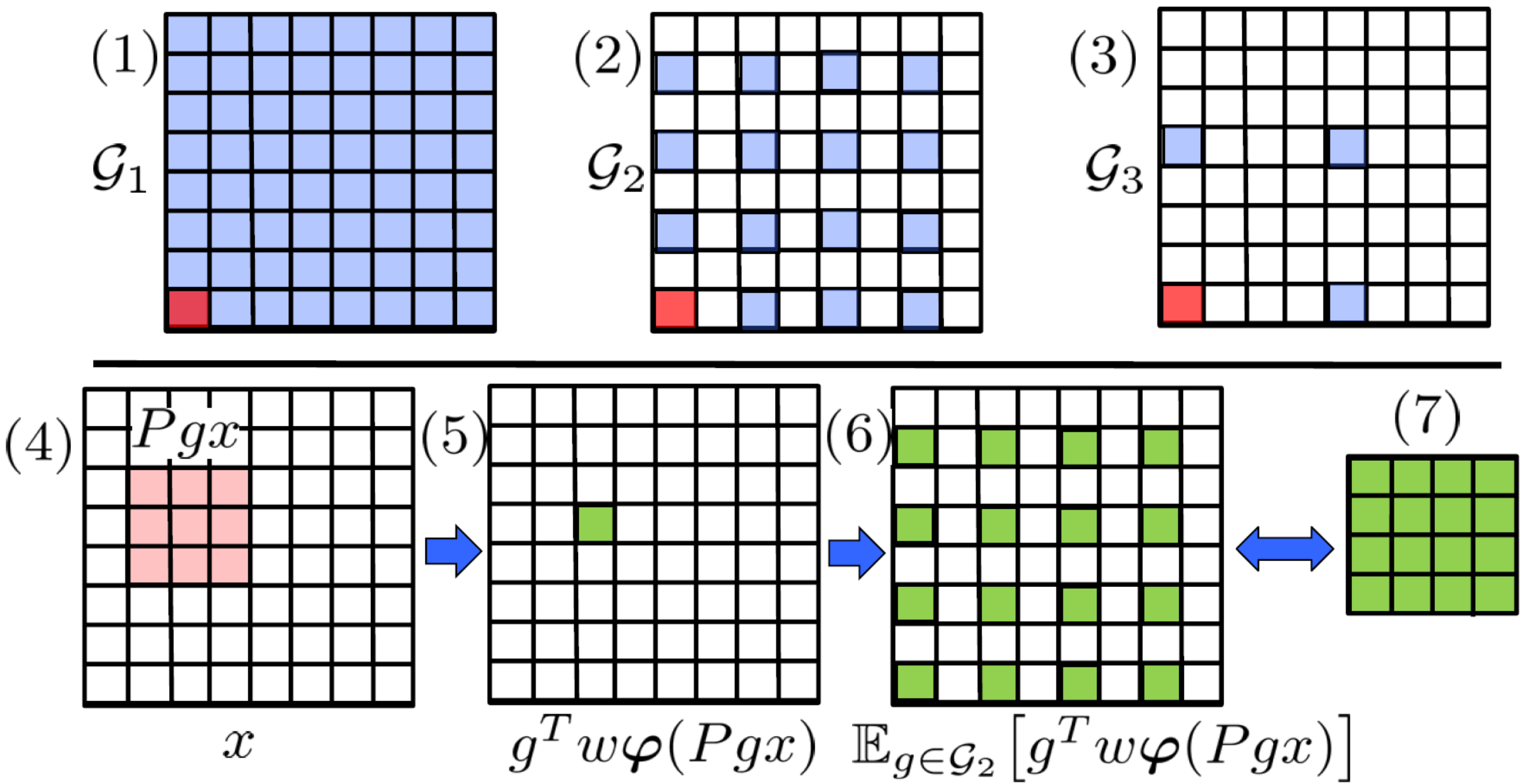}
		\caption{Downsampling with subgrouping.}\label{figds}
	\end{center}
\end{figure}

\subsection{Downsampling with subgrouping}
ANNs include downsampling operations such as pooling or striding.
Downsampling is incorporated in REM kernels and feature maps by employing sub-groups of $\G$ in the construction of
the REM feature maps. Note that \eqref{eqjhwbekjbdjhedbd} and \eqref{eqjbhhebddbdbgw} are contained in
\begin{equation}
\G R \X:=\oplus_{ \G} g R \X
\end{equation}
with $g R \X:=\{g R x\mid x\in \X\}$. Note that $\G R \X\subset \X$ and this inclusion can be a strict one when $\G$ is a proper subgroup of an overgroup.
 When $\G$ is a group of translations, then subgrouping is equivalent to striding.
Fig.~\ref{figds} illustrates the proposed downsampling approach for $\X=\Xf$. In that illustration, the elements of $\Xf$ are $8\times 8$ images (with periodic boundaries). (1)
$\G_1$ is the group of all $64$ possible translations. (2) $\G_2$ is a group of $16$ possible translations   obtained as a
 sub-group of $\G_1$    with a stride of $2$. (3) $\G_3$ is a group of $4$ possible translations  obtained as a
 sub-group of $\G_1$    with a stride of $4$ or as a  sub-group of $\G_2$    with a stride of $2$.
The bottom row shows the action of a REM feature map constructed from the sub-group $\G_2$. (4) shows $Pgx$ for a given $g\in \G_2$ ($Px$ is $3\times 3$ image).
(5) shows $g^T w \bvarphi(Pgx)$ ($R\Xf$ is a set of $1\times 1$ images and $w\in \L(P \Xf\oplus \R, R\Xf )$). (6) shows the average of $g^T w \bvarphi(Pgx)$ over $g\in \G_2$. The range of $\E_{ \G_2}[g^T w \bvarphi(Pgx)]$ is the set of $8\times 8$ images whose pixel values are zero outside the green pixels. (7) Ignoring the white pixels (whose values are zero), the  range $\G_2 R \Xf$ of $\E_{\G_2}[g^T w \bvarphi(Pgx)]$ (writing $\E_{\G_2}$  for the expectation with respect to the normalized Haar measure on $\G_2$)  can be identified with the set of $4\times 4$ images as it is done with CNNs.

\section{Composed idea registration}\label{seckehddjhevd}
We will now compose ridge regression, warping regression, and idea registration blocks across layers of abstraction between $\X$ and $\Y$.
The resulting input/output functions generalize ANNs, and deep kernel learning \cite{wilson2016deep}.
This section shows that the proposed framework is as expressive as standard deep learning and the results of the previous sections generalize to
the composition of idea registration blocks, which we call composed idea registration. In particular, in
 Sec.~\ref{subseciuigiwuegew6gd}, we will now show that  CNNs and ResNets can be recovered as particular instances of composed warping regression with reduced equivariant multi-channel (REM) kernels and feature maps introduced kernels introduced in Sec.~\ref{secrem}.

\subsection{Block diagram representation}\label{subseckjageejhgdy}
For ease of presentation and conceptual simplicity, we  will first summarize  Sec.~\ref{secregular} in block diagram representation.

\subsubsection{Warping regression}
Given $p_0\in \X^N$,  $X\in \X^N$, $x\in \X$, let $(q^s,p^s)$ be the solution of \eqref{ejkhdbejhdbdreg} with initial value
$(q^1,p^1)=(X,p_0)$, let $v_s=\wK( \cdot, q^s)p^s=$\eqref{eqkjhbdejdhbdh} and set
 $x'=(I+v_L)\circ \cdots \circ (I+v_1)(x)$, $X'=q^{L+1}$ and $\Vk=\frac{1}{2}  \sum_{s=1}^L  (p^s)^T \wK_r(q^s,q^s) p^s\, \Delta t$ with $\Delta t=1/L$.
 We represent the corresponding multivariate function
 $(x,X,p_0)\rightarrow (x',X',\Vk)$
  with the diagram
 \begin{equation}\label{eqjhbdjhbjehddd}
 \uset[1ex]{\uset[0.5ex]{\Vk}{\downarrow}}{\oset[0.5ex]{\oset[-0ex]{p_0}{\downarrow}}{\tensor*[^{X\rightarrow}_{\phantom{+}x\rightarrow}]{\boxed{\Gamma\mid r \mid L}}{^{\rightarrow X'}_{\rightarrow x'}}}}\,.
 \end{equation}
A deformation   $x'=\phi_L(x)=(I+v_L)\circ \cdots \circ (I+v_1)(x)$ obtained by minimizing \eqref{eqlktddsyeytdsefsreg} must be (Thm.~\ref{thmksbsahskd2greg}) of the form \eqref{eqjhbdjhbjehddd} where $\Vk$ is
$\frac{1}{2}\,L\sum_{s=1}^L \big(\|v_s\|_{\Gamma}^2+ \frac{1}{r} \|q^{s+1}-(I+v_s)(q^s)\|_{\X^N}^2\big)$.
Furthermore (by the proof of Thm.~\ref{thmwksjhgw76sgw}) \eqref{eqjhbdjhbjehddd} is uniformly continuous in $x,X,p_0$ if $p_0$ is restricted to a compact set, and $\Vk$ diverges uniformly towards $\infty$ as $p_0^T p_0\rightarrow \infty$.

\subsubsection{Idea registration} Given $p_0\in \X^N$,  $X\in \X^N$, $x\in \X$, let $(q(t),p(t))$ be the solution of \eqref{eqkedmdledkemdlreg} with initial value
$(q(0),p(0))=(X,p_0)$, let $v(\cdot,t)= \wK(\cdot,q) p=$\eqref{eqkjhwbkhjbjewhbd} and set
 $x'=\phi^v(x,1)$ (where $\phi^v$ is defined as the solution of \eqref{eqflmp}), $X'=q(1)$ and
  $\Vk= \frac{1}{2}p^T(0) \wK_r\big(X,X\big) p(0)$.
 We represent the corresponding multivariate function with the following diagram
  \begin{equation}\label{eqjhbdjhbjehddd2}
 \uset[1ex]{\uset[0.5ex]{\Vk}{\downarrow}}{\oset[0.5ex]{\oset[-0ex]{p_0}{\downarrow}}{\tensor*[^{X\rightarrow}_{\phantom{+}x\rightarrow}]{\boxed{\Gamma\mid r \mid \infty}}{^{\rightarrow X'}_{\rightarrow x'}}}}\,.
 \end{equation}
A deformation   $\phi^v$ obtained by minimizing \eqref{eqlkjgehgddjedhjdreg} must be (Thm.~\ref{thmksbsahshkdreg}) of the form \eqref{eqjhbdjhbjehddd2} and $\Vk$ is
$\frac{1}{2}\,\int_0^1 \big(\|v\|_{\Gamma}^2+ \frac{1}{r}\|\dot{q}-v(q,t)\|_{\X^N}^2\big)\,dt$.
 Furthermore, if  $p_0$ is restricted to a compact set then
\eqref{eqjhbdjhbjehddd2} is uniformly continuous in $x,X,p_0$ and
  \eqref{eqjhbdjhbjehddd} converges\footnote{By Thm.~\ref{thmksbsahskd2greg}, given same inputs, all the outputs of \eqref{eqjhbdjhbjehddd} converge to  the outputs of \eqref{eqjhbdjhbjehddd2}.}  uniformly towards \eqref{eqjhbdjhbjehddd2}.

 \subsubsection{Ridge regression}
Given $Z\in \X^N$,  $X\in \X^N$ and $x\in \X$, set $Y'=K(X,X)Z$, $y=K(x,X)Z$ and $\Wk=Z^T K_\rho(X,X)Z$.
 We represent the corresponding multivariate function with the following diagram
   \begin{equation}\label{eqjhbdjhbjehddd3}
 \uset[1ex]{\uset[0.5ex]{\Wk}{\downarrow}}{\oset[0.8ex]{\oset[0.4ex]{Z}{\downarrow}}{\tensor*[^{X\rightarrow}_{\phantom{+}x\rightarrow}]{\boxed{K\mid \rho}}{^{\rightarrow Y'}_{\rightarrow y}}}}\,.
 \end{equation}
 A function $f$ minimizing \eqref{eqlktdelldreg} must be of the form \eqref{eqjhbdjhbjehddd3} and $\Wk$ is the value of
 $\|f\|_{K}^2+\frac{1}{\rho} \|f(X)-Y'\|_{\Y^N}^2$.

 \subsubsection{Composing blocks}
Using
$\tensor*[^{Y'\rightarrow}_{Y\rightarrow}]{\boxed{\ell_\Y}}{_{\rightarrow \ell_\Y}}$ for the block diagram representation of the loss $\ell_\Y(Y',Y)$,  warping regression can be represented with the diagram
 \begin{equation}\label{eqjhbdjhbjehddddis}
 \uset[1ex]{\uset[0.5ex]{\Vk}{\downarrow}}{\oset[0.5ex]{\oset[-0ex]{p_0}{\downarrow}}{\tensor*[^{X\rightarrow}_{\phantom{+}x\rightarrow}]{\boxed{\Gamma\mid r \mid L}}{^{\rightarrow }_{\rightarrow }}}}\kern-1.9em
  \uset[1ex]{\uset[0.5ex]{\Wk}{\downarrow}}{\oset[0.8ex]{\oset[0.4ex]{Z}{\downarrow}}{\tensor*[^{\phantom{X\rightarrow}}_{\phantom{+x\rightarrow}}]{\boxed{K\mid \rho}}{^{\phantom{\rightarrow Y'}}_{\rightarrow y}}}} \kern-2.1em \tensor*[^{\xrightarrow{\hspace*{1.5cm}}}_{Y\rightarrow}]{\boxed{\ell_\Y}}{_{\rightarrow \ell_\Y}}
  \,,
 \end{equation}
 and idea registration can be represented with the diagram
  \begin{equation}\label{eqjhbdjhbjehddddis2}
 \uset[1ex]{\uset[0.5ex]{\Vk}{\downarrow}}{\oset[0.5ex]{\oset[-0ex]{p_0}{\downarrow}}{\tensor*[^{X\rightarrow}_{\phantom{+}x\rightarrow}]{\boxed{\Gamma\mid r \mid \infty}}{^{\rightarrow }_{\rightarrow }}}}\kern-1.9em
  \uset[1ex]{\uset[0.5ex]{\Wk}{\downarrow}}{\oset[0.8ex]{\oset[0.4ex]{Z}{\downarrow}}{\tensor*[^{\phantom{X\rightarrow}}_{\phantom{+x\rightarrow}}]{\boxed{K\mid \rho}}{^{\phantom{\rightarrow Y'}}_{\rightarrow y}}}} \kern-2.1em \tensor*[^{\xrightarrow{\hspace*{1.5cm}}}_{Y\rightarrow}]{\boxed{\ell_\Y}}{_{\rightarrow \ell_\Y}}
  \,.
 \end{equation}
 For both diagrams $p_0$ and $Z$ are identified as minimizers of the
  $\text{Total Loss}=\nu \Vk+\lambda \Wk+ \ell_\Y$ and are contained in a set that is closed and uniformly bounded (in $X$).
  As $L\rightarrow \infty$,  \eqref{eqjhbdjhbjehddddis} converges uniformly to \eqref{eqjhbdjhbjehddddis2}, the adherence values of the minimizers of
  \eqref{eqjhbdjhbjehddddis} are the minimizers of  \eqref{eqjhbdjhbjehddddis2}, the total loss of \eqref{eqjhbdjhbjehddddis} converges to that of \eqref{eqjhbdjhbjehddddis2}.

\subsection{Composed idea registration}\label{subseckhuyg777g}
 Let $\X_0,\ldots, \X_{D+1}$ be finite-dimensional Hilbert spaces, with $\X_0=\X$ and $X_{D+1}=\Y$.
 Let $\ell_\Y$ be a loss function on $\Y$ as in Subsec.~\ref{subsecoptrecrr2}.
Let $\nu_1,\ldots,\nu_{D}$ and $\lambda_{0},\ldots,\lambda_{D}$ be  strictly positive parameters.
Let  $\H_{\Gamma^m}$ and $\H_{K^m}$ be RKHS defined by operator-valued kernels $\wK^m\,:\, \X^m\times \X^m \rightarrow \L(\X^m)$ and
$K^m\,:\, \X^m \times \X^m \rightarrow \L(\X^{m+1})$ satisfying the regularity conditions \ref{condnuggetreg}.
Let $L_1,\ldots,L_D$ be strictly positive integers.
The discrete hierarchical warping regression solution to Problem \ref{pb828827hee} is to approximate $f^\dagger$ with
$F_{D+1}$ defined by inductive composition
\begin{equation}\label{eqinductive}
F_{m+1}=f_{m}\circ \phi^{m}(F_{m})\text{ with } \phi^m=(I+v^{m,L_m})\circ \cdots \circ (I+v^{m,1})\text{ and }F_{1}=f_0\,,
\end{equation}
 where
the $v_{m,j}$ are  $f_{m}$ are minimizers of
{\small
\begin{equation}\label{eqlkjgewqdhbjdhyghjeddB2w21reg}
\begin{cases}
\text{Min} &\lambda_{0} \big(\|f_{0}\|_{\H_{K^0}}^2+\frac{1}{\rho} \|f_0(X)-q^{1,1}\|_{\X_{1}^N}^2\big)\\
&+
\sum_{m=1}^{D}\Big(\frac{\nu_m L_m}{2}\sum_{j=1}^{L_m} \big(\|v_{m,j}\|_{\Gamma^m}^2
+\frac{1}{r} \|q^{m,j+1}-(I+v_{m,j})(q^{m,j})\|_{\X_m^N}^2
\big)\\&
+\, \lambda_{m} \big(\|f_{m}\|_{\H_{K^m}}^2+\frac{1}{\rho} \|f_m(q^{m,L_m+1})-q^{m+1,1}\|_{\X_{m+1}^N}^2\big)\Big)+ \ell_\Y \big(q^{D+1,1},Y\big)\\
\text{over}&v_{m,j} \in\H_{\Gamma^m},\, f_{m}\in \H_{K^m},\, q^{m,j}\in \X_m^N\,.
\end{cases}
\end{equation}
}
In the continuous  limit $\min_m L_m \rightarrow \infty$, the  hierarchical warping regression solution to Problem \ref{pb828827hee} is to approximate $f^\dagger$ with
$F_{D+1}$ defined by inductive composition
\begin{equation}\label{eqinductive2}
F_{m+1}=f_{m}\circ \phi^{v_m}(F_{m},1)\text{ with }F_{1}=f_0\,,
\end{equation}
 where
the $v_{m}$ are  $f_{m}$ are minimizers of
{\small
\begin{equation}\label{eqlkjgewqdhbjdhyghjeddB2w21reg2}
\begin{cases}
\text{Min} &\lambda_{0} \big(\|f_{0}\|_{K^0}^2+\frac{1}{\rho} \|f_0(X)-q^{1}(0)\|_{\X_{1}^N}^2\big)\\
&+
\sum_{m=1}^{D}\Big(\frac{\nu_m}{2}\int_0^1 \big(\|v_{m}(\cdot,t)\|_{\Gamma^m}^2
+\frac{1}{r} \|\dot{q}^{m}-v_m(q^m,t)\|_{\X_m^N}^2
\big)\,dt \\&
+\, \lambda_{m} \big(\|f_{m}\|_{K^m}^2+\frac{1}{\rho} \|f_m(q^{m}(1))-q^{m+1}(0)\|_{\X_{m+1}^N}^2\big)\Big)+ \ell_\Y \big(q^{D+1}(0),Y\big)\\
\text{over}&v_{m} \in C([0,1],\H_{\Gamma^m}),\, f_{m}\in \H_{K^m},\, q^{m}\in  C([0,1],\X_m^N)\,.
\end{cases}
\end{equation}
}

\begin{Theorem}\label{thmkjhebedjhbdb}
The map $y=F_{D+1}(x)$ obtained from \eqref{eqinductive} is equal to the output $y$ produced by the block diagram
  \begin{equation}\label{eqhierarch}
  \uset[1ex]{\uset[0.5ex]{\Wk_0}{\downarrow}}{\oset[0.8ex]{\oset[0.4ex]{Z^0}{\downarrow}}{\tensor*[^{{X\rightarrow}}_{{\phantom{+}x\rightarrow}}]{\boxed{K^0\mid \rho}}{^{\rightarrow \phantom{Y'}}_{\rightarrow \phantom{y}}}}} \kern-2.8em
 \uset[1ex]{\uset[0.5ex]{\Vk_1}{\downarrow}}{\oset[1.9ex]{\oset[1.3ex]{p_0^1}{\downarrow}}{\tensor*[^{\phantom{X}\rightarrow}_{\phantom{+x}\rightarrow}]{\boxed{\wK^1\mid r \mid L_1}}{^{\rightarrow }_{\rightarrow }}}}\kern-1.9em
  \uset[1ex]{\uset[0.5ex]{\Wk_1}{\downarrow}}{\oset[0.8ex]{\oset[0.4ex]{Z^1}{\downarrow}}{\tensor*[^{\phantom{X\rightarrow}}_{\phantom{+x\rightarrow}}]{\boxed{K^1\mid \rho}}{^{\rightarrow \phantom{Y'}}_{\rightarrow \phantom{y}}}}} \cdots
 \uset[1ex]{\uset[0.5ex]{\Vk_D}{\downarrow}}{\oset[1.9ex]{\oset[1.3ex]{p_0^D}{\downarrow}}{\tensor*[^{\phantom{X}\rightarrow}_{\phantom{+x}\rightarrow}]{\boxed{\wK^D\mid r \mid L_D}}{^{\rightarrow }_{\rightarrow }}}}\kern-1.9em
  \uset[1ex]{\uset[0.5ex]{\Wk_D}{\downarrow}}{\oset[0.8ex]{\oset[0.4ex]{Z^D}{\downarrow}}{\tensor*[^{\phantom{X\rightarrow}}_{\phantom{+x\rightarrow}}]{\boxed{K^D\mid \rho}}{^{\phantom{\rightarrow Y'}}_{\rightarrow y}}}} \kern-2.1em \tensor*[^{\xrightarrow{\hspace*{1.5cm}}}_{Y\rightarrow}]{\boxed{\ell_\Y}}{_{\rightarrow \ell_\Y}}
  \,.
 \end{equation}
where the initial momenta $p^m_0$ and $Z^m$ are identified as minimizers of the total loss
\begin{equation}\label{eqkhiuhiuhiuhiuh}
\text{Total loss}=\lambda_0\Wk_0+\sum_{m=1}^D (\nu_m \Vk_m+\lambda_m \Wk_m)+\ell_\Y\,.
\end{equation}
All the minimizers $p^m_0$ and $Z^m$  of \eqref{eqhierarch} are contained in a compact set.
The map $y=F_{D+1}(x)$ obtained from \eqref{eqinductive2} is equal to the output $y$ produced by the block diagram
  \begin{equation}\label{eqhierarch2}
  \uset[1ex]{\uset[0.5ex]{\Wk_0}{\downarrow}}{\oset[0.8ex]{\oset[0.4ex]{Z^0}{\downarrow}}{\tensor*[^{{X\rightarrow}}_{{\phantom{+}x\rightarrow}}]{\boxed{K^0\mid \rho}}{^{\rightarrow \phantom{Y'}}_{\rightarrow \phantom{y}}}}} \kern-2.8em
 \uset[1ex]{\uset[0.5ex]{\Vk_1}{\downarrow}}{\oset[1.9ex]{\oset[1.3ex]{p_0^1}{\downarrow}}{\tensor*[^{\phantom{X}\rightarrow}_{\phantom{+x}\rightarrow}]{\boxed{\wK^1\mid r \mid \infty}}{^{\rightarrow }_{\rightarrow }}}}\kern-1.9em
  \uset[1ex]{\uset[0.5ex]{\Wk_1}{\downarrow}}{\oset[0.8ex]{\oset[0.4ex]{Z^1}{\downarrow}}{\tensor*[^{\phantom{X\rightarrow}}_{\phantom{+x\rightarrow}}]{\boxed{K^1\mid \rho}}{^{\rightarrow \phantom{Y'}}_{\rightarrow \phantom{y}}}}} \cdots
 \uset[1ex]{\uset[0.5ex]{\Vk_D}{\downarrow}}{\oset[1.9ex]{\oset[1.3ex]{p_0^D}{\downarrow}}{\tensor*[^{\phantom{X}\rightarrow}_{\phantom{+x}\rightarrow}]{\boxed{\wK^D\mid r \mid \infty}}{^{\rightarrow }_{\rightarrow }}}}\kern-1.9em
  \uset[1ex]{\uset[0.5ex]{\Wk_D}{\downarrow}}{\oset[0.8ex]{\oset[0.4ex]{Z^D}{\downarrow}}{\tensor*[^{\phantom{X\rightarrow}}_{\phantom{+x\rightarrow}}]{\boxed{K^D\mid \rho}}{^{\phantom{\rightarrow Y'}}_{\rightarrow y}}}} \kern-2.1em \tensor*[^{\xrightarrow{\hspace*{1.5cm}}}_{Y\rightarrow}]{\boxed{\ell_\Y}}{_{\rightarrow \ell_\Y}}
  \,.
 \end{equation}
where the initial momenta $p^m_0$ and $Z^m$ are identified as minimizers of \eqref{eqkhiuhiuhiuhiuh}.
All the minimizers $p^m_0$ and $Z^m$  of \eqref{eqhierarch2} are contained in a compact set.
The multivariate input/output maps \eqref{eqhierarch} and \eqref{eqhierarch2} are uniformly  continuous (for $p^m_0$ and $Z^m$ in compact sets). As $\min_m L_m \rightarrow \infty$, (1) the multivariate input/output map \eqref{eqhierarch} converges uniformly (for $p^m_0$ and $Z^m$ in compact sets) to the
multivariate input/output map \eqref{eqhierarch2} (2) The minimal value of total loss of \eqref{eqhierarch} converges to the minimal value of total loss of \eqref{eqhierarch2} (3) The adherence values of the momenta ($p^m_0$ and $Z^m$) minimizing \eqref{eqhierarch} is the set of momenta minimizing \eqref{eqhierarch2} (4) The adherence values of $F_{D+1}$ obtained from \eqref{eqhierarch}  is the set of
 $F_{D+1}$ obtained from \eqref{eqhierarch2}.
\end{Theorem}
\begin{proof}
The proof is a direct consequence of the results of Sec.~\ref{secregular} summarized in Subsec.~\ref{subseckjageejhgdy}. Note that for $r,\rho>0$, \eqref{eqkhiuhiuhiuhiuh} diverges towards infinity as $\max_m (p^m_0)^Tp^m_0 +\max_m (Z^m)^TZ^m\rightarrow \infty$.
Therefore the search for minimizers can be restricted to a compact set.
\end{proof}

\subsection{Further reduction}\label{secfr}

Minimizing over $f_m$ and $v_m$, \eqref{eqlkjgewqdhbjdhyghjeddB2w21reg2} reduces (as in Sec.~\ref{secregular}) to
{\small
\begin{equation}\label{eqlkjgewqdhbjdhyghjeddB2w21reg2bis}
\begin{cases}
\text{Min} &\lambda_{0} (q^{1}(0))^T K_\rho^0(X,X)^{-1} q^{1}(0)+
\sum_{m=1}^{D}\Big(\frac{\nu_m}{2}\int_0^1 \dot{q}^{m}\Gamma^m_r(q^m,q^m)^{-1}\dot{q}^{m}\,dt \\&
+\, \lambda_{m} (q^{m+1}(0))^TK^m_\rho(q^{m}(1),q^{m}(1))^{-1}q^{m+1}(0) \Big)+ \ell_\Y \big(q^{D+1}(0),Y\big)\\
\text{over}& q^{m}\in  C([0,1],\X_m^N)\,.
\end{cases}
\end{equation}
}
 Introduce the momentum variables  $p^m=\Gamma_r(q^m,q^m)^{-1}\dot{q}^m$.
Taking the Fr\'{e}chet derivative of \eqref{eqlkjgewqdhbjdhyghjeddB2w21reg2bis} with respect to $q^m$, implies that (for $1\leq m\leq D$) $(q^m,p^m)$ satisfies the Hamiltonian dynamic \eqref{eqkedmdledkemdlreg} (with $\Gamma_r$ replaced by  $\Gamma^m_r$) and the boundary equations
\begin{equation}\label{eqlgugiuyggh}
\begin{cases}
2 \lambda_{m-1} K_\rho^{m-1}(q^{m-1}(1),q^{m-1}(1))^{-1} q^m(0)-\nu_m p^m(0)&=0\\
\nu_m p^m(1)+\lambda_{m} \partial_{q^m(1)}\big((q^{m+1}(0))^T K_\rho^{m}(q^m(1),q^m(1))^{-1} q^{m+1}(0) \big)&=0
\end{cases}
\end{equation}
We deduce (Prop.~\ref{prophghgfhgfvjhvjh7}) that,
in the search for minimizers of \eqref{eqhierarch2},
the initial momenta $p^m_0=p^m(0)$ can be expressed as explicit functions of the $Z^{m-1}$ and the $Z_m$ can be expressed as implicit functions of the $Z^{m-1}$. Therefore, the search for minimizers of \eqref{eqhierarch2} could, in theory, be reduced to a shooting method (the selection of the initial momentum $Z^0$).

\begin{Proposition}\label{prophghgfhgfvjhvjh7}
In the setting of Thm.~\ref{thmkjhebedjhbdb} the minimizers of \eqref{eqhierarch2} satisfy $q^1(0)=K^{0}(X,X)Z^{0} $
and (for $m\in \{1,\ldots, D\}$)
\begin{equation}\label{eqjhgugigrd3uygyu}
q^m(0)=K^{m-1}(q^{m-1}(1),q^{m-1}(1))Z^{m-1}\,,
\end{equation}
\begin{equation}\label{eqlgugiuyyuggh}
p^m(0) =2 \frac{\lambda_{m-1}}{\nu_m } K_\rho^{m-1}(q^{m-1}(1),q^{m-1}(1))^{-1} K^{m-1}(q^{m-1}(1),q^{m-1}(1))Z^{m-1}\,,
\end{equation}
{\small
\begin{equation}\label{eqlgugiuyyuggh2b}
 \partial_{x}\big( (Z^m)^T K^{m}(q^{m}(1),q^{m}(1)) K_\rho^{m}(x,x)^{-1}K^{m}(q^{m}(1),q^{m}(1))Z^{m}\big)\Big|_{x=q^{m}(1)}=-\frac{\nu_m}{\lambda_{m}} p^m(1)\,.
\end{equation}}
\end{Proposition}
\begin{proof}
\eqref{eqjhgugigrd3uygyu} follows from \eqref{eqjhbdjhbjehddd3}.
Combining \eqref{eqlgugiuyggh} with \eqref{eqjhgugigrd3uygyu} implies \eqref{eqlgugiuyyuggh} and \eqref{eqlgugiuyyuggh2b}.
\end{proof}
Let us now consider the discrete setting.
Minimizing over $v_{m,j}$ and $f_{m}$ \eqref{eqlkjgewqdhbjdhyghjeddB2w21regbis} reduces to
{\small
\begin{equation}\label{eqlkjgewqdhbjdhyghjeddB2w21regbis}
\begin{cases}
\text{Min} &\lambda_{0} (q^{1,1})^TK^0_\rho(X,X)^{-1}q^{1,1}+
\sum_{m=1}^{D}\Big(\frac{\nu_m L_m}{2}\sum_{j=1}^{L_m} (q^{m,j+1}-q^{m,j})^T \Gamma^m_r(q^{m,j},q^{m,j})^{-1}\\& (q^{m,j+1}-q^{m,j})
+\, \lambda_{m} (q^{m+1,1})^TK^m_\rho(q^{m,L_m+1},q^{m,L_m+1})^{-1}q^{m+1,1} \Big)+ \ell_\Y \big(q^{D+1,1},Y\big)\\
\text{over}& q^{m,j}\in \X_m^N\,.
\end{cases}
\end{equation}
}
Introduce the discrete momenta $p^{m,j}=L_m \Gamma^m_r(q^{m,j},q^{m,j})^{-1} (q^{m,j+1}-q^{m,j})$.
Taking the Fr\'{e}chet derivative of \eqref{eqlkjgewqdhbjdhyghjeddB2w21regbis} with respect to $q^{m,j}$ implies that  $(q^{m,j},p^{m,j})$ satisfies the discrete Hamiltonian dynamic \eqref{ejkhdbejhdbdreg} (with $\Gamma_r$ replaced by  $\Gamma^m_r$ and $\Delta t=1/L_m$) and the boundary equations presented in the following proposition (that is the  analogue of Prop.~\ref{prophghgfhgfvjhvjh7}).

\begin{Proposition}\label{propejkhgeuhgde7d6}
In the setting of Thm.~\ref{thmkjhebedjhbdb} the minimizers of \eqref{eqhierarch} satisfy $q^{1,1}=K^{0}(X,X)Z^{0}$
and (for $m\in \{1,\ldots, D\}$)
\begin{equation}\label{eqjhgugigrd3uygyuedw}
q^{m,1}=K^{m-1}(q^{m-1,L_{m-1}+1},q^{m-1,L_{m-1}+1})Z^{m-1}
\end{equation}
{\footnotesize
\begin{equation}\label{eqlgugiuyggh2}
 p^{m,1}-
\frac{1}{2L_m} \partial_{q^m} (p^{m,1})^T \Gamma(q^{m,1},q^{m,1}) p^{m,1}=
2 \frac{\lambda_{m-1}}{\nu_m} K_\rho^{m-1}(q^{m-1,L_{m-1}+1},q^{m-1,L_{m-1}+1})^{-1} q^{m,1}
\end{equation}}
\begin{equation}\label{eqlgugiuyggh2B}
\nu_m  p^{m,L_m}+\lambda_{m} \partial_{q^{m,L_{m+1}}}\big((q^{m+1,1})^T K_\rho^{m}(q^{m,L_m+1},q^{m,L_m+1})^{-1} q^{m+1,1} \big)=0
\end{equation}
\end{Proposition}

\subsection{Continuous limit of ResNets}\label{secjhgwyde6dga}
Consider the setting of Subsec.~\ref{subseckhuyg777g}.
Let $\bvarphi(x)=(\ba(x),1)$ where $\ba$ is an activation function obtained as an
 elementwise nonlinearity satisfying the regularity conditions of Rmk.~\ref{rmkkjhevgcvcg78}.
The ANN  solution to Problem \ref{pb828827hee} is to approximate $f^\dagger$ with
$F_{D+1}$ defined by inductive composition
\begin{equation}\label{eqinductive3}
F_{m+1}=\tilde{w}^{m}\bvarphi\big( \phi^{m}(F_{m})\big)\text{ with }\phi^m=(I+w^{m,L_m}\bvarphi(\cdot))\circ \cdots \circ (I+w^{m,1}\bvarphi(\cdot)) \text{ and }F_{1}=\tilde{w}^{0}\bvarphi(\cdot)\,,
\end{equation}
 where
the $\tilde{w}^{m}$ and  $w^{m,j}$ are minimizers of
{\small
\begin{equation}\label{eqlkjgewqdhbjdhyghjeddB2h}
\begin{cases}
\text{Min} &\lambda_{0} \big(\|\tilde{w}^{0}\|_{\L(\X_0\oplus \R,\X_{1})}^2+\frac{1}{\rho} \|\tilde{w}^{0}\bvarphi(X)-q^{1,1}\|_{\X_{1}^N}^2\big)\\
&+
\sum_{m=1}^{D}\Big(\frac{\nu_m L_m}{2}\sum_{j=1}^{L_m} \big(\|w^{m,j}\|_{\L(\X_m \oplus \R,\X_m)}^2+
\frac{1}{r} \|q^{m,j+1}-q^{m,j}-w^{m,j}\bvarphi(q^{m,j})\|_{\X_m^N}^2
\big)\\&
+\, \lambda_{m} \big(\|\tilde{w}^{m}\|_{\L(\X_m \oplus \R,\X_{m+1})}^2+\frac{1}{\rho} \|\tilde{w}^{m}\bvarphi(q^{m,L_m+1})-q^{m+1,1}\|_{\X_{m}^N}^2\big)
\Big)+ \ell_\Y \big(q^{D+1,1},Y\big)\\
\text{over}&w^{m,j} \in \L(\X_m \oplus \R,\X_m),\, \tilde{w}^m \in \L(\X_m \oplus \R,\X_{m+1}),\,q^{m,j}\in \X_m^N \,.
\end{cases}
\end{equation}
}
Note that in the limit $\nu_m \rightarrow \infty$ we have
$F_{m+1}= \tilde{w}^m \bvarphi(F_{m})$, whereas the traditional way is to use $h_m=\bvarphi(F_m)$ as variables and represent ANNs as $h_{m+1}= \bvarphi(\tilde{w}^m h_{m})$.  Furthermore the  $\phi^m$ represent concatenation of ResNet  blocks \cite{he2016deep}.
In the continuous ($\min_m {L_m}\rightarrow \infty$)  limit, the composed idea registration solution
to Problem \ref{pb828827hee} is to approximate $f^\dagger$ with
$F_{D+1}$ defined by inductive composition
\begin{equation}\label{eqinductive4}
F_{m+1}=\tilde{w}^{m}\bvarphi\big(  \phi^{v_m}(F_{m},1)\big)\text{ with } v_{m}(x,t)=w^{m}(t)\bvarphi(x)\text{ and }F_{1}=\tilde{w}^{0}\bvarphi(\cdot)\,,
\end{equation}
 where the    $\tilde{w}^{m}$ and $w^{m}$ are minimizers of

{\small
\begin{equation}\label{eqlkjgewqdhbjdhyghjeddB2h2}
\begin{cases}
\text{Min} &\lambda_{0} \big(\|\tilde{w}^{0}\|_{\L(\X_0\oplus \R,\X_{1})}^2+\frac{1}{\rho} \|\tilde{w}^{0}\bvarphi(X)-q^{1}(0)\|_{\X_{1}^N}^2\big)\\
&+
\sum_{m=1}^{D}\Big(\frac{\nu_m}{2}\int_0^1 \big(\|w^{m}(\cdot,t)\|_{\L(\X_m\oplus \R,\X_m)}^2
+\frac{1}{r} \|\dot{q}^{m}-w^{m}(t)\bvarphi(q^{m}(t))\|_{\X_m^N}^2
\big)\,dt\\&
+\, \lambda_{m} \big(\|\tilde{w}^{m}\|_{\L(\X_m\oplus \R,\X_{m+1})}^2+\frac{1}{\rho} \|\tilde{w}^{m}\bvarphi(q^{m}(1))-q^{m+1}(0)\|_{\X_{m}^N}^2\big)
\Big)+ \ell_\Y \big(q^{D+1}(0),Y\big)\\
\text{over}&w^{m} \in  C^1([0,1],\L(\X_m\oplus \R,\X_m)),\, \tilde{w}^m \in \L(\X_m\oplus \R,\X_{m+1}),\,q^{m}\in C^1([0,1], \X_m^N)\,.
\end{cases}
\end{equation}
}

The following theorem is a direct consequence of the equivalence established in Subsec.~\ref{subsecekuwege62} and Thm.~\ref{thmkjhebedjhbdb}.

\begin{Theorem}\label{thmejhekjhbfjhrbf}
\eqref{eqlkjgewqdhbjdhyghjeddB2h} and \eqref{eqlkjgewqdhbjdhyghjeddB2h2} have minimizers.
Minimal values of \eqref{eqlkjgewqdhbjdhyghjeddB2h} and \eqref{eqlkjgewqdhbjdhyghjeddB2h2} are continuous in $(X,Y)$.
Minimal values and minimizers $F_{D+1}$=\eqref{eqinductive3} of \eqref{eqlkjgewqdhbjdhyghjeddB2h} converge (in the sense of adherence values of Subsec.~\ref{subskjkecgyf67f6reg}), as $\min_m L_m\rightarrow \infty$ towards minimal values and minimizers $F_{D+1}$=\eqref{eqinductive4} of \eqref{eqlkjgewqdhbjdhyghjeddB2h2}.
At the minima, the $\big(\|w^{m}(\cdot,t)\|_{\L(\X_m\oplus \R,\X_m)}^2
+\frac{1}{r} \|\dot{q}^{m}-w^{m}(t)\bvarphi(q^{m}(t))\|_{\X_m^N}^2
\big)$ are constant over $t\in [0,1]$ and $\big(\|w^{m,j}\|_{\L(\X_m\oplus \R,\X_m)}^2
+\frac{1}{r} \|q^{m,j+1}-q^{m,j}-w^{m,j}\bvarphi(q^{m,j})\|_{\X_m^N}^2
\big)$ fluctuates by at most $\mathcal{O}(1/L_m)$ over $j\in \{1,\ldots,L_m\}$.
Let $\wK^m(x,x')=\bvarphi(x)^T \bvarphi(x') \, I_{\X_m}$, $K^m(x,x')=\bvarphi(x)^T \bvarphi(x') \, I_{\X_{m+1}}$ be the kernels defined by the activation function $\bvarphi$ as in \eqref{eqkj2lehdbljebd}.
The maps $y=F_{D+1}(x)$ obtained from \eqref{eqlkjgewqdhbjdhyghjeddB2h} and \eqref{eqlkjgewqdhbjdhyghjeddB2h2} are equal to the output $y$ produced by the block diagrams \eqref{eqhierarch} and \eqref{eqhierarch2}, where the initial momenta $p^m_0$ and $Z^m$ are identified as minimizers of the total loss \eqref{eqkhiuhiuhiuhiuh}. In particular the results of Thm.~\ref{thmkjhebedjhbdb}, Prop.~\ref{prophghgfhgfvjhvjh7} and Prop.~\ref{propejkhgeuhgde7d6} hold true for \eqref{eqlkjgewqdhbjdhyghjeddB2h} and \eqref{eqlkjgewqdhbjdhyghjeddB2h2}.
\end{Theorem}

 \begin{figure}[h!]
	\begin{center}
			\includegraphics[width=\textwidth]{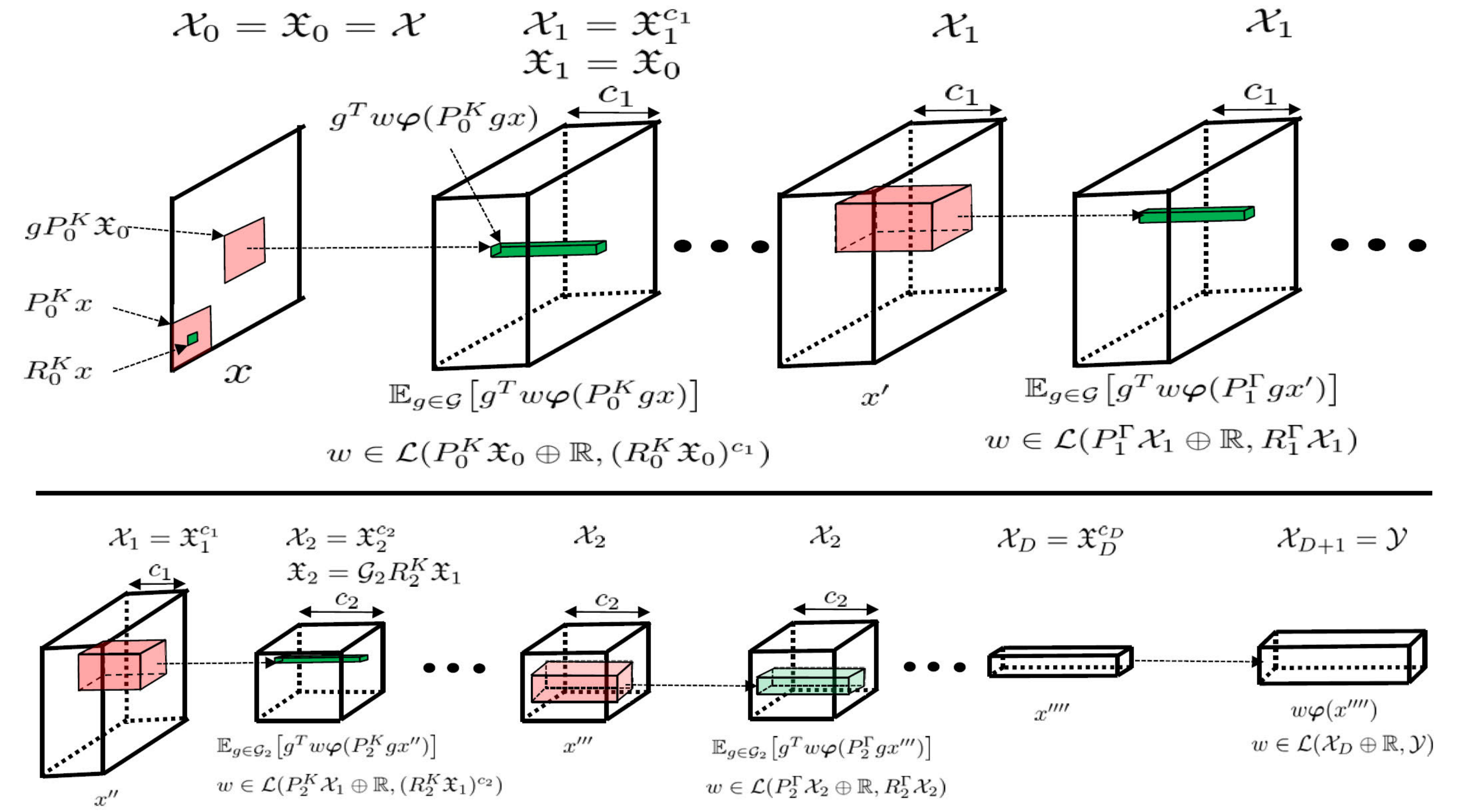}
		\caption{CNNs and ResNets as composed warping regression with REM kernels.}\label{figifcnn}
	\end{center}
\end{figure}

\subsection{Composed warping regression with REM kernels}\label{subseciuigiwuegew6gd}
We will now show that {\bf CNNs and ResNets are particular instances of composing warping regression blocks (composed discretized idea registration with REM kernels} as illustrated in Fig.~\ref{figifcnn}.
Consider the setting of Sec.~\ref{subseckhuyg777g} and \ref{secjhgwyde6dga}.
Let $\bvarphi(x)=(\ba(x),1)$ where $\ba$ is an activation function obtained as an
 elementwise nonlinearity satisfying the regularity conditions of Rmk.~\ref{rmkkjhevgcvcg78}.
Given $D \geq 1$, let $\Xf_0,\ldots,\Xf_D$ and $\X_0,\ldots, \X_{D+1}$ be finite-dimensional Hilbert spaces constructed as follows.
Set $\Xf_0=\X_0=\X$ and $\X_{D+1}=\Y$.
For $m\in \{1,\ldots,D\}$ let $P_m^{\Gamma}\,:\, \Xf_m\rightarrow \Xf_m$, $R_m^{\Gamma}\,:\, \Xf_m\rightarrow \Xf_m$,
 be linear projections, and for
  $m\in \{0,\ldots,D-1\}$ let
 $P_m^{K}\,:\, \Xf_{m}\rightarrow \Xf_{m}$, $R_m^{K}\,:\, \Xf_{m}\rightarrow \Xf_{m}$ be linear projections.
Let $c_1,\ldots,c_D\in \mathbb{N}^*$.
For $m\in \{1,\ldots,D\}$, let $\X_m=\Xf^{c_m}$. Let $\G$ be a unitary unimodular group on $\X$, write $\G_0=\G$ and for $m\in \{0,\ldots,D\}$ let $\G_{m+1}$ be a subgroup of $\G_{m}$ and let $\Xf_{m+1}=\G_{m} R_{m}\Xf_m$.
For $m\in \{0,\ldots,D-1\}$ let  $K^m\,:\, \X_m\times \X_m \rightarrow \L(\X_{m+1})$ be the REM kernel defined (as in \eqref{eqjbhhebddbdbgw}) by
$P^K_m$, $R^K_m$, $\G_m$. Let $K^D\,:\, \X_D\times \X_D\rightarrow \L(\X_{D+1})$ be the REM kernel with feature map defined by $\Psi^T(x) \alpha= w \bvarphi( x)$ with $w\in \L(\X^D\oplus \R, \Y)$.
For $m\in \{1,\ldots,D\}$ let  $\Gamma^m\,:\, \X_m\times \X_m \rightarrow \L(\X_{m})$ be the REM kernel defined by the feature map \eqref{eqjbhhebddbdbgw} using
$P^\Gamma_m$, $R^\Gamma_m$, $\G_m$.

The corresponding composed warping regression solution to Problem \ref{pb828827hee} is to approximate $f^\dagger$ with
$F_{D+1}$ defined by inductive composition \eqref{eqinductive}
 with
\begin{equation}\label{eqjwgde8dge6d762s}
f_m=\E_{\G_m}\big[g^T \tilde{w}^{m}\bvarphi\big( P^K_{m} g \cdot \big]\text{ and }
v_{m,j}=\E_{\G_{m}}\big[g^T w^{m,j}\bvarphi(P^\Gamma_m g \cdot)\big]\,,
\end{equation}
  where
the $\tilde{w}^{m}$ are  $w^{m,j}$ are minimizers of
{\small
\begin{equation}\label{eqlkjgewqdhbjdhyghjeddB2hT}
\begin{cases}
\text{Min} &\lambda_{0} \big(\|\tilde{w}^{0}\|_{\L(P^K_0\Xf_0\oplus \R,(R^K_0\Xf_0)^{c_1})}^2+\frac{1}{\rho} \|f_0(X)-q^{1,1}\|_{\X_{1}^N}^2\big)
+
\sum_{m=1}^{D}\Big(\frac{\nu_m L_m}{2}  \sum_{j=1}^{L_m} \\&\big(\|w^{m,j}\|_{\L(P^\Gamma_m\X_m \oplus \R,R^\Gamma_m\X_m)}^2+
\frac{1}{r} \|q^{m,j+1}-q^{m,j}-v_{m,j}(q^{m,j})\|_{\X_m^N}^2
\big)
+\lambda_{m}\cdot \\\big(\|\tilde{w}^{m}&\|_{\L(P^K_m \Xf^{c_m} \oplus \R,R^\Gamma_m\Xf^{c_{m+1}})}^2+\frac{1}{\rho} \|f_m(q^{m,L_m+1})-q^{m+1,1}\|_{\X_{m}^N}^2\big)
\Big)+ \ell_\Y \big(q^{D+1,1},Y\big)\\
\text{over}&w^{m,j} \in \L(P^\Gamma_m\X_m \oplus \R,R^\Gamma_m\X_m),\,\tilde{w}^{m}\in \L(P^K_m \Xf^{c_m} \oplus \R,R^\Gamma_m\Xf^{c_{m+1}}),\,q^{m,j}\in \X_m^N \,.
\end{cases}
\end{equation}
}
In the continuous ($\min_m {L_m}\rightarrow \infty$)  limit, the composed idea registration solution
to Problem \ref{pb828827hee} is to approximate $f^\dagger$ with
$F_{D+1}$ defined by inductive composition \eqref{eqinductive2} with $f_m=\E_{\G_m}\big[g^T \tilde{w}^{m}\bvarphi\big( P^K_{m} g \cdot\big) \big]$ and
 $v_{m}(x,t)=\E_{\G_{m}}\big[g^T w^{m}(t)\bvarphi(P^\Gamma_m g \cdot)\big]$ where the    $\tilde{w}^{m}$ and $w^{m}$ are minimizers of

 {\tiny
\begin{equation}\label{eqlkjgewqdhbjdhyghjeddB2h2T}
\begin{cases}
\text{Min} &\lambda_{0} \big(\|\tilde{w}^{0}\|_{\L(P^K_0\Xf_0\oplus \R,(R^K_0\Xf_0)^{c_1})}^2+\frac{1}{\rho} \|f_0(X)-q^{1}(0)\|_{\X_{1}^N}^2\big)
+
\sum_{m=1}^{D}\Big(\frac{\nu_m }{2}  \int_0^1 \\&\big(\|w^{m}(t)\|_{\L(P^\Gamma_m\X_m \oplus \R,R^\Gamma_m\X_m)}^2+
\frac{1}{r} \|\dot{q}^{m}-v_{m}(q^{m},t)\|_{\X_m^N}^2\,dt
\big)
+ \\ \lambda_{m} \big(\|\tilde{w}^{m}&\|_{\L(P^K_m \Xf^{c_m} \oplus \R,R^\Gamma_m\Xf^{c_{m+1}})}^2+\frac{1}{\rho} \|f_m(q^{m}(1))-q^{m+1}(0)\|_{\X_{m}^N}^2\big)
\Big)+ \ell_\Y \big(q^{D+1}(0),Y\big)\\
\text{over}&w^{m} \in C^1\big([0,1], \L(P^\Gamma_m\X_m \oplus \R,R^\Gamma_m\X_m)\big),\,\tilde{w}^{m}\in\L(P^K_m \Xf^{c_m} \oplus \R,R^\Gamma_m\Xf^{c_{m+1}}),\,q^{m}\in C^1\big([0,1],\X_m^N\big) \,.
\end{cases}
\end{equation}
}

 The following theorem is a direct consequence of the equivalence established in Subsec.~\ref{subsecekuwege62} and Thm.~\ref{thmkjhebedjhbdb}.

\begin{Theorem}\label{thmejhekjhbfjhrbf2}
\eqref{eqlkjgewqdhbjdhyghjeddB2hT} and \eqref{eqlkjgewqdhbjdhyghjeddB2h2T} have minimizers.
Minimal values of \eqref{eqlkjgewqdhbjdhyghjeddB2hT} and \eqref{eqlkjgewqdhbjdhyghjeddB2h2T} are continuous in $(X,Y)$.
Minimal values  and $F_{D+1}$ determined by minimizers  of \eqref{eqlkjgewqdhbjdhyghjeddB2hT}  converge (in the sense of adherence values of Subsec.~\ref{subskjkecgyf67f6reg}), as $\min_m L_m\rightarrow \infty$ towards minimal values and  $F_{D+1}$ determined by minimizers of \eqref{eqlkjgewqdhbjdhyghjeddB2h2T}.
At minima, the $\big(\|w^{m}(t)\|_{\L(P^\Gamma_m\X_m \oplus \R,R^\Gamma_m\X_m)}^2+
\frac{1}{r} \|\dot{q}^{m}-v_m(q^{m},t)\|_{\X_m^N}^2\,dt
\big)$ are constant over $t\in [0,1]$ and $\big(\|w^{m,j}\|_{\L(P^\Gamma_m\X_m \oplus \R,R^\Gamma_m\X_m)}^2+
\frac{1}{r} \|q^{m,j+1}-q^{m,j}-v^{m,j}(q^{m,j})\|_{\X_m^N}^2
\big)$ fluctuates by at most $\mathcal{O}(1/L_m)$ over $j\in \{1,\ldots,L_m\}$.
The maps $y=F_{D+1}(x)$ obtained from \eqref{eqlkjgewqdhbjdhyghjeddB2hT} and \eqref{eqlkjgewqdhbjdhyghjeddB2h2T} are equal to the output $y$ produced by the block diagrams \eqref{eqhierarch} and \eqref{eqhierarch2}, and the initial momenta $p^m_0$ and $Z^m$ are identified as minimizers of the total loss \eqref{eqkhiuhiuhiuhiuh}. In particular the results of Thm.~\ref{thmkjhebedjhbdb}, Prop.~\ref{prophghgfhgfvjhvjh7} and Prop.~\ref{propejkhgeuhgde7d6} hold true for \eqref{eqlkjgewqdhbjdhyghjeddB2hT} and \eqref{eqlkjgewqdhbjdhyghjeddB2h2T}.
\end{Theorem}

\subsection{The algorithm}\label{subsecalg}
The practical minimization of \eqref{eqlkjgewqdhbjdhyghjeddB2h2T} is as follows.
Introduce the slack variables  $\tilde{z}^0=q^{1,1}-f_0(X)$, $z^{m,j}=q^{m,j+1}-q^{m,j}-v_{m,j}(q^{m,j})$
and $\tilde{z}^m=q^{m+1,1}-f_m(q^{m,L_m+1})$. Let  $\ell_\Y$ be an arbitrary empirical loss (e.g., $\ell_\Y(Y',Y)=\|Y'-Y\|_{\Y^N}^2$).
Replace the minimization over the variables $q^{m,j}$ by the minimization over the slack variables (note that $q^{D+1,1}$ is a function of $X$, weights, biases, and slack variables).
 Use minibatching (as commonly practiced in ML) to form an unbiased estimate of the gradient of the total loss with respect to the weights, biases, and slack variables.
 The exact averages $\E_{ \G_m}$ in \eqref{eqjwgde8dge6d762s} can be replaced\footnote{Since REM feature maps are expressed as expected values with respect to a randomization of the action of the group, their simulation can be randomized as in \cite{chan2015pcanet}.}  by Monte-Carlo averages (by sampling the Haar measure over $\G_m$). Modify the weights, biases, and slack variables in the gradient descent direction (note that the only slack variables impacted are those indexed by the minibatch). Repeat.

\end{document}